\documentclass[twoside,11pt]{article}
\usepackage{jmlr2e}            % <- uncomment for JMLR submission

\usepackage{amsmath,amsfonts}
\usepackage{url}
\usepackage[usenames,dvipsnames,svgnames,table]{xcolor}

\usepackage{algorithm}
\usepackage{algpseudocode}
\usepackage{caption}
\usepackage{comment}
\usepackage{dsfont}
\usepackage{graphicx,verbatim}
\usepackage{listings}
\usepackage{subcaption}
\usepackage{todonotes}

\usepackage{tikz}
\usetikzlibrary{arrows}

\def\ie{{\it i.e.},~}
\def\eg{{\it e.g.},~}

\def\bif{\bf if~}
\def\belif{{\bf else if~}}
\def\bthen{{\bf then predict~}}
\def\belse{{\bf else predict~}}
\def \band{{\bf ~and~}}

\def\E{\mathbb{E}}
\def\P{\mathbb{P}}

\def\one{\mathds{1}}

\newcommand{\x}{\mathbf{x}}
\newcommand{\y}{\mathbf{y}}
\def\RL{{d}}
\def\Prefix{d_p}
\def\Labels{\delta_p}
\def\Default{q_0}
\def\RLB{{D}}
\def\PrefixB{D_p}
\def\LabelsB{\Delta_p}
\def\DefaultB{Q_0}
\def\Obj{R}
\def\Loss{\ell}
\def\Reg{{\lambda}}

\def\RuleSet{S}
\def\Cap{\textnormal{cap}}
\def\Supp{\textnormal{supp}}
\def\Count{\textnormal{sum}}
\def\InitialObj{R^0}
\def\InitialRL{d^0}
\def\CurrentObj{{R^c}}
\def\CurrentRL{d^c}
\def\OptimalObj{R^*}
\def\OptimalRL{d^*}
\def\Remaining{\Gamma}
\def\TotalRemaining{{\Gamma_{\textnormal{tot}}}}
\def\StartsWith{\sigma}
\def\StartContains{\phi}
\def\Queue{Q}
\DeclareMathOperator*{\argmin}{argmin}
\DeclareMathOperator*{\argmax}{argmax}

\def\Curiosity{\mathcal{C}}
\def\NCap{N_{\textnormal{cap}}}
\def\Continue{\textbf{continue}}
\def\PMap{P}

\def\one{\mathds{1}}

\newcommand{\nn}{\nonumber}
\newcommand{\be}{\begin{equation}}
\newcommand{\ee}{\end{equation}}
\newcommand{\bea}{\begin{eqnarray}}
\newcommand{\eea}{\end{eqnarray}}

\newcommand{\given}{\,|\,}

\usepackage{caption}
\includecomment{arxiv}\excludecomment{kdd}

\usepackage{lastpage}

\jmlrheading{18}{2018}{1-\pageref{LastPage}}{11/17}{6/18}{17-716}{Elaine Angelino, Nicholas Larus-Stone, Daniel Alabi, Margo Seltzer, and Cynthia Rudin}

\ShortHeadings{Learning Certifiably Optimal Rule Lists}{Angelino, Larus-Stone, Alabi, Seltzer, and Rudin}
\firstpageno{1}

\begin{document}

\title{Learning Certifiably Optimal Rule Lists for Categorical Data}

\author{\name Elaine Angelino \email elaine@eecs.berkeley.edu \\
        \addr Department of Electrical Engineering and Computer Sciences\\
        University of California, Berkeley,
        Berkeley, CA 94720
        \AND
        \name Nicholas Larus-Stone \email nlarusstone@alumni.harvard.edu \\
        \name Daniel Alabi \email alabid@g.harvard.edu \\
        \name Margo Seltzer \email margo@eecs.harvard.edu \\
        \addr School of Engineering and Applied Sciences\\
        Harvard University,
        Cambridge, MA 02138
        \AND
        \name Cynthia Rudin$^*$ \email cynthia@cs.duke.edu \\
        \addr Department of Computer Science and
        Department of Electrical and Computer Engineering\\
        Duke University,
        Durham, NC 27708}

\editor{Maya Gupta\\ $^*$To whom correspondence should be addressed.}
%\footnote{}

\maketitle

\begin{abstract}%   <- trailing '%' for backward compatibility of .sty file
We present the design and implementation of a custom discrete optimization
technique for building rule lists over a categorical feature space.
Our algorithm produces rule lists with optimal training performance,
according to the regularized empirical risk, with a certificate of optimality.
By leveraging algorithmic bounds, efficient data structures,
and computational reuse, we achieve several orders of magnitude speedup in time
and a massive reduction of memory consumption.
We demonstrate that our approach produces optimal rule lists on practical
problems in seconds.
Our results indicate that it is possible to construct optimal sparse rule lists that are
approximately as accurate as the COMPAS proprietary risk prediction tool on data from
Broward County, Florida, but that are completely interpretable.
This framework is a novel alternative to CART and other decision tree methods for interpretable modeling.
\end{abstract}


\section{Introduction}

As machine learning continues to gain prominence in socially-important decision-making,
the interpretability of predictive models remains a crucial problem.
Our goal is to build models that are highly predictive, transparent, and easily understood by humans.
We use rule lists, also known as decision lists, to achieve this goal.
Rule lists are predictive models composed of if-then statements;
these models are interpretable because the rules provide a reason for each prediction~(Figure~\ref{fig:rule-list}).

Constructing rule lists, or more generally, decision trees, has been a challenge for more than
30 years; most approaches use greedy splitting techniques~\citep{Rivest87,Breiman84,Quinlan93}.
Recent approaches use Bayesian analysis, either to find a locally optimal solution~\citep{Chipman:1998jh} or to explore the search space~\citep{LethamRuMcMa15, YangRuSe16}.
These approaches achieve high accuracy while also managing to run reasonably quickly.
However, despite the apparent accuracy of the rule lists generated by these algorithms,
there is no way to determine either if the generated rule list is optimal or how close it is to optimal,
where optimality is defined with respect to minimization of a regularized loss function.

\begin{arxiv}
\begin{figure}[t!]
\begin{algorithmic}
\State \bif $(age=18-20) \band (sex=male)$ \bthen $yes$
\State \belif $(age=21-23)	 \band (priors=2-3)$ \bthen $yes$
\State \belif $(priors>3)$ \bthen $yes$
\State \belse $no$
\end{algorithmic}
\caption{An example rule list that predicts two-year recidivism
for the ProPublica data set, found by CORELS.
}
\label{fig:rule-list}
\end{figure}
\end{arxiv}

Optimality is important, because there are societal implications for a lack of optimality.
Consider the ProPublica article on the Correctional Offender Management Profiling for Alternative Sanctions
(COMPAS) recidivism prediction tool~\citep{LarsonMaKiAn16}.
It highlights a case where a black box, proprietary predictive model is being used for recidivism prediction.
The authors hypothesize that the COMPAS scores are racially biased,
but since the model is not transparent, no one (outside of the creators of COMPAS)
can determine the reason or extent of the bias~\citep{LarsonMaKiAn16},
nor can anyone determine the reason for any particular prediction.
By using COMPAS, users implicitly assumed that a transparent model
would not be sufficiently accurate for recidivism prediction,
\ie they assumed that a black box model would provide better accuracy.
We wondered whether there was indeed no transparent and sufficiently accurate model.
Answering this question requires solving a computationally hard problem.
Namely, we would like to both find a transparent model that is optimal
within a particular pre-determined class of models
and produce a certificate of its optimality, with respect to the regularized empirical risk.
This would enable one to say, for this problem and model class,
with certainty and before resorting to black box methods,
whether there exists a transparent~model.
While there may be differences between training and test performance,
finding the simplest model with optimal training performance is prescribed by
statistical learning theory.

To that end, we consider the class of rule lists assembled from pre-mined frequent itemsets
and search for an optimal rule list that minimizes a regularized risk function,~$R$.
This is a hard discrete optimization problem.
Brute force solutions that minimize~$R$ are computationally prohibitive
due to the exponential number of possible rule lists.
However, this is a worst case bound that is not realized in practical settings.
For realistic cases, it is possible to solve fairly large cases of this problem to optimality,
with the careful use of algorithms, data structures, and implementation techniques.

\begin{kdd}
\begin{figure}[b!]
\vspace{-3mm}
\begin{algorithmic}
\normalsize
\State \bif $(age=23-25) \band (priors=2-3)$ \bthen $yes$
\State \belif $(age=18-20)$ \bthen $yes$
\State \belif $(sex=male) \band (age=21-22)$ \bthen $yes$
\State \belif $(priors>3)$ \bthen $yes$
\State \belse $no$
\end{algorithmic}
\vspace{-3mm}
\caption{An example rule list that predicts two-year recidivism
for the ProPublica data set, found by CORELS.
}
\label{fig:rule-list}
\end{figure}
\end{kdd}

We develop specialized tools from the fields of discrete optimization and artificial intelligence.
Specifically, we introduce a special branch-and bound algorithm, called
Certifiably Optimal RulE ListS (CORELS), that provides the optimal solution
according to the training objective, along with a certificate of optimality.
The certificate of optimality means that we can investigate how close other models
(\eg models provided by greedy algorithms) are to optimal.

\begin{arxiv}
Within its branch-and-bound procedure, CORELS maintains a lower bound on the
minimum value of~$R$ that each incomplete rule list can achieve.
This allows CORELS to prune an incomplete rule list (and every possible extension)
if the bound is larger than the error of the best rule list that it has already evaluated.
The use of careful bounding techniques leads to massive pruning of
the search space of potential rule lists.
The algorithm continues to consider incomplete and complete rule lists until it has either
examined or eliminated every rule list from consideration.
Thus, CORELS terminates with the optimal rule list and a certificate of optimality.
\end{arxiv}

The efficiency of CORELS depends on how much of the search space our bounds
allow us to prune; we seek a tight lower bound on~$R$.
The bound we maintain throughout execution is a maximum of several bounds
that come in three categories.
The first category of bounds are those intrinsic to the rules themselves.
This category includes bounds stating that each rule must capture sufficient data;
if not, the rule list is provably non-optimal.
The second type of bound compares a lower bound on the value of~$R$
to that of the current best solution.
This allows us to exclude parts of the search space that could never be better
than our current solution.
Finally, our last type of bound is based on comparing incomplete rule lists that
capture the same data and allows us to pursue only the most accurate option.
This last class of bounds is especially important---without our use of a novel
\textit{symmetry-aware map}, we are unable to solve most problems of reasonable scale.
This symmetry-aware map keeps track of the best accuracy
over all observed permutations of a given incomplete rule list.

We keep track of these bounds using a modified \emph{prefix tree},
a data structure also known as a trie.
Each node in the prefix tree represents an individual rule;
thus, each path in the tree represents a rule list such that
the final node in the path contains metrics about that rule list.
This tree structure, together with a search policy and sometimes a queue,
enables a variety of strategies, including breadth-first,
best-first, and stochastic search.
In particular, we can design different best-first strategies
by customizing how we order elements in a priority queue.
In addition, we are able to limit the number of nodes in the trie
and thereby enable tuning of space-time tradeoffs in a robust manner.
This trie structure is a useful way of organizing the generation
and evaluation of rule lists.

\begin{arxiv}
We evaluated CORELS on a number of publicly available data sets.
Our metric of success was 10-fold cross-validated prediction accuracy on a subset of the data.
These data sets involve hundreds of rules and thousands of observations.
CORELS is generally able to find an optimal rule list in a matter of seconds
and certify its optimality within about 10 minutes.
We show that we are able to achieve better or similar out-of-sample accuracy on these
data sets compared to the popular greedy algorithms, CART and C4.5.
\end{arxiv}

CORELS targets large (not massive) problems,
where interpretability and certifiable optimality are important.
We illustrate the efficacy of our approach using (1)~the ProPublica COMPAS data set~\citep{LarsonMaKiAn16}, for the problem of two-year recidivism prediction,
and (2)~stop-and-frisk data sets from the
\begin{kdd}
New York Civil Liberties Union (NYCLU)~\citep{nyclu:2014},
\end{kdd}
\begin{arxiv}
NYPD~\citep{nypd} and the NYCLU~\citep{nyclu:2014},
\end{arxiv}
to predict whether a weapon will be found
on a stopped individual who is frisked or searched.
On these data, we produce certifiably optimal, interpretable rule lists that achieve
the same accuracy as approaches such as random forests.
This calls into question the need for use of a proprietary,
black box algorithm for recidivism prediction.

Our work overlaps with the thesis of~\citet{Larus-Stone17}.
We have also written a
\begin{kdd}
long version of this report that includes proofs to all
bounds in~\S\ref{sec:framework}, additional bounds and empirical results,
and further implementation and data processing details~\citep{AngelinoLaAlSeRu17}. \\

Our code is at \textbf{\url{https://github.com/nlarusstone/corels}}.
\end{kdd}
\begin{arxiv}
preliminary conference version of this article~\citep{AngelinoLaAlSeRu17-kdd}, and a report
highlighting systems optimizations of our implementation~\citep{Larus-Stone18-sysml}; the latter includes
additional empirical measurements not presented here. \\

Our code is at \textbf{\url{https://github.com/nlarusstone/corels}},
where we provide the C++ implementation we used in our experiments~(\S\ref{sec:experiments}).
% Python and R wrappers?
%
\citet{corels-website} have also created an interactive web interface at
\textbf{\url{https://corels.eecs.harvard.edu}}, where a user can upload data and
run CORELS from a browser.
\end{arxiv}

\section{Related Work}

Since every rule list is a decision tree and every decision tree can be expressed as an equivalent rule list, the problem we are solving is a version of the ``optimal decision tree'' problem, though regularization changes the nature of the problem (as shown through our bounds). The optimal decision tree problem is computationally hard, though since the late 1990's, there has been research on building optimal decision trees using optimization techniques~\citep{Bennett96optimaldecision,dobkininduction,FarhangfarGZ08}. A particularly interesting paper along these lines is that of \citet{NijssenFromont2010}, who created a ``bottom-up'' way to form optimal decision trees. Their method performs an expensive search step, mining all possible leaves (rather than all possible rules), and uses those leaves to form trees. Their method can lead to memory problems, but it is possible that these memory issues can be mitigated using the theorems in this paper.\footnote{There is no public version of their code for distribution as of this writing.} None of these methods used the tight bounds and data structures of CORELS.

Because the optimal decision tree problem is hard, there are a huge number of algorithms such as CART \citep{Breiman84} and C4.5 \citep{Quinlan93} that do not perform exploration of the search space beyond greedy splitting. Similarly, there are decision list and associative classification methods that construct rule lists iteratively in a greedy way
\citep{Rivest87,Liu98,Li01,Yin03,Sokolova03,Marchand05,Vanhoof10,RudinLeMa13}.
Some exploration of the search space is done by Bayesian decision tree methods~\citep{Dension:1998hl,Chipman:2002hc,Chipman10} and Bayesian rule-based methods \citep{LethamRuMcMa15,YangRuSe16}. The space of trees of a given depth is much larger than the space of
rule lists of that same depth, and the trees within the Bayesian tree algorithms
are grown in a top-down greedy way. Because of this, authors of Bayesian tree algorithms have noted that their MCMC chains tend to reach only locally optimal solutions.
The RIPPER algorithm \citep{ripper} is similar to the Bayesian tree methods in that it grows, prunes, and then locally optimizes.
The space of rule lists is smaller than that of trees, and has simpler structure.
Consequently, Bayesian rule list algorithms tend to be more successful at
escaping local minima and can introduce methods of exploring the search space
that exploit this structure---these properties motivate our focus on lists.
That said, the tightest bounds for the Bayesian lists \citep[namely, those of][upon whose work we build]{YangRuSe16},
are not nearly as tight as those of CORELS.

Tight bounds, on the other hand, have been developed for the (immense) literature on building disjunctive normal form (DNF) models; a good example of this is the work of \citet{Rijnbeek10}.
For models of a given size, since the class of DNF's is a proper subset of decision lists, our framework can be restricted to learn optimal DNF's.
The field of DNF learning includes work from the fields of rule learning/induction \citep[\eg early algorithms by][]{Michalski1969,ClarkNiblett1989,Frank1998} and associative classification \citep{Vanhoof10}.
Most papers in these fields aim to carefully guide the search through the space of models. If we were to place a restriction on our code to learn DNF's, which would require restricting predictions within the list to the positive class only, we could potentially use methods from rule learning and associative classification to help order CORELS' queue, which would in turn help us eliminate parts of the search space more quickly.

Some of our bounds, including the minimum support bound
(\S\ref{sec:lb-support}, Theorem~\ref{thm:min-capture}),
come from \citet{RudinEr16}, who provide flexible mixed-integer programming (MIP)
formulations using the same objective as we use here;
MIP solvers in general cannot compete with the speed of CORELS.

CORELS depends on pre-mined rules, which we obtain here via enumeration.
The literature on association rule mining is huge, and any method for rule mining could be reasonably substituted.

CORELS' main use is for producing interpretable predictive models. There is a growing interest in interpretable (transparent, comprehensible) models because of their societal importance \citep[see][]{ruping2006learning,bratko1997machine,dawes1979robust,VellidoEtAl12,Giraud98,Holte93,Schmueli10,Huysmans11,Freitas14}. There are now regulations on algorithmic decision-making in the European Union on the ``right to an explanation'' \citep{Goodman2016EU} that would legally require interpretability of predictions. There is work in both the DNF literature \citep{Ruckert2008} and decision tree literature \citep{GarofalakisHyRaSh00} on building interpretable models. Interpretable models must be so sparse that they need to be heavily optimized; heuristics tend to produce either inaccurate or non-sparse models.

Interpretability has many meanings, and it is possible to extend the ideas in this work to other definitions of interpretability; these rule lists may have exotic constraints that help with ease-of-use. For example, Falling Rule Lists \citep{WangRu15} are constrained to have decreasing probabilities down the list, which makes it easier to assess whether an observation is in a high risk subgroup. In parallel to this paper, we have been working on an algorithm for Falling Rule Lists \citep{ChenRu18} with bounds similar to those presented here, but even CORELS' basic support bounds do not hold for the falling case, which is much more complicated. One advantage of the approach taken by \citet{ChenRu18} is that it can handle class imbalance by weighting the positive and negative classes differently; this extension is possible in CORELS but not addressed here.

The models produced by CORELS are predictive only; they cannot be used for policy-making because they are not causal models, they do not include the costs of true and false positives, nor the cost of gathering information. It is possible to adapt CORELS' framework for causal inference \citep{WangRu15CFRL}, dynamic treatment regimes \citep{ZhangEtAl15}, or cost-sensitive dynamic treatment regimes \citep{LakkarajuRu17} to help with policy design.  CORELS could potentially be adapted to handle these kinds of interesting problems.

\section{Learning Optimal Rule Lists}
\label{sec:framework}

\begin{arxiv}
In this section, we present our framework for learning certifiably optimal rule lists.
First, we define our setting and useful notation~(\S\ref{sec:setup})
and then the objective function we seek to minimize~(\S\ref{sec:objective}).
Next, we describe the principal structure of our optimization algorithm~(\S\ref{sec:optimization}), which depends on a hierarchically
structured objective lower bound~(\S\ref{sec:hierarchical}).
We then derive a series of additional bounds that we incorporate into our
algorithm, because they enable aggressive pruning of our state space.
\end{arxiv}

\subsection{Notation}
\label{sec:setup}

\begin{arxiv}
We restrict our setting to binary classification,
where rule lists are Boolean functions;
this framework is straightforward to generalize to multi-class classification.
\end{arxiv}
\begin{kdd}
We restrict our setting to binary classification.
\end{kdd}
Let~${\{(x_n, y_n)\}_{n=1}^N}$ denote training data,
where ${x_n \in \{0, 1\}^J}$ are binary features and ${y_n \in \{0, 1\}}$ are labels.
Let~${\x = \{x_n\}_{n=1}^N}$ and~${\y = \{y_n\}_{n=1}^N}$,
and let~${x_{n,j}}$ denote the $j$-th feature of~$x_n$.

A rule list ${\RL = (r_1, r_2, \dots, r_K, r_0)}$ of length~${K \ge 0}$
is a ${(K+1)}$-tuple consisting of~$K$ distinct association rules,
${r_k = p_k \rightarrow q_k}$, for ${k = 1, \dots, K}$,
followed by a default rule~$r_0$.
\begin{arxiv}
Figure~\ref{fig:rule-list-symbols} illustrates
a rule list, ${\RL =}$ ${(r_1, r_2, r_3, r_0)}$,
which for clarity, we sometimes call a $K$-rule list.
\end{arxiv}
\begin{kdd}
Figure~\ref{fig:rule-list} illustrates a 3-rule list,
${\RL =}$ ${(r_1, r_2, r_3, r_0)}$.
\end{kdd}
An association rule~${r = p \rightarrow q}$ is an implication
corresponding to the conditional statement, ``if~$p$, then~$q$.''
In our setting, an antecedent~$p$ is a Boolean assertion that
evaluates to either true or false for each datum~$x_n$,
and a consequent~$q$ is a label prediction.
For example, ${(x_{n, 1} = 0) \wedge (x_{n, 3} = 1) \rightarrow}$ ${(y_n = 1)}$
is an association rule.
%
%The number of conditions in an antecedent is its cardinality;
%the antecedent in the previous example has a cardinality of two.
%
The final default rule~$r_0$ in a rule list can be thought of
as a
\begin{arxiv}
special
\end{arxiv}
association rule~${p_0 \rightarrow q_0}$
whose antecedent~$p_0$ simply asserts true.

\begin{arxiv}
\begin{figure}[t!]
\small
\begin{subfigure}{0.67\textwidth}
\begin{algorithmic}
\State \bif $(age=18-20) \band (sex=male)$ \bthen $yes$
\State \belif $(age=21-23)	 \band (priors=2-3)$ \bthen $yes$
\State \belif $(priors>3)$ \bthen $yes$
\State \belse $no$
\end{algorithmic}
\end{subfigure}
\hfill
\begin{subfigure}{0.32\textwidth}
\begin{algorithmic}
\State \bif $p_1$ \bthen $q_1$
\State \belif $p_2$ \bthen $q_2$
\State \belif $p_3$ \bthen $q_3$
\State \belse $q_0$
\end{algorithmic}
\end{subfigure}
\caption{The rule list ${\RL = (r_1, r_2, r_3, r_0)}$.
Each rule is of the form ${r_k = p_k \rightarrow q_k}$,
for all ${k = 0, \dots, 3}$.
We can also express this rule list as ${\RL = (\Prefix, \Labels, \Default, K)}$,
where ${\Prefix = (p_1, p_2, p_3)}$, ${\Labels = (1, 1, 1, 1)}$,
${\Default = 0}$, and ${K=3}$.
This is the same 3-rule list as in Figure~\ref{fig:rule-list},
that predicts two-year recidivism for the ProPublica data set.
}
\label{fig:rule-list-symbols}
\end{figure}
\end{arxiv}

Let ${\RL = (r_1, r_2, \dots, r_K, r_0)}$ be a
\begin{arxiv}
$K$-rule list,
\end{arxiv}
\begin{kdd}
rule list,
\end{kdd}
where ${r_k = p_k \rightarrow q_k}$ for each ${k = 0, \dots, K}$.
We introduce a useful alternate rule list representation:
${\RL = (\Prefix, \Labels, \Default, K)}$,
where we define ${\Prefix = (p_1, \dots, p_K)}$ to be $\RL$'s prefix,
${\Labels = (q_1, \dots, q_K) \in \{0, 1\}^K}$~gives
the label predictions associated with~$\Prefix$,
and ${\Default \in \{0, 1\}}$ is the default label prediction.
For
\begin{arxiv}
example, for
\end{arxiv}
the rule list in Figure~\ref{fig:rule-list},
we would write ${\RL = (\Prefix, \Labels, \Default, K)}$,
where ${\Prefix = (p_1, p_2, p_3)}$, ${\Labels = (1, 1, 1)}$,
${\Default = 0}$, and ${K=3}$.
Note that ${((), (), q_0, 0)}$ is a well-defined rule list with an empty prefix;
it is completely defined by a single default rule.

Let ${\Prefix = (p_1, \dots, p_k, \dots, p_K)}$ be an antecedent list,
then for any ${k \le K}$, we define ${\Prefix^k =}$ ${(p_1, \dots, p_k)}$
to be the $k$-prefix of~$\Prefix$.
For any such $k$-prefix~$\Prefix^k$,
we say that~$\Prefix$ starts with~$\Prefix^k$.
For any given space of rule lists,
we define~$\StartsWith(\Prefix)$ to be the set of
all rule lists whose prefixes start with~$\Prefix$:
\begin{align}
\StartsWith(\Prefix) =
\{(\Prefix', \Labels', \Default', K') : \Prefix' \textnormal{ starts with } \Prefix \}.
\label{eq:starts-with}
\end{align}
%We also say that an antecedent list~$\Prefix$ contains another
%antecedent list~$\Prefix'$ if the antecedents in~$\Prefix'$ correspond to
% a contiguous subsequence of antecedents anywhere in~$\Prefix$.
%
If ${\Prefix = (p_1, \dots, p_K)}$ and ${\Prefix' = (p_1, \dots, p_K, p_{K+1})}$
are two prefixes such that~$\Prefix'$ starts with~$\Prefix$ and extends it by
a single antecedent, we say that~$\Prefix$ is the parent of~$\Prefix'$
and that~$\Prefix'$ is a child of~$\Prefix$.

A rule list~$\RL$ classifies datum~$x_n$ by providing the label prediction~$q_k$
of the first rule~$r_k$ whose antecedent~$p_k$ is true for~$x_n$.
We say that an antecedent~$p_k$ of antecedent list~$\Prefix$ captures~$x_n$
in the context of~$\Prefix$ if~$p_k$ is the first antecedent in~$\Prefix$ that
evaluates to true for~$x_n$.
\begin{arxiv}
We also say that a
\end{arxiv}
\begin{kdd}
A
\end{kdd}
prefix captures those data captured by its antecedents;
for a rule list~${\RL = (\Prefix, \Labels, \Default, K)}$,
data not captured by the prefix~$\Prefix$
are classified according to the default label prediction~$\Default$.

Let~$\beta$ be a set of antecedents.
We define~${\Cap(x_n, \beta) = 1}$ if an antecedent in~$\beta$
captures datum~$x_n$, and~0 otherwise.
For example, let~$\Prefix$ and~$\Prefix'$ be prefixes such that~$\Prefix'$ starts
with~$\Prefix$, then~$\Prefix'$ captures all the data that~$\Prefix$ captures:
\begin{arxiv}
\begin{align*}
\{x_n: \Cap(x_n, \Prefix)\} \subseteq \{x_n: \Cap(x_n, \Prefix')\}.
%\label{eq:cap-subset}
\end{align*}
\end{arxiv}
\begin{kdd}
${\{x_n: \Cap(x_n, \Prefix)\} \subseteq \{x_n: \Cap(x_n, \Prefix')\}}$.
\end{kdd}
%We also define ${\Cap(\x, \beta) = \{\Cap(x_n, \beta)\}_{n=1}^N} \in \{0, 1\}^N$
%to be~$\beta$'s captures vector.

Now let~$\Prefix$ be an ordered list of antecedents,
and let~$\beta$ be a subset of antecedents in~$\Prefix$.
Let us define~${\Cap(x_n, \beta \given \Prefix) = 1}$ if~$\beta$
captures datum~$x_n$ in the context of~$\Prefix$,
\ie if the first antecedent in~$\Prefix$ that evaluates to true for~$x_n$
is an antecedent in~$\beta$, and~0 otherwise.
Thus, ${\Cap(x_n, \beta \given \Prefix) = 1}$ only if ${\Cap(x_n, \beta) = 1}$;
${\Cap(x_n, \beta \given \Prefix) = 0}$ either if ${\Cap(x_n, \beta) = 0}$,
or if ${\Cap(x_n, \beta) = 1}$ but there is an antecedent~$\alpha$ in~$\Prefix$,
preceding all antecedents in~$\beta$, such that ${\Cap(x_n, \alpha) = 1}$.
For example, if ${\Prefix = (p_1, \dots, p_k, \dots, p_K)}$ is a prefix, then
\begin{align*}
\Cap(x_n, p_k \given \Prefix) =
  \left(\bigwedge_{k'=1}^{k - 1} \neg\, \Cap(x_n, p_{k'}) \right)
  \wedge \Cap(x_n, p_k)
\end{align*}
indicates whether antecedent~$p_k$ captures datum~$x_n$ in the context of~$\Prefix$.
Now, define ${\Supp(\beta, \x)}$ to be the normalized support of~$\beta$,
\begin{align}
\Supp(\beta, \x) = \frac{1}{N} \sum_{n=1}^N \Cap(x_n, \beta),
\label{eq:support}
\end{align}
and similarly define~${\Supp(\beta, \x \given \Prefix)}$
to be the normalized support of~$\beta$ in the context of~$\Prefix$,
\begin{align}
\Supp(\beta, \x \given \Prefix) = \frac{1}{N} \sum_{n=1}^N \Cap(x_n, \beta \given \Prefix),
\label{eq:support-context}
\end{align}

Next, we address how empirical data constrains rule lists.
Given training data~${(\x, \y)}$,
an antecedent list ${\Prefix = (p_1, \dots, p_K)}$
implies a rule list ${\RL = (\Prefix, \Labels, \Default, K)}$
with prefix~$\Prefix$, where the label predictions
${\Labels = (q_1, \dots, q_K)}$ and~$\Default$ are empirically set
to minimize the number of misclassification errors made by
the rule list on the training data.
Thus for~${1 \le k \le K}$, label prediction~$q_k$ corresponds to the
majority label of data captured by antecedent~$p_k$ in the context of~$\Prefix$,
and the default~$\Default$ corresponds to the majority label of data
not captured by~$\Prefix$.
In the remainder of our presentation, whenever we refer to a rule list with a
particular prefix, we implicitly assume these empirically determined label predictions.

Our method is technically an associative classification method since it
leverages pre-mined rules.

\subsection{Objective Function}
\label{sec:objective}

\begin{arxiv}
We define
\end{arxiv}
\begin{kdd}
Define
\end{kdd}
a simple objective function for a rule list ${\RL = (\Prefix, \Labels, \Default, K)}$:
\begin{align}
\Obj(\RL, \x, \y) = \Loss(\RL, \x, \y) + \Reg K.
\label{eq:objective}
\end{align}
This objective function is a regularized empirical risk;
it consists of a loss~$\Loss(\RL, \x, \y)$, measuring misclassification error,
and a regularization term that penalizes longer rule lists.
$\Loss(\RL, \x, \y)$~is the fraction of training data whose labels are
incorrectly predicted by~$\RL$.
In our setting, the regularization parameter~${\Reg \ge 0}$ is a small constant;
\eg ${\Reg = 0.01}$ can be thought of as adding a penalty equivalent to misclassifying~$1\%$
of data when increasing a rule list's length
\begin{arxiv}
by one association rule.
\end{arxiv}
\begin{kdd}
by~one.
\end{kdd}
%
%As noted in~\S\ref{sec:setup}, a prefix~$\Prefix$ and training data together
%fully specify a rule list~${\RL = (\Prefix, \Labels, \Default, K)}$,
%thus let us define~${\Obj(\Prefix, \x, \y) \equiv \Obj(\RL, \x, \y)}$.

\subsection{Optimization Framework}
\label{sec:optimization}

Our objective has structure amenable to global optimization via a branch-and-bound framework.
In particular, we make a series of important observations, each of which translates into
a useful bound, and that together interact to eliminate large parts of the search~space.
We discuss these in depth in what follows:
\begin{itemize}
\item Lower bounds on a prefix also hold for every extension of that prefix.
(\S\ref{sec:hierarchical}, Theorem~\ref{thm:bound})

\item If a rule list is not accurate enough with respect to its length,
we can prune all extensions of it.
(\S\ref{sec:hierarchical}, Lemma~\ref{lemma:lookahead})

\item We can calculate \emph{a priori} an upper bound on the maximum length
of an optimal rule list.
(\S\ref{sec:ub-prefix-length}, Theorem~\ref{thm:ub-prefix-specific})

\item Each rule in an optimal rule list must have support that is sufficiently large.
%(Otherwise it would not be in an optimal rule list.)
%
This allows us to construct rule lists from frequent itemsets,
while preserving the guarantee that we can find a globally optimal
rule list from pre-mined rules.
(\S\ref{sec:lb-support}, Theorem~\ref{thm:min-capture})

\item Each rule in an optimal rule list must predict accurately.
In particular, the number of observations predicted correctly
by each rule in an optimal rule list must be above a threshold.
(\S\ref{sec:lb-support}, Theorem~\ref{thm:min-capture-correct})

\item We need only consider the optimal permutation of antecedents in a prefix;
we can omit all other permutations.
(\S\ref{sec:equivalent}, Theorem~\ref{thm:equivalent} and Corollary~\ref{thm:permutation})

\item  If multiple observations have identical features and opposite labels,
we know that any model will make mistakes.
In particular, the number of mistakes on these observations will be at least
the number of observations with the minority label.
(\S\ref{sec:identical}, Theorem~\ref{thm:identical})
\end{itemize}
\begin{kdd}
We present additional theorems and all proofs in~\citep{AngelinoLaAlSeRu17}.
\end{kdd}

\subsection{Hierarchical Objective Lower Bound}
\label{sec:hierarchical}

We can decompose the misclassification error
\begin{arxiv}
in~\eqref{eq:objective}
\end{arxiv}
into two contributions corresponding to the prefix and the default rule:
\begin{align*}
\Loss(\RL, \x, \y) %= \Loss(\Prefix, r_q, \Default, \x, \y)
\equiv \Loss_p(\Prefix, \Labels, \x, \y) + \Loss_0(\Prefix, \Default, \x, \y),
\end{align*}
where ${\Prefix = (p_1, \dots, p_K)}$ and ${\Labels = (q_1, \dots, q_K)}$;
\begin{align*}
\Loss_p(\Prefix, \Labels, \x, \y) =
\frac{1}{N} \sum_{n=1}^N \sum_{k=1}^K \Cap(x_n, p_k \given \Prefix) \wedge \one [ q_k \neq y_n ]
%\label{eq:loss}
\end{align*}
is the fraction of data captured and misclassified by the prefix, and
\begin{align*}
\Loss_0(\Prefix, \Default, \x, \y) =
\frac{1}{N} \sum_{n=1}^N \neg\, \Cap(x_n, \Prefix) \wedge \one [ \Default \neq y_n ]
\end{align*}
is the fraction of data not captured by the prefix and misclassified by the default rule.
Eliminating the latter error term gives a lower bound~$b(\Prefix, \x, \y)$ on the objective,
\begin{align}
b(\Prefix, \x, \y) \equiv \Loss_p(\Prefix, \Labels, \x, \y) + \Reg K \le \Obj(\RL, \x, \y),
\label{eq:lower-bound}
\end{align}
where we have suppressed the lower bound's dependence on label predictions~$\Labels$
because they are fully determined, given~${(\Prefix, \x, \y)}$.
Furthermore,
\begin{arxiv}
as we state next in Theorem~\ref{thm:bound},
\end{arxiv}
$b(\Prefix, \x, \y)$ gives a lower bound on the objective of
\emph{any} rule list whose prefix starts with~$\Prefix$.

\begin{theorem}[Hierarchical objective lower bound]
\begin{arxiv}
Define~${b(\Prefix, \x, \y)}$
\end{arxiv}
\begin{kdd}
Define~${b(\Prefix, \x, \y) = \Loss_p(\Prefix, \Labels, \x, \y) + \Reg K}$,
\end{kdd}
as in~\eqref{eq:lower-bound}.
Also, define $\StartsWith(\Prefix)$ to be the set of all rule lists
whose prefixes starts with~$\Prefix$, as in~\eqref{eq:starts-with}.
Let ${\RL = }$ ${(\Prefix, \Labels, \Default, K)}$ be a rule list
with prefix~$\Prefix$, and let
${\RL' = (\Prefix', \Labels', \Default', K')}$ $\in \StartsWith(\Prefix)$
be any rule list such that its prefix~$\Prefix'$ starts with~$\Prefix$
and ${K' \ge K}$, then ${b(\Prefix, \x, \y) \le}$ ${\Obj(\RL', \x, \y)}$.
\label{thm:bound}
\end{theorem}

\begin{arxiv}
\begin{proof}
Let ${\Prefix = (p_1, \dots, p_K)}$ and ${\Labels = (q_1, \dots, q_K)}$;
let ${\Prefix' = (p_1, \dots, p_K, p_{K+1}, \dots, p_{K'})}$
and ${\Labels' = (q_1, \dots, q_K, q_{K+1}, \dots, q_{K'})}$.
Notice that~$\Prefix'$ yields the same mistakes as~$\Prefix$,
and possibly additional mistakes:
\begin{align}
&\Loss_p(\Prefix', \Labels', \x, \y)
= \frac{1}{N} \sum_{n=1}^N  \sum_{k=1}^{K'} \Cap(x_n, p_k \given \Prefix') \wedge \one [ q_k \neq y_n ] \nn \\
&= \frac{1}{N} \sum_{n=1}^N \left( \sum_{k=1}^K \Cap(x_n, p_k \given \Prefix) \wedge \one [ q_k \neq y_n ]
+ \sum_{k=K+1}^{K'} \Cap(x_n, p_k \given \Prefix') \wedge \one [ q_k \neq y_n ] \right) \nn \\
&=\Loss_p(\Prefix, \Labels, \x, \y)
+ \frac{1}{N} \sum_{n=1}^N \sum_{k=K+1}^{K'} \Cap(x_n, p_k \given \Prefix') \wedge \one [ q_k \neq y_n ]
\ge \Loss_p(\Prefix, \Labels, \x, \y),
\label{eq:prefix-loss}
\end{align}
where in the second equality we have used the fact that
${\Cap(x_n, p_k \given \Prefix') = \Cap(x_n, p_k \given \Prefix)}$
for~${1 \le k \le K}$.
It follows that
\begin{align}
b(\Prefix, \x, \y) &= \Loss_p(\Prefix, \Labels, \x, \y) + \Reg K \nn \\
&\le  \Loss_p(\Prefix', \Labels', \x, \y) + \Reg K' = b(\Prefix', \x, \y)
\le \Obj(\RL', \x, \y).
\label{eq:prefix-lb}
\end{align}
\end{proof}
\end{arxiv}

To generalize, consider a sequence of prefixes such that each prefix
starts with all previous prefixes in the sequence.
It follows that the corresponding sequence of objective lower bounds
increases monotonically.
This is precisely the structure required and exploited by branch-and-bound,
illustrated in Algorithm~\ref{alg:branch-and-bound}.

\begin{algorithm}[t!]
\caption{Branch-and-bound for learning rule lists.}
\label{alg:branch-and-bound}
\begin{algorithmic}
\normalsize
\State \textbf{Input:} Objective function $\Obj(\RL, \x, \y)$,
objective lower bound ${b(\Prefix, \x, \y)}$,
set of antecedents ${\RuleSet = \{s_m\}_{m=1}^M}$,
training data $(\x, \y) = {\{(x_n, y_n)\}_{n=1}^N}$,
initial best known rule list~$\InitialRL$ with objective
${\InitialObj = \Obj(\InitialRL, \x, \y)}$;
$\InitialRL$ could be obtained as output from another (approximate) algorithm,
otherwise, $(\InitialRL, \InitialObj) = (\text{null}, 1)$ provide reasonable default values
\State \textbf{Output:} Provably optimal rule list~$\OptimalRL$ with minimum objective~$\OptimalObj$ \\

\State $(\CurrentRL, \CurrentObj) \gets (\InitialRL, \InitialObj)$ \Comment{Initialize best rule list and objective}
\State $Q \gets $ queue$(\,[\,(\,)\,]\,)$ \Comment{Initialize queue with empty prefix}
\While {$Q$ not empty} \Comment{Stop when queue is empty}
	\State $\Prefix \gets Q$.pop(\,) \Comment{Remove prefix~$\Prefix$ from the queue}
	\begin{arxiv}
	\State $\RL \gets (\Prefix, \Labels, \Default, K)$ \Comment{Set label predictions~$\Labels$ and~$\Default$ to minimize training error}
	\end{arxiv}
	\If {$b(\Prefix, \x, \y) < \CurrentObj$} \Comment{\textbf{Bound}: Apply Theorem~\ref{thm:bound}}
        \State $\Obj \gets \Obj(\RL, \x, \y)$ \Comment{Compute objective of~$\Prefix$'s rule list~$\RL$}
        \If {$\Obj < \CurrentObj$} \Comment{Update best rule list and objective}
            \State $(\CurrentRL, \CurrentObj) \gets (\RL, \Obj)$
        \EndIf
        \For {$s$ in $\RuleSet$} \Comment{\textbf{Branch}: Enqueue~$\Prefix$'s children}
            \If {$s$ not in $\Prefix$}
                \State $Q$.push$(\,(\Prefix, s)\,)$
            \EndIf
        \EndFor
    \EndIf
\EndWhile
\State $(\OptimalRL, \OptimalObj) \gets (\CurrentRL, \CurrentObj)$ \Comment{Identify provably optimal solution}
\end{algorithmic}
\end{algorithm}

Specifically, the objective lower bound in Theorem~\ref{thm:bound}
enables us to prune the state space hierarchically.
While executing branch-and-bound, we keep track of the current best (smallest)
objective~$\CurrentObj$, thus it is a dynamic, monotonically decreasing quantity.
If we encounter a prefix~$\Prefix$ with lower bound
${b(\Prefix, \x, \y) \ge \CurrentObj}$,
then by Theorem~\ref{thm:bound}, we do not need to consider \emph{any}
rule list~${\RL' \in \StartsWith(\Prefix)}$ whose prefix~$\Prefix'$ starts with~$\Prefix$.
%because ${b(\Prefix', \x, \y) \ge b(\Prefix, \x, \y)}$. \\
%
For the objective of such a rule list, the current best objective
provides a lower bound, \ie
${\Obj(\RL', \x, \y) \ge b(\Prefix', \x, \y) \ge}$ ${b(\Prefix, \x, \y) \ge \CurrentObj}$,
and thus~$\RL'$ cannot be optimal.

Next, we state an immediate consequence of Theorem~\ref{thm:bound}.

\begin{lemma}[Objective lower bound with one-step lookahead]
\label{lemma:lookahead}
Let~$\Prefix$ be a $K$-prefix
and let~$\CurrentObj$ be the current best objective.
If ${b(\Prefix, \x, \y) + \Reg \ge \CurrentObj}$,
then for any $K'$-rule list ${\RL' \in \StartsWith(\Prefix)}$
whose prefix~$\Prefix'$ starts with~$\Prefix$ and~${K' > K}$,
it follows that ${\Obj(\RL', \x, \y) \ge \CurrentObj}$.
\end{lemma}

\begin{arxiv}
\begin{proof}
By the definition of the lower bound~\eqref{eq:lower-bound},
which includes the penalty for longer prefixes,
\begin{align}
\Obj(\Prefix', \x, y) \ge b(\Prefix', \x, \y) &= \Loss_p(\Prefix', \Labels', \x, \y) + \Reg K' \nn \\
&= \Loss_p(\Prefix', \Labels', \x, \y) + \Reg K + \Reg (K' - K) \nn \\
&= b(\Prefix, \x, \y) + \Reg (K' - K)
\ge b(\Prefix, \x, \y) + \Reg \ge \CurrentObj.
%\label{eq:lookahead}
\end{align}
\end{proof}
\end{arxiv}

Therefore, even if we encounter a prefix~$\Prefix$
with lower bound ${b(\Prefix, \x, \y) \le \CurrentObj}$,
\begin{kdd}
if
\end{kdd}
\begin{arxiv}
as long as
\end{arxiv}
${b(\Prefix, \x, \y) + \Reg \ge \CurrentObj}$, then we can prune
all prefixes~$\Prefix'$ that start with and are longer than~$\Prefix$.

\subsection{Upper Bounds on Prefix Length}
\label{sec:ub-prefix-length}

\begin{arxiv}
In this section, we derive several upper bounds on prefix length:
\begin{itemize}
\item The simplest upper bound on prefix length is given by the total
number of available antecedents. (Proposition~\ref{prop:trivial-length})
\item The current best objective~$\CurrentObj$ implies an upper bound
on prefix length. (Theorem~\ref{thm:ub-prefix-length})
\item For intuition, we state a version of the above bound that is valid
at the start of execution. (Corollary~\ref{cor:ub-prefix-length})
\item By considering specific families of prefixes,
we can obtain tighter bounds on prefix length. (Theorem~\ref{thm:ub-prefix-specific})
\end{itemize}
In the next section (\S\ref{sec:ub-size}), we use these results
to derive corresponding upper bounds on the number of
prefix evaluations made by Algorithm~\ref{alg:branch-and-bound}.

\begin{proposition}[Trivial upper bound on prefix length]
\label{prop:trivial-length}
Consider a state space of all rule lists formed from
a set of~$M$ antecedents,
and let~$L(\RL)$ be the length of rule list~$\RL$.
$M$ provides an upper bound on the length of
any optimal rule list
${\OptimalRL \in \argmin_\RL \Obj(\RL, \x, \y)}$,
\ie ${L(\RL) \le M}$.
\end{proposition}

\begin{proof}
Rule lists consist of distinct rules by definition.
\end{proof}
\end{arxiv}

At any point during branch-and-bound execution, the current best objective~$\CurrentObj$
implies an upper bound on the maximum prefix length we might still have to consider.
\begin{theorem}[Upper bound on prefix length]
\label{thm:ub-prefix-length}
Consider a state space of all rule lists formed from a set of~$M$ antecedents.
Let~$L(\RL)$ be the length of rule list~$\RL$
and let~$\CurrentObj$ be the current best objective.
For all optimal rule lists ${\OptimalRL \in \argmin_\RL \Obj(\RL, \x, \y)}$
\begin{arxiv}
\begin{align}
L(\OptimalRL) \le \min \left(\left\lfloor \frac{\CurrentObj}{\Reg} \right\rfloor, M \right),
\label{eq:max-length}
\end{align}
\end{arxiv}
\begin{kdd}
\begin{align}
L(\OptimalRL) \le \min \left(\left\lfloor \CurrentObj / \Reg \right\rfloor, M \right),
\label{eq:max-length}
\end{align}
\end{kdd}
where~$\Reg$ is the regularization parameter.
\begin{arxiv}
Furthermore, if~$\CurrentRL$ is a rule list with
objective ${\Obj(\CurrentRL, \x, \y) = \CurrentObj}$,
length~$K$, and zero misclassification error,
then for every optimal rule list
${\OptimalRL \in}$ ${\argmin_\RL \Obj(\RL, \x, \y)}$,
if ${\CurrentRL \in \argmin_d \Obj(\RL, \x, \y)}$,
then ${L(\OptimalRL) \le K}$,
or otherwise if ${\CurrentRL \notin \argmin_d \Obj(\RL, \x, \y)}$,
then ${L(\OptimalRL) \le K - 1}$.
\end{arxiv}
\end{theorem}

\begin{arxiv}
\begin{proof}
For an optimal rule list~$\OptimalRL$ with objective~$\OptimalObj$,
\begin{align}
\Reg L(\OptimalRL) \le \OptimalObj = \Obj(\OptimalRL, \x, \y)
= \Loss(\OptimalRL, \x, \y) + \Reg L(\OptimalRL)
\le \CurrentObj. \nn
\end{align}
The maximum possible length for~$\OptimalRL$ occurs
when~$\Loss(\OptimalRL, \x, \y)$ is minimized;
combining with Proposition~\ref{prop:trivial-length}
gives bound~\eqref{eq:max-length}.

For the rest of the proof,
let~${K^* = L(\OptimalRL)}$ be the length of~$\OptimalRL$.
If the current best rule list~$\CurrentRL$ has zero
misclassification error, then
\begin{align}
\Reg K^* \leq \Loss(\OptimalRL, \x, \y) + \Reg K^* = \Obj(\OptimalRL, \x, \y)
\le \CurrentObj = \Obj(\CurrentRL, \x, \y) = \Reg K, \nn
\end{align}
and thus ${K^* \leq K}$.
If the current best rule list is suboptimal,
\ie ${\CurrentRL \notin \argmin_\RL \Obj(\RL, \x, \y)}$, then
\begin{align}
\Reg K^* \leq \Loss(\OptimalRL, \x, \y) + \Reg K^* = \Obj(\OptimalRL, \x, \y)
< \CurrentObj = \Obj(\CurrentRL, \x, \y) = \Reg K, \nn
\end{align}
in which case ${K^* < K}$, \ie ${K^* \leq K-1}$, since $K$ is an integer.
\end{proof}

The latter part of Theorem~\ref{thm:ub-prefix-length} tells us that
if we only need to identify a single instance of an optimal rule list
${\OptimalRL \in \argmin_\RL \Obj(\RL, \x, \y)}$, and we encounter a perfect
$K$-rule list with zero misclassification error, then we can prune all
prefixes of length~$K$ or greater.

\end{arxiv}

\begin{corollary}[Simple upper bound on prefix length]
\label{cor:ub-prefix-length}
\begin{arxiv}
Let~$L(\RL)$ be the length of rule list~$\RL$.
\end{arxiv}
For all optimal rule lists ${\OptimalRL \in \argmin_\RL \Obj(\RL, \x, \y)}$,
\begin{arxiv}
\begin{align}
L(\OptimalRL) \le \min \left( \left\lfloor \frac{1}{2\Reg} \right\rfloor, M \right).
\label{eq:max-length-trivial}
\end{align}
\end{arxiv}
\begin{kdd}
\begin{align}
L(\OptimalRL) \le \min \left( \left\lfloor 1 / 2\Reg \right\rfloor, M \right).
\label{eq:max-length-trivial}
\end{align}
\end{kdd}
\end{corollary}

\begin{arxiv}
\begin{proof}
Let ${d = ((), (), q_0, 0)}$ be the empty rule list;
it has objective ${\Obj(\RL, \x, \y) = \Loss(\RL, \x, \y) \le}$ ${1/2}$,
which gives an upper bound on~$\CurrentObj$.
Combining with~\eqref{eq:max-length}
and Proposition~\ref{prop:trivial-length}
gives~\eqref{eq:max-length-trivial}.
\end{proof}
\end{arxiv}

For any particular prefix~$\Prefix$, we can obtain potentially tighter
upper bounds on prefix length for
\begin{arxiv}
the family of
\end{arxiv}
all prefixes that start with~$\Prefix$.

\begin{theorem}[Prefix-specific upper bound on prefix length]
\label{thm:ub-prefix-specific}
Let ${\RL = (\Prefix, \Labels, \Default, K)}$ be a rule list, let
${\RL' = (\Prefix', \Labels', \Default', K') \in \StartsWith(\Prefix)}$
be any rule list such that~$\Prefix'$ starts with~$\Prefix$,
and let~$\CurrentObj$ be the current best objective.
If~$\Prefix'$ has lower bound~${b(\Prefix', \x, \y) < \CurrentObj}$, then
\begin{align}
K' < \min \left( K + \left\lfloor \frac{\CurrentObj - b(\Prefix, \x, \y)}{\Reg} \right\rfloor, M \right).
\label{eq:max-length-prefix}
\end{align}
\end{theorem}

\begin{arxiv}
\begin{proof}
First, note that~${K' \ge K}$, since~$\Prefix'$ starts with~$\Prefix$.
Now recall from~\eqref{eq:prefix-lb} that
\begin{align}
b(\Prefix, \x, \y) = \Loss_p(\Prefix, \Labels, \x, \y) + \Reg K
\le \Loss_p(\Prefix', \Labels', \x, \y) + \Reg K' = b(\Prefix', \x, \y), \nn
\end{align}
and from~\eqref{eq:prefix-loss} that
${\Loss_p(\Prefix, \Labels, \x, \y) \le \Loss_p(\Prefix', \Labels', \x, \y)}$.
Combining these bounds and rearranging gives
\begin{align}
b(\Prefix', \x, \y) &= \Loss_p(\Prefix', \Labels', \x, \y) + \Reg K + \Reg(K' - K) \nn \\
&\ge \Loss_p(\Prefix, \Labels, \x, \y) + \Reg K + \Reg(K' - K)
= b(\Prefix, \x, \y) + \Reg (K' - K).
\label{eq:length-diff}
\end{align}
Combining~\eqref{eq:length-diff} with~${b(\Prefix', \x, \y) < \CurrentObj}$
and Proposition~\ref{prop:trivial-length} gives~\eqref{eq:max-length-prefix}.
\end{proof}
\end{arxiv}

We can view Theorem~\ref{thm:ub-prefix-specific} as a generalization
of our one-step lookahead bound (Lemma~\ref{lemma:lookahead}),
as~\eqref{eq:max-length-prefix} is equivalently a bound on ${K' - K}$,
an upper bound on the number of remaining `steps' corresponding to
an iterative sequence of single-rule extensions of a prefix~$\Prefix$.
\begin{arxiv}
Notice that when~${\RL = ((), (), q_0, 0)}$ is the empty rule list,
this bound replicates~\eqref{eq:max-length}, since ${b(\Prefix, \x, \y) = 0}$.
\end{arxiv}

\begin{arxiv}
\subsection{Upper Bounds on the Number of Prefix Evaluations}
\end{arxiv}
\begin{kdd}
\subsection{Upper bounds on prefix evaluations}
\end{kdd}
\label{sec:ub-size}

\begin{arxiv}
In this section, we use our upper bounds on prefix length
from~\S\ref{sec:ub-prefix-length} to derive corresponding
upper bounds on the number of prefix evaluations made by
Algorithm~\ref{alg:branch-and-bound}.
First, we
\end{arxiv}
\begin{kdd}
In this section, we use Theorem~\ref{thm:ub-prefix-specific}'s
upper bound on prefix length to derive a corresponding
upper bound on the number of prefix evaluations made by
Algorithm~\ref{alg:branch-and-bound}.
We
\end{kdd}
present Theorem~\ref{thm:remaining-eval-fine},
in which we use information about the state of
Algorithm~\ref{alg:branch-and-bound}'s execution
to calculate, for any given execution state,
upper bounds on the number of additional prefix evaluations that might
be required for the execution to complete.
The relevant execution state depends on the current
best objective~$\CurrentObj$ and information about
prefixes we are planning to evaluate, \ie prefixes in the
queue~$\Queue$ of Algorithm~\ref{alg:branch-and-bound}.
We define the number of \emph{remaining prefix evaluations} as the number of
prefixes that are currently in or will be inserted into the queue.

\begin{arxiv}
We use Theorem~\ref{thm:remaining-eval-fine} in some of our empirical results
(\S\ref{sec:experiments}, Figure~\ref{fig:objective}) to help illustrate
the dramatic impact of certain algorithm optimizations.
The execution trace of this upper bound on remaining prefix evaluations
complements the execution traces of other quantities,
\eg that of the current best objective~$\CurrentObj$.
After presenting Theorem~\ref{thm:remaining-eval-fine}, we also give two
weaker propositions that provide useful intuition.
In particular, Proposition~\ref{prop:remaining-eval-coarse} is a practical
approximation to Theorem~\ref{thm:remaining-eval-fine} that is significantly
easier to compute; we use it in our implementation as a metric of
execution progress that we display to the user.
\end{arxiv}

\begin{arxiv}
\begin{theorem}[Fine-grained upper bound on remaining prefix evaluations]~
\end{arxiv}
\begin{kdd}
\begin{theorem}[Upper bound on the number of remaining prefix evaluations]
\end{kdd}
\label{thm:remaining-eval-fine}
~Consider the state space of all rule lists formed from a set of~$M$ antecedents,
and consider Algorithm~\ref{alg:branch-and-bound} at a particular instant
during execution.
Let~$\CurrentObj$ be the current best objective, let~$\Queue$ be the queue,
and let~$L(\Prefix)$ be the length of prefix~$\Prefix$.
Define~${\Remaining(\CurrentObj, \Queue)}$ to be the number of remaining
prefix evaluations, then
\begin{align}
\Remaining(\CurrentObj, \Queue)
\le \sum_{\Prefix \in Q} \sum_{k=0}^{f(\Prefix)} \frac{(M - L(\Prefix))!}{(M - L(\Prefix) - k)!},
\label{eq:remaining}
\end{align}
\begin{arxiv}
where
\begin{align}
f(\Prefix) = \min \left( \left\lfloor
  \frac{\CurrentObj - b(\Prefix, \x, \y)}{\Reg} \right\rfloor, M - L(\Prefix)\right). \nn
%\label{eq:f}
\end{align}
\end{arxiv}
\begin{kdd}
\begin{align}
\text{where} \quad f(\Prefix) = \min \left( \left\lfloor
  \frac{\CurrentObj - b(\Prefix, \x, \y)}{\Reg} \right\rfloor, M - L(\Prefix)\right).
%\label{eq:f}
\end{align}
\end{kdd}
\end{theorem}

\begin{proof}
The number of remaining prefix evaluations is equal to the number of
prefixes that are currently in or will be inserted into queue~$\Queue$.
For any such prefix~$\Prefix$, Theorem~\ref{thm:ub-prefix-specific}
gives an upper bound on the length of any prefix~$\Prefix'$
that starts with~$\Prefix$:
\begin{align}
L(\Prefix') \le \min \left( L(\Prefix) + \left\lfloor \frac{\CurrentObj - b(\Prefix, \x, \y)}{\Reg} \right\rfloor, M \right)
\equiv U(\Prefix). \nn
\end{align}
This gives an upper bound on the number of remaining prefix evaluations:
\begin{arxiv}
\begin{align}
\Remaining(\CurrentObj, \Queue)
\le \sum_{\Prefix \in Q} \sum_{k=0}^{U(\Prefix) - L(\Prefix)} P(M - L(\Prefix), k)
= \sum_{\Prefix \in Q} \sum_{k=0}^{f(\Prefix)} \frac{(M - L(\Prefix))!}{(M - L(\Prefix) - k)!}\,, \nn
\end{align}
where~$P(m, k)$ denotes the number of $k$-permutations of~$m$.
\end{arxiv}
\begin{kdd}
${\Remaining(\CurrentObj, \Queue)
\le \sum_{\Prefix \in Q} \sum_{k=0}^{U(\Prefix) - L(\Prefix)} P(M - L(\Prefix), k)}$.
\end{kdd}
\end{proof}

\begin{arxiv}
Proposition~\ref{thm:ub-total-eval} is strictly weaker than
Theorem~\ref{thm:remaining-eval-fine} and is the starting point for its derivation.
It
\end{arxiv}
\begin{kdd}
The proposition below
\end{kdd}
is a na\"ive upper bound on
the total number of prefix evaluations over the course of
Algorithm~\ref{alg:branch-and-bound}'s execution.
It only depends on the number of rules and
the regularization parameter~$\Reg$;
\ie unlike Theorem~\ref{thm:remaining-eval-fine},
it does not use algorithm execution state to
bound the size of the search space.

\begin{arxiv}
\begin{proposition}[Upper bound on the total number of prefix evaluations]
\end{arxiv}
\begin{kdd}
\begin{proposition}[Upper bound on the total number of prefix evaluations]
\end{kdd}
\label{thm:ub-total-eval}
~Define $\TotalRemaining(\RuleSet)$ to be the total number of prefixes
evaluated by Algorithm~\ref{alg:branch-and-bound}, given the state space of
all rule lists formed from a set~$\RuleSet$ of~$M$ rules.
For any set~$\RuleSet$ of $M$ rules,
\begin{arxiv}
\begin{align}
\TotalRemaining(\RuleSet) \le \sum_{k=0}^K \frac{M!}{(M - k)!}, \nn
%\label{eq:size-naive}
\end{align}
where ${K = \min(\lfloor 1/2 \Reg \rfloor, M)}$.
\end{arxiv}
\begin{kdd}
\begin{align}
\TotalRemaining(\RuleSet) \le \sum_{k=0}^K \frac{M!}{(M - k)!},
\quad \text{where} \quad K = \min(\lfloor 1/2 \Reg \rfloor, M).
\label{eq:size-naive}
\end{align}
\end{kdd}
\end{proposition}

\begin{proof}
By Corollary~\ref{cor:ub-prefix-length},
${K \equiv \min(\lfloor 1 / 2 \Reg \rfloor, M)}$
gives an upper bound on the length of any optimal rule list.
\begin{arxiv}
Since we can think of our problem as finding the optimal
selection and permutation of~$k$ out of~$M$ rules,
over all~${k \le K}$,
\begin{align}
\TotalRemaining(\RuleSet) \le 1 + \sum_{k=1}^K P(M, k)
= \sum_{k=0}^K \frac{M!}{(M - k)!}. \nn
\end{align}
\end{arxiv}
\begin{kdd}
We obtain~\eqref{eq:size-naive} by viewing
our problem as finding the optimal
selection and permutation of~$k$ out of~$M$ rules,
over all~${k \le K}$.
\end{kdd}
\end{proof}

\begin{arxiv}

Our next upper bound is strictly tighter than the bound in
Proposition~\ref{thm:ub-total-eval}.
Like Theorem~\ref{thm:remaining-eval-fine}, it uses the
current best objective and information about
the lengths of prefixes in the queue to constrain
the lengths of prefixes in the remaining search space.
However, Proposition~\ref{prop:remaining-eval-coarse}
is weaker than Theorem~\ref{thm:remaining-eval-fine} because
it leverages only coarse-grained information from the queue.
Specifically, Theorem~\ref{thm:remaining-eval-fine} is
strictly tighter because it additionally incorporates
prefix-specific objective lower bound information from
prefixes in the queue, which further constrains
the lengths of prefixes in the remaining search space.

\begin{proposition}[Coarse-grained upper bound on remaining prefix evaluations] \hfill
\label{prop:remaining-eval-coarse}
Consider a state space of all rule lists formed from a set of~$M$ antecedents,
and consider Algorithm~\ref{alg:branch-and-bound} at a particular instant
during execution.
Let~$\CurrentObj$ be the current best objective, let~$\Queue$ be the queue,
and let~$L(\Prefix)$ be the length of prefix~$\Prefix$.
Let~$\Queue_j$ be the number of prefixes of length~$j$ in~$\Queue$,
\begin{align}
\Queue_j = \big | \{ \Prefix : L(\Prefix) = j, \Prefix \in \Queue \} \big | \nn
\end{align}
and let~${J = \argmax_{\Prefix \in \Queue} L(\Prefix)}$
be the length of the longest prefix in~$\Queue$.
Define~${\Remaining(\CurrentObj, \Queue)}$ to be the number of remaining
prefix evaluations, then
\begin{align}
\Remaining(\CurrentObj, \Queue)
\le \sum_{j=1}^J \Queue_j \left( \sum_{k=0}^{K-j} \frac{(M-j)!}{(M-j - k)!} \right), \nn
\end{align}
where~${K = \min(\lfloor \CurrentObj / \Reg \rfloor, M)}$.
\end{proposition}

\begin{proof}
The number of remaining prefix evaluations is equal to the number of
prefixes that are currently in or will be inserted into queue~$\Queue$.
For any such remaining prefix~$\Prefix$,
Theorem~\ref{thm:ub-prefix-length} gives an upper bound on its length;
define~$K$ to be this bound:
${L(\Prefix) \le \min(\lfloor \CurrentObj / \Reg \rfloor, M) \equiv K}$.
For any prefix~$\Prefix$ in queue~$\Queue$ with length~${L(\Prefix) = j}$,
the maximum number of prefixes that start with~$\Prefix$
and remain to be evaluated is:
\begin{align}
\sum_{k=0}^{K-j} P(M-j, k) = \sum_{k=0}^{K-j} \frac{(M-j)!}{(M-j - k)!}, \nn
\end{align}
where~${P(T, k)}$ denotes the number of $k$-permutations of~$T$.
This gives an upper bound on the number of remaining prefix evaluations:
\begin{align}
\Remaining(\CurrentObj, \Queue)
\le \sum_{j=0}^J \Queue_j \left( \sum_{k=0}^{K-j} P(M-j, k) \right)
= \sum_{j=0}^J \Queue_j \left( \sum_{k=0}^{K-j} \frac{(M-j)!}{(M-j - k)!} \right). \nn
\end{align}
\end{proof}
\end{arxiv}

\subsection{Lower Bounds on Antecedent Support}
\label{sec:lb-support}

In this section, we give two lower bounds on the normalized support
of each antecedent in any optimal rule list;
both are related to the regularization parameter~$\Reg$.

\begin{theorem}[Lower bound on antecedent support]
\label{thm:min-capture}
\begin{arxiv}
~Let ${\OptimalRL = (\Prefix, \Labels, \Default, K)}$
be any optimal rule list with objective~$\OptimalObj$, \ie
${\OptimalRL \in \argmin_\RL \Obj(\RL, \x, \y)}$.
\end{arxiv}
\begin{kdd}
Let ${\OptimalRL = (\Prefix, \Labels, \Default, K) \in \argmin_\RL \Obj(\RL, \x, \y)}$
be any optimal rule list, with objective~$\OptimalObj$.
\end{kdd}
For each antecedent~$p_k$ in prefix ${\Prefix = (p_1, \dots, p_K)}$,
\begin{arxiv}
the regularization parameter~$\Reg$ provides a lower bound
on the normalized support of~$p_k$,
\begin{align}
\Reg \le \Supp(p_k, \x \given \Prefix).
\label{eq:min-capture}
\end{align}
\end{arxiv}
\begin{kdd}
the regularization parameter provides a lower bound,
${\Reg \le \Supp(p_k, \x \given \Prefix)}$, on the normalized support of~$p_k$.
\end{kdd}
\end{theorem}

\begin{arxiv}
\begin{proof}
Let ${\OptimalRL = (\Prefix, \Labels, \Default, K)}$ be an optimal
rule list with prefix ${\Prefix = (p_1, \dots, p_K)}$
and labels ${\Labels = (q_1, \dots, q_K)}$.
Consider the rule list ${\RL = (\Prefix', \Labels', \Default', K-1)}$
derived from~$\OptimalRL$ by deleting a rule ${p_i \rightarrow q_i}$,
therefore ${\Prefix' = (p_1, \dots, p_{i-1}, p_{i+1}, \dots, p_K)}$
and ${\Labels' = (q_1, \dots, q_{i-1},}$ ${q'_{i+1}, \dots, q'_K)}$,
where~$q'_k$ need not be the same as~$q_k$, for ${k > i}$ and~${k = 0}$.

The largest possible discrepancy between~$\OptimalRL$ and~$\RL$ would occur
if~$\OptimalRL$ correctly classified all the data captured by~$p_i$,
while~$\RL$ misclassified these data.
This gives an upper bound:
\begin{align}
\Obj(\RL, \x, \y) = \Loss(\RL, \x, \y) + \Reg (K - 1)
&\le \Loss(\OptimalRL, \x, \y) + \Supp(p_i, \x \given \Prefix) + \Reg(K - 1) \nn \\
&= \Obj(\OptimalRL, \x, \y) + \Supp(p_i, \x \given \Prefix) - \Reg \nn \\
&= \OptimalObj + \Supp(p_i, \x \given \Prefix) - \Reg
\label{eq:ub-i}
\end{align}
where~$\Supp(p_i, \x \given \Prefix)$ is the normalized support of~$p_i$
in the context of~$\Prefix$, defined in~\eqref{eq:support-context},
and the regularization `bonus' comes from the fact that~$\RL$
is one rule shorter than~$\OptimalRL$.

At the same time, we must have ${\OptimalObj \le \Obj(\RL, \x, \y)}$ for~$\OptimalRL$ to be optimal.
Combining this with~\eqref{eq:ub-i} and rearranging gives~\eqref{eq:min-capture},
therefore the regularization parameter~$\Reg$ provides a lower bound
on the support of an antecedent~$p_i$ in an optimal rule list~$\OptimalRL$.
\end{proof}
\end{arxiv}

Thus, we can prune a prefix~$\Prefix$ if any of its antecedents captures less than
a fraction~$\Reg$ of data, even if~${b(\Prefix, \x, \y) < \OptimalObj}$.
\begin{arxiv}
Notice that the
\end{arxiv}
\begin{kdd}
The
\end{kdd}
bound in Theorem~\ref{thm:min-capture}
depends on the antecedents, but not the label predictions,
and thus does not account for misclassification error.
\begin{arxiv}
Theorem~\ref{thm:min-capture-correct} gives a tighter bound
by leveraging this additional information, which specifically
tightens the upper bound on~$\Obj(\RL, \x, \y)$ in~\eqref{eq:ub-i}.
\end{arxiv}
\begin{kdd}
Theorem~\ref{thm:min-capture-correct} gives a tighter bound
by leveraging this information.
\end{kdd}

\begin{theorem}[Lower bound on accurate antecedent support]
\label{thm:min-capture-correct}
\begin{arxiv}
Let ${\OptimalRL}$
be any optimal rule list with objective~$\OptimalObj$, \ie
${\OptimalRL = (\Prefix, \Labels, \Default, K) \in \argmin_\RL \Obj(\RL, \x, \y)}$.
Let $\OptimalRL$ have prefix ${\Prefix = (p_1, \dots, p_K)}$
and labels ${\Labels = (q_1, \dots, q_K)}$.
\end{arxiv}
\begin{kdd}
Let ${\OptimalRL \in \argmin_\RL \Obj(\RL, \x, \y)}$
be any optimal rule list, with objective~$\OptimalObj$;
let ${\OptimalRL = (\Prefix, \Labels, \Default, K)}$,
with prefix ${\Prefix = (p_1, \dots, p_K)}$
and labels ${\Labels = (q_1, \dots, q_K)}$.
\end{kdd}
For each rule~${p_k \rightarrow q_k}$ in~$\OptimalRL$,
define~$a_k$ to be the fraction of data that are captured by~$p_k$
and correctly classified:
\begin{align}
a_k \equiv \frac{1}{N} \sum_{n=1}^N
  \Cap(x_n, p_k \given \Prefix) \wedge \one [ q_k = y_n ].
\label{eq:rule-correct}
\end{align}
\begin{arxiv}
The regularization parameter~$\Reg$ provides a lower bound on~$a_k$:
\begin{align}
\Reg \le a_k.
\label{eq:min-capture-correct}
\end{align}
\end{arxiv}
\begin{kdd}
The regularization parameter provides a lower bound, $\Reg \le a_k$.
\end{kdd}
\end{theorem}

\begin{arxiv}
\begin{proof}
As in Theorem~\ref{thm:min-capture},
let ${\RL =  (\Prefix', \Labels', \Default', K-1)}$ be the rule list
derived from~$\OptimalRL$ by deleting a rule~${p_i \rightarrow q_i}$.
Now, let us define~$\Loss_i$ to be the portion of~$\OptimalObj$
due to this rule's misclassification error,
\begin{align}
\Loss_i \equiv \frac{1}{N} \sum_{n=1}^N
  \Cap(x_n, p_i \given \Prefix) \wedge \one [ q_i \neq y_n ]. \nn
\end{align}
The largest discrepancy between~$\OptimalRL$ and~$\RL$ would
occur if~$\RL$ misclassified all the data captured by~$p_i$.
This gives an upper bound on the difference between
the misclassification error of~$\RL$ and~$\OptimalRL$:
\begin{align}
\Loss(\RL, \x, \y) - \Loss(\OptimalRL, \x, \y)
&\le \Supp(p_i, \x \given \Prefix) - \Loss_i \nn \\
&= \frac{1}{N} \sum_{n=1}^N \Cap(x_n, p_i \given \Prefix)
  - \frac{1}{N} \sum_{n=1}^N
  \Cap(x_n, p_i \given \Prefix) \wedge \one [ q_i \neq y_n ] \nn \\
&= \frac{1}{N} \sum_{n=1}^N
  \Cap(x_n, p_i \given \Prefix) \wedge \one [ q_i = y_n ] = a_i, \nn
\end{align}
where we defined~$a_i$ in~\eqref{eq:rule-correct}.
Relating this bound to the objectives of~$\RL$ and~$\OptimalRL$ gives
\begin{align}
\Obj(\RL, \x, \y) = \Loss(\RL, \x, \y) + \Reg (K - 1)
&\le \Loss(\OptimalRL, \x, \y) + a_i + \Reg(K - 1) \nn \\
&= \Obj(\OptimalRL, \x, \y) + a_i - \Reg \nn \\
&= \OptimalObj + a_i - \Reg.
\label{eq:ub-ii}
\end{align}
Combining~\eqref{eq:ub-ii} with the requirement
${\OptimalObj \le \Obj(\RL, \x, \y)}$ gives the bound~${\Reg \le a_i}$.
\end{proof}
\end{arxiv}

Thus, we can prune a prefix if any of its rules correctly classifies
less than a fraction~$\Reg$ of data.
While the lower bound in Theorem~\ref{thm:min-capture} is a sub-condition
of the lower bound in Theorem~\ref{thm:min-capture-correct},
we can still leverage both---since the sub-condition is easier to check,
checking it first can accelerate pruning.
In addition to applying Theorem~\ref{thm:min-capture} in the context of
constructing rule lists, we can furthermore apply it in the context of
rule mining~(\S\ref{sec:setup}).
Specifically, it implies that we should only mine rules with
normalized support of at least~$\Reg$;
we need not mine rules with a smaller fraction of
\begin{arxiv}
observations.\footnote{We describe our application of this idea in
Appendix~\ref{appendix:data}, where we provide details on data processing.}
\end{arxiv}
\begin{kdd}
observations.\footnote{We describe our application of this idea in the appendix
of our long report~\citep{AngelinoLaAlSeRu17}.}
\end{kdd}
In contrast, we can only apply Theorem~\ref{thm:min-capture-correct}
in the context of constructing rule lists;
it depends on the misclassification error associated with each
rule in a rule list, thus it provides a lower bound on the number of
observations that each such rule must correctly classify.

\begin{arxiv}
\subsection{Upper Bound on Antecedent Support}
\label{sec:ub-support}

In the previous section~(\S\ref{sec:lb-support}), we proved lower bounds on
antecedent support; in Appendix~\ref{appendix:ub-supp},
we give an upper bound on antecedent support.
Specifically, Theorem~\ref{thm:ub-support} shows that an antecedent's
support in a rule list cannot be too similar to the set of data not
captured by preceding antecedents in the rule list.
In particular, Theorem~\ref{thm:ub-support} implies that we should
only mine rules with normalized support less than or equal to ${1 - \Reg}$;
we need not mine rules with a larger fraction of observations.
Note that we do not otherwise use this bound in our implementation,
because we did not observe a meaningful benefit in preliminary experiments.
\end{arxiv}

\begin{arxiv}
\subsection{Antecedent Rejection and its Propagation}
\label{sec:reject}

In this section, we demonstrate further consequences of
our lower~(\S\ref{sec:lb-support}) and upper
bounds (\S\ref{sec:ub-support}) on antecedent support,
under a unified framework we refer to as antecedent rejection.
Let ${\Prefix = (p_1, \dots, p_K)}$ be a prefix,
and let~$p_k$ be an antecedent in~$\Prefix$.
Define~$p_k$ to have insufficient support in~$\Prefix$
if it does not obey the bound in~\eqref{eq:min-capture}
of Theorem~\ref{thm:min-capture}.
Define~$p_k$ to have insufficient accurate support in~$\Prefix$
if it does not obey the bound in~\eqref{eq:min-capture-correct}
of Theorem~\ref{thm:min-capture-correct}.
Define~$p_k$ to have excessive support in~$\Prefix$
if it does not obey the bound in~\eqref{eq:ub-support}
of Theorem~\ref{thm:ub-support} (Appendix~\ref{appendix:ub-supp}).
If~$p_k$ in the context of~$\Prefix$ has insufficient support,
insufficient accurate support, or excessive support,
%or support too similar to the set of data not captured by
%preceding antecedents (Theorem~\ref{thm:ub-support}),
let us say that prefix~$\Prefix$ rejects antecedent~$p_K$.
Next, in Theorem~\ref{thm:reject}, we describe large classes of
related rule lists whose prefixes all reject the same antecedent.

\begin{theorem}[Antecedent rejection propagates]
\label{thm:reject}
For any prefix ${\Prefix = (p_1, \dots, p_K)}$,
let $\StartContains(\Prefix)$ denote the set of all
prefixes~$\Prefix'$ such that
the set of all antecedents in~$\Prefix$ is a subset of
the set of all antecedents in~$\Prefix'$, \ie
\begin{align}
\StartContains(\Prefix) =
\{\Prefix' = (p'_1, \dots, p'_{K'})
~s.t.~ \{p_k : p_k \in \Prefix \} \subseteq
\{p'_\kappa : p'_\kappa \in \Prefix'\}, K' \ge K \}.
\label{eq:start-contains}
\end{align}
Let ${\RL = (\Prefix, \Labels, \Default, K)}$ be a rule list
with prefix ${\Prefix = (p_1, \dots, p_{K-1}, p_{K})}$,
such that~$\Prefix$ rejects its last antecedent~$p_{K}$,
either because~$p_{K}$ in the context of~$\Prefix$ has
insufficient support, insufficient accurate support,
or excessive support.
Let ${\Prefix^{K-1} = (p_1, \dots, p_{K-1})}$ be the
first~${K - 1}$ antecedents of~$\Prefix$.
Let ${\RLB = (\PrefixB, \LabelsB, \DefaultB, \kappa)}$
be any rule list with prefix
${\PrefixB = (P_1, \dots, P_{K'-1},}$ ${P_{K'}, \dots, P_{\kappa})}$
such that~$\PrefixB$ starts with ${\PrefixB^{K'-1} =}$
${(P_1, \dots, P_{K'-1}) \in}$ ${\StartContains(\Prefix^{K-1})}$
and antecedent ${P_{K'} = p_{K}}$.
It follows that prefix~$\PrefixB$ rejects~$P_{K'}$
for the same reason that~$\Prefix$ rejects~$p_{K}$,
and furthermore, $\RLB$~cannot be optimal, \ie
${\RLB \notin \argmin_{\RL^\dagger} \Obj(\RL^\dagger, \x, \y)}$.
\end{theorem}

\begin{proof}
Combine Proposition~\ref{prop:min-capture}, Proposition~\ref{prop:min-capture-correct},
and Proposition~\ref{prop:ub-support}.
The first two are found below, and the last in Appendix~\ref{appendix:ub-supp}.
\end{proof}

Theorem~\ref{thm:reject} implies potentially significant
computational savings.
We know from Theorems~\ref{thm:min-capture},
\ref{thm:min-capture-correct}, and~\ref{thm:ub-support}
that during branch-and-bound execution, if we ever encounter a
prefix ${\Prefix = (p_1, \dots, p_{K-1}, p_K)}$ that rejects its
last antecedent~$p_K$, then we can prune~$\Prefix$.
By Theorem~\ref{thm:reject}, we can also prune \emph{any} prefix~$\Prefix'$
whose antecedents contains the set of antecedents in~$\Prefix$,
in almost any order, with the constraint that all antecedents
in ${\{p_1, \dots, p_{K-1}\}}$ precede~$p_K$.
These latter antecedents are also rejected
directly by the bounds in Theorems~\ref{thm:min-capture},
\ref{thm:min-capture-correct}, and~\ref{thm:ub-support};
this is how our implementation works in practice.
In a preliminary implementation (not shown), we maintained additional
data structures to support the direct use of Theorem~\ref{thm:reject}.
We leave the design of efficient data structures for this task as future work.

\begin{proposition}[Insufficient antecedent support propagates]
\label{prop:min-capture}
First define~$\StartContains(\Prefix)$ as in~\eqref{eq:start-contains},
and let ${\Prefix = (p_1, \dots, p_{K-1}, p_{K})}$ be a prefix,
such that its last antecedent~$p_{K}$ has insufficient support,
\ie the opposite of the bound in~\eqref{eq:min-capture}:
${\Supp(p_K, \x \given \Prefix) < \Reg}$.
Let ${\Prefix^{K-1} =}$ ${(p_1, \dots, p_{K-1})}$,
and let ${\RLB = (\PrefixB, \LabelsB, \DefaultB, \kappa)}$
be any rule list with prefix
${\PrefixB = (P_1, \dots, P_{K'-1},}$ ${P_{K'}, \dots, P_{\kappa})}$,
such that~$\PrefixB$ starts with ${\PrefixB^{K'-1} =}$
${(P_1, \dots, P_{K'-1}) \in \StartContains(\Prefix^{K-1})}$
and~${P_{K'} = p_{K}}$.
It follows that~$P_{K'}$ has insufficient support in
prefix~$\PrefixB$, and furthermore, $\RLB$~cannot be optimal,
\ie ${\RLB \notin \argmin_{\RL} \Obj(\RL, \x, \y)}$.
\end{proposition}

\begin{proof}
The support of~$p_K$ in~$\Prefix$ depends only on the
set of antecedents in ${\Prefix^{K} = (p_1, \dots, p_{K})}$:
\begin{align}
\Supp(p_K, \x \given \Prefix)
= \frac{1}{N} \sum_{n=1}^N \Cap(x_n, p_K \given \Prefix)
&= \frac{1}{N} \sum_{n=1}^N \left( \neg\, \Cap(x_n, \Prefix^{K-1}) \right)
  \wedge \Cap(x_n, p_K) \nn \\
&= \frac{1}{N} \sum_{n=1}^N \left( \bigwedge_{k=1}^{K-1} \neg\, \Cap(x_n, p_k) \right)
  \wedge \Cap(x_n, p_K)
< \Reg, \nn
\end{align}
and the support of~$P_{K'}$ in~$\PrefixB$ depends only on
the set of antecedents in ${\PrefixB^{K'} =}$ ${(P_1, \dots, P_{K'})}$:
\begin{align}
\Supp(P_{K'}, \x \given \PrefixB)
= \frac{1}{N} \sum_{n=1}^N \Cap(x_n, P_{K'} \given \PrefixB) %\nn \\
%&= \frac{1}{N} \sum_{n=1}^N \left( \neg\, \Cap(x_n, \PrefixB^{K'-1}) \right)
%  \wedge \Cap(x_n, P_{K'}) \nn \\
&= \frac{1}{N} \sum_{n=1}^N \left( \bigwedge_{k=1}^{K'-1} \neg\, \Cap(x_n, P_k) \right)
   \wedge \Cap(x_n, P_{K'}) \nn \\
&\le \frac{1}{N} \sum_{n=1}^N \left( \bigwedge_{k=1}^{K-1} \neg\, \Cap(x_n, p_k) \right)
  \wedge \Cap(x_n, P_{K'}) \nn \\
&= \frac{1}{N} \sum_{n=1}^N \left( \bigwedge_{k=1}^{K-1} \neg\, \Cap(x_n, p_k) \right)
  \wedge \Cap(x_n, p_{K}) \nn \\
&= \Supp(p_K, \x \given \Prefix) < \Reg.
\label{ineq:supp}
\end{align}
The first inequality reflects the condition that
${\PrefixB^{K'-1} \in \StartContains(\Prefix^{K-1})}$,
which implies that the set of antecedents in~$\PrefixB^{K'-1}$
contains the set of antecedents in~$\Prefix^{K-1}$,
and the next equality reflects the fact that~${P_{K'} = p_K}$.
Thus,~$P_K'$ has insufficient support in prefix~$\PrefixB$,
therefore by Theorem~\ref{thm:min-capture}, $\RLB$~cannot be optimal,
\ie ${\RLB \notin \argmin_{\RL} \Obj(\RL, \x, \y)}$.
\end{proof}

\begin{proposition}[Insufficient accurate antecedent support propagates]
\label{prop:min-capture-correct}
~Let~$\StartContains(\Prefix)$ denote the set of all
prefixes~$\Prefix'$ such that
the set of all antecedents in~$\Prefix$ is a subset of
the set of all antecedents in~$\Prefix'$,
as in~\eqref{eq:start-contains}.
Let ${\RL = (\Prefix, \Labels, \Default, K)}$ be a rule list
with prefix ${\Prefix = (p_1, \dots, p_{K})}$
and labels ${\Labels = (q_1, \dots, q_{K})}$, such that
the last antecedent~$p_{K}$ has insufficient accurate support,
\ie the opposite of the bound in~\eqref{eq:min-capture-correct}:
\begin{align}
\frac{1}{N} \sum_{n=1}^N \Cap(x_n, p_K \given \Prefix) \wedge \one [ q_K = y_n ]
< \Reg. \nn
\end{align}
Let ${\Prefix^{K-1} = (p_1, \dots, p_{K-1})}$
and let ${\RLB = (\PrefixB, \LabelsB, \DefaultB, \kappa)}$
be any rule list with prefix ${\PrefixB =}$ ${(P_1, \dots, P_{\kappa})}$
and labels ${\LabelsB = (Q_1, \dots, Q_{\kappa})}$,
such that~$\PrefixB$ starts with ${\PrefixB^{K'-1} =}$
${(P_1, \dots, P_{K'-1})}$ ${\in \StartContains(\Prefix^{K-1})}$
and ${P_{K'} = p_{K}}$.
It follows that~$P_{K'}$ has insufficient accurate support in
prefix~$\PrefixB$, and furthermore,
${\RLB \notin \argmin_{\RL^\dagger} \Obj(\RL^\dagger, \x, \y)}$.
\end{proposition}

\begin{proof}
The accurate support of~$P_{K'}$ in~$\PrefixB$ is insufficient:
\begin{align}
\frac{1}{N} \sum_{n=1}^N \Cap(x_n, P_{K'} &\given \PrefixB) \wedge \one [ Q_{K'} = y_n ] \nn \\
&= \frac{1}{N} \sum_{n=1}^N \left( \bigwedge_{k=1}^{K'-1} \neg\, \Cap(x_n, P_k) \right)
   \wedge \Cap(x_n, P_{K'}) \wedge \one [ Q_{K'} = y_n ] \nn \\
&\le \frac{1}{N} \sum_{n=1}^N \left( \bigwedge_{k=1}^{K-1} \neg\, \Cap(x_n, p_k) \right)
   \wedge \Cap(x_n, P_{K'}) \wedge \one [ Q_{K'} = y_n ] \nn \\
&= \frac{1}{N} \sum_{n=1}^N \left( \bigwedge_{k=1}^{K-1} \neg\, \Cap(x_n, p_k) \right)
   \wedge \Cap(x_n, p_K) \wedge \one [ Q_{K'} = y_n ] \nn \\
&= \frac{1}{N} \sum_{n=1}^N \Cap(x_n, p_K \given \Prefix) \wedge \one [ Q_{K'} = y_n ] \nn \\
&\le \frac{1}{N} \sum_{n=1}^N \Cap(x_n, p_K \given \Prefix) \wedge \one [ q_{K} = y_n ]
< \Reg. \nn
\end{align}
The first inequality reflects the condition that
${\PrefixB^{K'-1} \in \StartContains(\Prefix^{K-1})}$,
the next equality reflects the fact that~${P_{K'} = p_K}$.
For the following equality, notice that~$Q_{K'}$ is the majority
class label of data captured by~$P_{K'}$ in~$\PrefixB$, and~$q_K$
is the majority class label of data captured by~$P_K$ in~$\Prefix$,
and recall from~\eqref{ineq:supp} that
${\Supp(P_{K'}, \x \given \PrefixB) \le \Supp(p_{K}, \x \given \Prefix)}$.
By Theorem~\ref{thm:min-capture-correct},
${\RLB \notin \argmin_{\RL^\dagger} \Obj(\RL^\dagger, \x, \y)}$.
\end{proof}

Propositions~\ref{prop:min-capture} and~\ref{prop:min-capture-correct},
combined with Proposition~\ref{prop:ub-support} (Appendix~\ref{appendix:ub-supp}),
constitute the proof of Theorem~\ref{thm:reject}.
\end{arxiv}

\subsection{Equivalent Support Bound}
\label{sec:equivalent}

If two prefixes capture the same data, and one is more accurate than the other,
then there is no benefit to considering prefixes that start with the less accurate one.
Let~$\Prefix$ be a prefix,
and consider the best possible rule list whose prefix starts with~$\Prefix$.
If we take its antecedents in~$\Prefix$ and replace them with another prefix
with the same support (that could include different antecedents),
then its objective can only become worse or remain the same.

Formally, let~$\PrefixB$ be a prefix, and let~$\xi(\PrefixB)$ be the set
of all prefixes that capture exactly the same data as~$\PrefixB$.
Now, let~$\RL$ be a rule list with prefix~$\Prefix$
in~$\xi(\PrefixB)$, such that~$\RL$ has the minimum objective
over all rule lists with prefixes in~$\xi(\PrefixB)$.
Finally, let~$\RL'$ be a rule list whose prefix~$\Prefix'$
starts with~$\Prefix$, such that~$\RL'$ has the minimum objective
over all rule lists whose prefixes start with~$\Prefix$.
Theorem~\ref{thm:equivalent} below implies that~$\RL'$ also has
the minimum objective over all rule lists whose prefixes start with
\emph{any} prefix in~$\xi(\PrefixB)$.

\begin{theorem}[Equivalent support bound]
\label{thm:equivalent}
Define $\StartsWith(\Prefix)$ to be the set of all rule lists
whose prefixes start with~$\Prefix$, as in~\eqref{eq:starts-with}.
Let ${\RL = (\Prefix, \Labels, \Default, K)}$
be a rule list with prefix ${\Prefix = (p_1, \dots, p_K)}$,
and let ${\RLB = (\PrefixB, \LabelsB, \DefaultB, \kappa)}$
be a rule list with prefix ${\PrefixB = (P_1, \dots, P_{\kappa})}$,
such that~$\Prefix$ and~$\PrefixB$ capture the same data,~\ie
\begin{align}
\{x_n : \Cap(x_n, \Prefix)\} = \{x_n : \Cap(x_n, \PrefixB)\}. \nn
\end{align}
If the objective lower bounds of~$\RL$ and~$\RLB$
obey ${b(\Prefix, \x, \y) \le b(\PrefixB, \x, \y)}$,
then the objective of the optimal rule list in~$\StartsWith(\Prefix)$ gives a
lower bound on the objective of the optimal rule list in~$\StartsWith(\PrefixB)$:
\begin{align}
\min_{\RL' \in \StartsWith(\Prefix)} \Obj(\RL', \x, \y)
\le \min_{\RLB' \in \StartsWith(\PrefixB)} \Obj(\RLB', \x, \y).
\label{eq:equivalent}
\end{align}
\end{theorem}

\begin{arxiv}
\begin{proof}
See Appendix~\ref{appendix:equiv-supp} for the proof of Theorem~\ref{thm:equivalent}.
\end{proof}
\end{arxiv}

Thus, if prefixes~$\Prefix$ and~$\PrefixB$ capture the same data,
and their objective lower bounds obey
${b(\Prefix, \x, \y) \le b(\PrefixB, \x, \y)}$,
Theorem~\ref{thm:equivalent} implies that we can prune~$\PrefixB$.
%
%In our implementation, we call this symmetry-aware garbage collection.
%
Next, in Sections~\ref{sec:permutation} and~\ref{sec:permutation-counting},
we highlight and analyze the special case of prefixes that capture
the same data because they contain the same antecedents.

\subsection{Permutation Bound}% for permutation-aware garbage collection}
\label{sec:permutation}

If two prefixes are composed of the same antecedents,
\ie they contain the same antecedents up to a permutation,
then they capture the same data, and thus Theorem~\ref{thm:equivalent} applies.
Therefore, if one is more accurate than the other, then there is no benefit to
considering prefixes that start with the less accurate one.
Let~$\Prefix$ be a prefix,
and consider the best possible rule list whose prefix starts with~$\Prefix$.
If we permute its antecedents in~$\Prefix$,
then its objective can only become worse or remain the same.

Formally, let~${P = \{p_k\}_{k=1}^K}$ be a set of~$K$ antecedents,
and let~$\Pi$ be the set of all $K$-prefixes corresponding to
permutations of antecedents in~$P$.
Let prefix~$\Prefix$ in~$\Pi$ have the minimum prefix misclassification
error over all prefixes in~$\Pi$.
Also, let~$\RL'$ be a rule list whose prefix~$\Prefix'$
starts with~$\Prefix$, such that~$\RL'$ has the minimum objective
over all rule lists whose prefixes start with~$\Prefix$.
Corollary~\ref{thm:permutation} below,
which can be viewed as special case of Theorem~\ref{thm:equivalent},
implies that~$\RL'$ also has the minimum objective over all
rule lists whose prefixes start with \emph{any} prefix in~$\Pi$.

\begin{corollary}[Permutation bound]
\label{thm:permutation}
Let~$\pi$ be any permutation of ${\{1, \dots, K\}}$,
\begin{arxiv}
and define ${\StartsWith(\Prefix) = }$
${\{(\Prefix', \Labels', \Default', K') : \Prefix' \textnormal{ starts with } \Prefix \}}$
to be the set of all rule lists whose prefixes start with~$\Prefix$.
\end{arxiv}
\begin{kdd}
and define $\StartsWith(\Prefix)$
to be the set of all rule lists whose prefix starts with~$\Prefix$,
as in~\eqref{eq:starts-with}.
\end{kdd}
Let ${\RL = (\Prefix, \Labels, \Default, K)}$
and ${\RLB = (\PrefixB, \LabelsB, \DefaultB, K)}$
denote rule lists with prefixes ${\Prefix = (p_1, \dots, p_K)}$
and ${\PrefixB = (p_{\pi(1)}, \dots, p_{\pi(K)})}$,
respectively, \ie the antecedents in~$\PrefixB$
correspond to a permutation of the antecedents in~$\Prefix$.
If the objective lower bounds of~$\RL$ and~$\RLB$
obey ${b(\Prefix, \x, \y) \le b(\PrefixB, \x, \y)}$,
then the objective of the optimal rule list in~$\StartsWith(\Prefix)$ gives a
lower bound on the objective of the optimal rule list in~$\StartsWith(\PrefixB)$:
\begin{align}
\min_{\RL' \in \StartsWith(\Prefix)} \Obj(\RL', \x, \y)
\le \min_{\RLB' \in \StartsWith(\PrefixB)} \Obj(\RLB', \x, \y). \nn
%\label{eq:permutation}
\end{align}
\end{corollary}

\begin{arxiv}
\begin{proof}
Since prefixes~$\Prefix$ and~$\PrefixB$ contain
the same antecedents, they both capture the same data.
Thus, we can apply Theorem~\ref{thm:equivalent}.
\end{proof}
\end{arxiv}

Thus if prefixes~$\Prefix$ and~$\PrefixB$ have the same antecedents,
up to a permutation, and their objective lower bounds
obey~${b(\Prefix, \x, \y) \le}$ ${b(\PrefixB, \x, \y)}$,
Corollary~\ref{thm:permutation} implies that we can prune~$\PrefixB$.
We call this %permutation-aware garbage collection,
symmetry-aware pruning,
and we illustrate the subsequent
computational savings next in~\S\ref{sec:permutation-counting}.

\begin{arxiv}
\subsection{Upper Bound on Prefix Evaluations with Symmetry-aware Pruning}
\end{arxiv}
\begin{kdd}
\subsubsection{Upper bound on prefix evaluations with symmetry-aware pruning}
\end{kdd}
\label{sec:permutation-counting}

Here, we present an upper bound on the total number of prefix
evaluations that accounts for the effect of symmetry-aware
pruning~(\S\ref{sec:permutation}).
Since every subset of~$K$ antecedents generates an equivalence
class of~$K!$ prefixes equivalent up to permutation, symmetry-aware
pruning dramatically reduces the search space.

\begin{arxiv}
First, notice that
\end{arxiv}
Algorithm~\ref{alg:branch-and-bound} describes a
breadth-first exploration of the state space of rule lists.
Now suppose we integrate symmetry-aware pruning into
our execution of branch-and-bound, so that after evaluating
prefixes of length~$K$, we only keep a single best prefix
from each set of prefixes equivalent up to a permutation.

\begin{theorem}[\fontdimen2\font=0.6ex Upper bound on prefix evaluations with symmetry-aware pruning]
Consider a state space of all rule lists formed from a set~$\RuleSet$
of~$M$ antecedents, and consider the branch-and-bound algorithm with
symmetry-aware pruning.
Define $\TotalRemaining(\RuleSet)$ to be the total number of prefixes evaluated.
For any set~$\RuleSet$ of $M$ rules,
\begin{align}
\TotalRemaining(\RuleSet)
\le  1 + \sum_{k=1}^K \frac{1}{(k - 1)!} \cdot \frac{M!}{(M - k)!}, \nn
\end{align}
where ${K = \min(\lfloor 1 / 2 \Reg \rfloor, M)}$.
\end{theorem}

\begin{proof}
By Corollary~\ref{cor:ub-prefix-length},
${K \equiv \min(\lfloor 1 / 2 \Reg \rfloor, M)}$
gives an upper bound on the length of any optimal rule list.
The algorithm begins by evaluating the empty prefix,
followed by~$M$ prefixes of length~${k=1}$,
then~${P(M, 2)}$ prefixes of length~${k=2}$,
where~${P(M, 2)}$ is the number of size-2 subsets of~$\{1, \dots, M \}$.
Before proceeding to length~${k=3}$, we keep only~${C(M, 2)}$
prefixes of length~${k=2}$, where~${C(M, k)}$ denotes the
number of $k$-combinations of~$M$.
Now, the number of length~${k=3}$ prefixes we evaluate is~${C(M, 2) (M - 2)}$.
Propagating this forward~gives
\begin{arxiv}
\begin{align}
\TotalRemaining(\RuleSet) \le 1 + \sum_{k=1}^K C(M, k-1) (M - k + 1)
%= 1 + \sum_{k=1}^K {M \choose k-1}(M - k + 1)
%= 1 + \sum_{k=1}^K \frac{M! (M - k + 1)}{(k - 1)! (M - k + 1)!}
= 1 + \sum_{k=1}^K \frac{1}{(k - 1)!} \cdot \frac{M!}{(M - k)!}. \nn
\end{align}
\end{arxiv}
\begin{kdd}
${\TotalRemaining(\RuleSet) \le 1 + \sum_{k=1}^K C(M, k-1) (M - k + 1)}$.
\end{kdd}
\end{proof}

Pruning based on permutation symmetries thus yields significant
computational savings.
Let us compare, for example, to the na\"ive number of prefix evaluations
given by the upper bound in Proposition~\ref{thm:ub-total-eval}.
If~${M = 100}$ and~${K = 5}$, then the na\"ive number is about
${9.1 \times 10^9}$, while the reduced number due to symmetry-aware
pruning is about ${3.9 \times 10^8}$,
which is smaller by a factor of about~23.
If~${M=1000}$ and~${K = 10}$, the number of evaluations falls from
about~${9.6 \times 10^{29}}$ to about~${2.7 \times 10^{24}}$,
which is smaller by a factor of about~360,000.
\begin{arxiv}

\end{arxiv}
\begin{kdd}
\end{kdd}
While~$10^{24}$ seems infeasibly enormous,
it does not represent the number of rule lists we evaluate.
\begin{kdd}
As we show in~\S\ref{sec:experiments},
\end{kdd}
\begin{arxiv}
As we show in our experiments~(\S\ref{sec:experiments}),
\end{arxiv}
our permutation bound in Corollary~\ref{thm:permutation}
and our other bounds together conspire to reduce the search space
to a size manageable on a single computer.
The choice of ${M=1000}$ and ${K=10}$ in our example above
corresponds to the state space size our efforts target.
${K=10}$ rules represents a (heuristic) upper limit on
the size of an interpretable rule list,
and ${M=1000}$ represents the approximate number of rules
with sufficiently high support (Theorem~\ref{thm:min-capture})
we expect to obtain via rule mining~(\S\ref{sec:setup}).

\begin{arxiv}
\subsection{Similar Support Bound}
\label{sec:similar}

We now present a relaxation of Theorem~\ref{thm:equivalent},
our equivalent support bound.
Theorem~\ref{thm:similar} implies that if we know that no extensions of
a prefix~$\Prefix$ are better than the current best objective,
then we can prune all prefixes with support similar to~$\Prefix$'s support.
Understanding how to exploit this result in practice
represents an exciting direction for future work;
our implementation~(\S\ref{sec:implementation}) does not
currently leverage the bound in Theorem~\ref{thm:similar}.

\begin{theorem}[Similar support bound]
\label{thm:similar}
Define $\StartsWith(\Prefix)$ to be the set of all rule lists
whose prefixes start with~$\Prefix$, as in~\eqref{eq:starts-with}.
Let ${\Prefix = (p_1, \dots, p_K)}$ and
${\PrefixB = (P_1, \dots, P_{\kappa})}$ be prefixes
that capture nearly the same data.
Specifically, define~$\omega$ to be the normalized support
of data captured by~$\Prefix$ and not captured by~$\PrefixB$, \ie
%and let us require that~${\omega \le \Reg}$, \ie
\begin{align}
\omega \equiv \frac{1}{N} \sum_{n=1}^N
  \neg\, \Cap(x_n, \PrefixB)
  \wedge \Cap(x_n, \Prefix). % \le \Reg,
\label{eq:omega}
\end{align}
%where~$\Reg$ is the regularization parameter.
%
Similarly, define~$\Omega$ to be the normalized support
of data captured by~$\PrefixB$ and not captured by~$\Prefix$, \ie
%and let us require that~${\Omega \le \Reg}$, \ie
\begin{align}
\Omega \equiv \frac{1}{N} \sum_{n=1}^N
  \neg\, \Cap(x_n, \Prefix)
  \wedge \Cap(x_n, \PrefixB). %\le \Reg.
\label{eq:big-omega}
\end{align}
We can bound the difference between the objectives of the
optimal rule lists in~$\StartsWith(\Prefix)$
and $\StartsWith(\PrefixB)$ as follows:
\begin{align}
\min_{\RLB^\dagger \in \StartsWith(\PrefixB)} \Obj(\RLB^\dagger, \x, \y)
- \min_{\RL^\dagger \in \StartsWith(\Prefix)} \Obj(\RL^\dagger, \x, \y)
&\ge b(\PrefixB, \x, \y) - b(\Prefix, \x, \y) - \omega - \Omega,
\label{eq:similar}
\end{align}
where~$b(\Prefix, \x, \y)$ and~$b(\PrefixB, \x, \y)$ are the
objective lower bounds of~$\RL$ and~$\RLB$, respectively.
\end{theorem}

\begin{proof}
See Appendix~\ref{appendix:similar-supp} for the proof of Theorem~\ref{thm:similar}.
\end{proof}

Theorem~\ref{thm:similar} implies that if prefixes~$\Prefix$
and~$\PrefixB$ are similar, and we know the optimal objective
of rule lists starting with~$\Prefix$, then
\begin{align}
\min_{\RLB' \in \StartsWith(\PrefixB)} \Obj(\RLB', \x, \y)
&\ge \min_{\RL' \in \StartsWith(\Prefix)} \Obj(\RL', \x, \y)
+ b(\PrefixB, \x, \y) - b(\Prefix, \x, \y) - \chi \nn \\
&\ge \CurrentObj + b(\PrefixB, \x, \y) - b(\Prefix, \x, \y) - \chi, \nn
\end{align}
where~$\CurrentObj$ is the current best objective,
and~$\chi$ is the normalized support of the set of data captured
either exclusively by~$\Prefix$ or exclusively by~$\PrefixB$.
It follows that
\begin{align}
\min_{\RLB' \in \StartsWith(\PrefixB)} \Obj(\RLB', \x, \y)
\ge \CurrentObj + b(\PrefixB, \x, \y) - b(\Prefix, \x, \y) - \chi \ge \CurrentObj \nn
\end{align}
if ${b(\PrefixB, \x, \y) - b(\Prefix, \x, \y) \ge \chi}$.
To conclude, we summarize this result and combine it with
our notion of lookahead from Lemma~\ref{lemma:lookahead}.
During branch-and-bound execution, if we demonstrate that
${\min_{\RL' \in \StartsWith(\Prefix)} \Obj(\RL', \x, \y) \ge \CurrentObj}$,
then we can prune all prefixes that start with any
prefix~$\PrefixB'$ in the following set:
\begin{align}
\left\{ \PrefixB' : b(\PrefixB', \x, \y) + \Reg - b(\Prefix, \x, \y) \ge
\frac{1}{N} \sum_{n=1}^N \Cap(x_n, \Prefix) \oplus \Cap(x_n, \PrefixB') \right\}, \nn
\end{align}
where the symbol~$\oplus$ denotes the logical operation, exclusive or (XOR).

\end{arxiv}

\subsection{Equivalent Points Bound}
\label{sec:identical}

The bounds in this section quantify the following:
If multiple observations that are not captured by a prefix~$\Prefix$
have identical features and opposite labels, then no rule list that
starts with~$\Prefix$ can correctly classify all these observations.
For each set of such observations, the number of mistakes is at least
the number of observations with the minority label within the set.

Consider a data set~${\{(x_n, y_n)\}_{n=1}^N}$ and also a set of antecedents
${\{s_m\}_{m=1}^M}$.
Define distinct observations to be equivalent if they are captured by
exactly the same antecedents, \ie ${x_i \neq x_j}$ are equivalent~if
\begin{align}
\frac{1}{M} \sum_{m=1}^M \one [ \Cap(x_i, s_m) = \Cap(x_j, s_m) ] = 1. \nn
\end{align}
Notice that we can partition a data set into sets of equivalent points;
let~${\{e_u\}_{u=1}^U}$ enumerate these sets.
Let~$e_u$ be the equivalent points set that contains observation~$x_i$.
Now define~$\theta(e_u)$ to be the normalized support of the minority
class label with respect to set~$e_u$, \eg let
\begin{arxiv}
\begin{align}
{e_u = \{x_n : \forall m \in [M],\, \one [ \Cap(x_n, s_m) = \Cap(x_i, s_m) ] \}}, \nn
\end{align}
\end{arxiv}
\begin{kdd}
${e_u = \{x_n : \forall m \in [M],\, \one [ \Cap(x_n, s_m) = \Cap(x_i, s_m) ] \}}$,
\end{kdd}
and let~$q_u$ be the minority class label among points in~$e_u$, then
\begin{align}
\theta(e_u) = \frac{1}{N} \sum_{n=1}^N \one [ x_n \in e_u ]\, \one [ y_n = q_u ].
\label{eq:theta}
\end{align}

The existence of equivalent points sets with non-singleton support
yields a tighter objective lower bound that we can combine with our other bounds;
as our experiments demonstrate~(\S\ref{sec:experiments}),
the practical consequences can be dramatic.
First, for intuition, we present a general bound in
Proposition~\ref{prop:identical}; next, we explicitly integrate
this bound into our framework in Theorem~\ref{thm:identical}.

\begin{proposition}[General equivalent points bound]
\label{prop:identical}
Let ${\RL = (\Prefix, \Labels, \Default, K)}$ be a rule list, then
\begin{arxiv}
\begin{align}
\Obj(\RL, \x, \y) \ge \sum_{u=1}^U \theta(e_u) + \Reg K. \nn
\end{align}
\end{arxiv}
\begin{kdd}
${\Obj(\RL, \x, \y) \ge \sum_{u=1}^U \theta(e_u) + \Reg K}$.
\end{kdd}
\end{proposition}

\begin{arxiv}
\begin{proof}
Recall that the objective is ${\Obj(\RL, \x, \y) = \Loss(\RL, \x, \y) + \Reg K}$,
where the misclassification error~${\Loss(\RL, \x, \y)}$ is given by
\begin{align}
\Loss(\RL, \x, \y) &= \Loss_0(\Prefix, \Default, \x, \y) + \Loss_p(\Prefix, \Labels, \x, \y) \nn \\
&= \frac{1}{N} \sum_{n=1}^N \left( \neg\, \Cap(x_n, \Prefix) \wedge \one[q_0 \neq y_n]
   + \sum_{k=1}^K \Cap(x_n, p_k \given \Prefix) \wedge \one [q_k \neq y_n] \right). \nn
\end{align}
Any particular rule list uses a specific rule, and therefore a single class label,
to classify all points within a set of equivalent points.
Thus, for a set of equivalent points~$u$, the rule list~$\RL$ correctly classifies either
points that have the majority class label, or points that have the minority class label.
It follows that~$\RL$ misclassifies a number of points in~$u$ at least as great as
the number of points with the minority class label.
To translate this into a lower bound on~$\Loss(\RL, \x, \y)$,
we first sum over all sets of equivalent points, and then for each such set,
count differences between class labels and the minority class label of the set,
instead of counting mistakes:
\begin{align}
&\Loss(\RL, \x, \y) \nn \\
&= \frac{1}{N} \sum_{u=1}^U \sum_{n=1}^N \left( \neg\, \Cap(x_n, \Prefix) \wedge \one[q_0 \neq y_n]
   + \sum_{k=1}^K \Cap(x_n, p_k \given \Prefix) \wedge \one [q_k \neq y_n] \right)
   \one [x_n \in e_u]  \nn \\
&\ge \frac{1}{N} \sum_{u=1}^U \sum_{n=1}^N \left( \neg\, \Cap(x_n, \Prefix) \wedge \one[y_n = q_u]
   + \sum_{k=1}^K \Cap(x_n, p_k \given \Prefix) \wedge \one [y_n = q_u] \right)
   \one [x_n \in e_u].
\label{eq:lb-equiv-pts}
\end{align}
Next, we factor out the indicator for equivalent point set membership,
which yields a term that sums to one, because every datum is either captured or
not captured by prefix~$\Prefix$.
\begin{align}
\Loss(\RL, \x, \y) &= \frac{1}{N} \sum_{u=1}^U \sum_{n=1}^N \left( \neg\, \Cap(x_n, \Prefix)
   + \sum_{k=1}^K \Cap(x_n, p_k \given \Prefix) \right)
   \wedge \one [x_n \in e_u]\, \one[y_n = q_u] \nn \\
&= \frac{1}{N} \sum_{u=1}^U \sum_{n=1}^N \left( \neg\, \Cap(x_n, \Prefix)
   + \Cap(x_n, \Prefix) \right)
   \wedge \one [x_n \in e_u]\, \one[y_n = q_u] \nn \\
&= \frac{1}{N} \sum_{u=1}^U \sum_{n=1}^N \one [ x_n \in e_u ]\, \one [ y_n = q_u ]
= \sum_{u=1}^U \theta(e_u), \nn
\end{align}
where the final equality applies the definition of~$\theta(e_u)$ in~\eqref{eq:theta}.
Therefore, ${\Obj(\RL, \x, \y) =}$ ${\Loss(\RL, \x, \y) + \Reg K}$ ${\ge \sum_{u=1}^U \theta(e_u) + \Reg K}$.
\end{proof}
\end{arxiv}

Now, recall that to obtain our lower bound~${b(\Prefix, \x, \y)}$
in~\eqref{eq:lower-bound}, we simply deleted the
default rule misclassification error~$\Loss_0(\Prefix, \Default, \x, \y)$
from the objective~${\Obj(\RL, \x, \y)}$.
Theorem~\ref{thm:identical} obtains a tighter objective lower bound
via a tighter lower bound on the default rule misclassification error,
${0 \le b_0(\Prefix, \x, \y) \le}$ $\Loss_0(\Prefix, \Default, \x, \y)$.

\begin{theorem}[Equivalent points bound]
\label{thm:identical}
Let~$\RL$ be a rule list with prefix~$\Prefix$
and lower bound ${b(\Prefix, \x, \y)}$,
then for any rule list~${\RL' \in \StartsWith(\RL)}$
whose prefix~$\Prefix'$ starts with~$\Prefix$,
\begin{arxiv}
\begin{align}
\Obj(\RL', \x, \y) \ge b(\Prefix, \x, \y) + b_0(\Prefix, \x, \y),
\label{eq:identical}
\end{align}
where
\end{arxiv}
\begin{kdd}
\begin{align}
\Obj(\RL', \x, \y) \ge b(\Prefix, \x, \y) + b_0(\Prefix, \x, \y), \quad \text{where}
\end{align}
\end{kdd}
\begin{arxiv}
\begin{align}
b_0(\Prefix, \x, \y) = \frac{1}{N} \sum_{u=1}^U \sum_{n=1}^N
    \neg\, \Cap(x_n, \Prefix) \wedge \one [ x_n \in e_u ]\, \one [ y_n = q_u ].
\label{eq:lb-b0}
\end{align}
\end{arxiv}
\begin{kdd}
\begin{align}
b_0(\Prefix, \x, \y) = \frac{1}{N} \sum_{u=1}^U \sum_{n=1}^N
    \neg\, \Cap(x_n, \Prefix) \wedge \one [ x_n \in e_u ]\, \one [ y_n = q_u ]. \nn
\end{align}
\end{kdd}
\end{theorem}

\begin{arxiv}
\begin{proof}
See Appendix~\ref{appendix:equiv-pts} for the proof of Theorem~\ref{thm:identical}.
\end{proof}
\end{arxiv}

\begin{arxiv}
\section{Incremental Computation}
\label{sec:incremental}
\end{arxiv}

For every prefix~$\Prefix$ evaluated during
Algorithm~\ref{alg:branch-and-bound}'s execution, we compute
the objective lower bound~${b(\Prefix, \x, \y)}$ and sometimes
the objective~${\Obj(\RL, \x, \y)}$ of the corresponding rule list~$\RL$.
These calculations are the dominant
\begin{arxiv}
computations with respect to execution time.
This motivates
\end{arxiv}
\begin{kdd}
computations and motivate
\end{kdd}
our use of a highly optimized library,
designed by~\citet{YangRuSe16}, for representing rule lists and
performing operations encountered in evaluating functions of rule lists.
Furthermore, we exploit the hierarchical nature of the objective
function and its lower bound to compute these quantities
incrementally throughout branch-and-bound execution.
\begin{arxiv}
In this section, we provide explicit expressions for
the incremental computations that are central to our approach.
Later, in~\S\ref{sec:implementation}, we describe a cache data structure
for supporting our incremental framework in practice.

For completeness, before presenting our incremental expressions,
let us begin by writing down the objective lower bound and objective
of the empty rule list, ${\RL = ((), (), \Default, 0)}$,
the first rule list evaluated in Algorithm~\ref{alg:branch-and-bound}.
Since its prefix contains zero rules, it has zero prefix
misclassification error and also has length zero.
Thus, the empty rule list's objective lower bound is zero, \ie ${b((), \x, \y) = 0}$.
Since none of the data are captured by the empty prefix, the default rule
corresponds to the majority class, and the objective corresponds to the
default rule misclassification error, \ie ${\Obj(\RL, \x, \y) = \Loss_0((), \Default, \x, \y)}$.

Now, we derive our incremental expressions for the objective function and its lower bound.
Let ${\RL = (\Prefix, \Labels, \Default, K)}$ and
${\RL' = (\Prefix', \Labels', \Default', K + 1)}$
be rule lists such that prefix ${\Prefix =}$ ${(p_1, \dots, p_K)}$
is the parent of ${\Prefix' = (p_1, \dots, p_K, p_{K+1})}$.
Let ${\Labels = (q_1, \dots, q_K)}$ and
${\Labels' = (q_1, \dots,}$ ${q_K, q_{K+1})}$ be the corresponding labels.
The hierarchical structure of Algorithm~\ref{alg:branch-and-bound}
enforces that if we ever evaluate~$\RL'$, then we will have already
evaluated both the objective and objective lower bound of its parent,~$\RL$.
We would like to reuse as much of these computations as possible
in our evaluation of~$\RL'$.
We can write the objective lower bound of~$\RL'$ incrementally,
with respect to the objective lower bound of~$\RL$:
\begin{align}
b(\Prefix', \x, \y)
  &= \Loss_p(\Prefix', \Labels', \x, \y) + \Reg (K + 1) \nn \\
&= \frac{1}{N} \sum_{n=1}^N \sum_{k=1}^{K+1} \Cap(x_n, p_k \given \Prefix')
  \wedge \one [ q_k \neq y_n ] + \Reg (K+1) \label{eq:non-inc-lb} \\
&= \Loss_p(\Prefix, \Labels, \x, \y) + \Reg K + \Reg
  + \frac{1}{N} \sum_{n=1}^N \Cap(x_n, p_{K+1} \given \Prefix') \wedge \one [q_{K+1} \neq y_n ] \nn \\
&= b(\Prefix, \x, \y) + \Reg
  + \frac{1}{N} \sum_{n=1}^N \Cap(x_n, p_{K+1} \given \Prefix') \wedge \one [q_{K+1} \neq y_n ] \nn \\
&= b(\Prefix, \x, \y) + \Reg  + \frac{1}{N} \sum_{n=1}^N \neg\, \Cap(x_n, \Prefix) \wedge
  \Cap(x_n, p_{K+1}) \wedge \one [q_{K+1} \neq y_n].
\label{eq:inc-lb}
\end{align}
Thus, if we store $b(\Prefix, \x, \y)$, % and ${\neg\, \Cap(\x, \Prefix)}$,
then we can reuse this quantity when computing $b(\Prefix', \x, \y)$.
Transforming~\eqref{eq:non-inc-lb} into~\eqref{eq:inc-lb} yields a
significantly simpler expression that is a function of the stored
quantity~$b(\Prefix, \x, \y)$.
For the objective of~$\RL'$, first let us write a na\"ive expression:
\begin{align}
&\Obj(\RL', \x, \y) = \Loss(\RL', \x, \y) + \Reg (K + 1)
= \Loss_p(\Prefix', \Labels', \x, \y) + \Loss_0(\Prefix', \Default', \x, \y) + \Reg(K + 1) \nn \\
&= \frac{1}{N} \sum_{n=1}^N \sum_{k=1}^{K+1} \Cap(x_n, p_k \given \Prefix')
  \wedge \one [ q_k \neq y_n ] + \frac{1}{N}\sum_{n=1}^N \neg\, \Cap(x_n, \Prefix') \wedge
  \one [q'_0 \neq y_n] + \Reg (K+1). \label{eq:non-inc-obj}
\end{align}
Instead, we can compute the objective of~$\RL'$ incrementally
with respect to its objective lower bound:
\begin{align}
\Obj(\RL', \x, \y) &=  \Loss_p(\Prefix', \Labels', \x, \y) +
  \Loss_0(\Prefix', \Default', \x, \y) + \Reg (K + 1) \nn \\
&= b(\Prefix', \x, \y) + \Loss_0(\Prefix', \Default', \x, \y) \nn \\
&= b(\Prefix', \x, \y) + \frac{1}{N}\sum_{n=1}^N \neg\, \Cap(x_n, \Prefix') \wedge
  \one [q'_0 \neq y_n] \nn \\
&= b(\Prefix', \x, \y) + \frac{1}{N}\sum_{n=1}^N \neg\, \Cap(x_n, \Prefix) \wedge
  (\neg\, \Cap(x_n, p_{K+1})) \wedge \one [q'_0 \neq y_n].
\label{eq:inc-obj}
\end{align}
The expression in~\eqref{eq:inc-obj} is simpler to compute than that
in~\eqref{eq:non-inc-obj}, because the former reuses $b(\Prefix', \x, \y)$,
which we already computed in~\eqref{eq:inc-lb}.
Note that instead of computing~$\Obj(\RL', \x, \y)$ incrementally from $b(\Prefix', \x, \y)$
as in~\eqref{eq:inc-obj}, we could have computed it incrementally from $\Obj(\RL, \x, \y)$.
However, doing so would in practice require that we store~$\Obj(\RL, \x, \y)$
in addition to~$b(\Prefix, \x, \y)$, which we already must store to support~\eqref{eq:inc-lb}.
We prefer the incremental approach suggested by~\eqref{eq:inc-obj}
since it avoids this additional storage overhead.

\begin{algorithm}[t!]
  \caption{Incremental branch-and-bound for learning rule lists, for simplicity, from a cold start.
  We explicitly show the incremental objective lower bound and objective functions in Algorithms~\ref{alg:incremental-lb} and~\ref{alg:incremental-obj}, respectively.}
\label{alg:incremental}
\begin{algorithmic}
\normalsize
\State \textbf{Input:} Objective function~$\Obj(\RL, \x, \y)$,
objective lower bound~${b(\Prefix, \x, \y)}$,
set of antecedents ${\RuleSet = \{s_m\}_{m=1}^M}$,
training data~$(\x, \y) = {\{(x_n, y_n)\}_{n=1}^N}$,
regularization parameter~$\Reg$
\State \textbf{Output:} Provably optimal rule list~$\OptimalRL$ with minimum objective~$\OptimalObj$ \\

\State $\CurrentRL \gets ((), (), \Default, 0)$ \Comment{Initialize current best rule list with empty rule list}
\State $\CurrentObj \gets \Obj(\CurrentRL, \x, \y)$ \Comment{Initialize current best objective}
\State $Q \gets $ queue$(\,[\,(\,)\,]\,)$ \Comment{Initialize queue with empty prefix}
\State $C \gets $ cache$(\,[\,(\,(\,)\,, 0\,)\,]\,)$ \Comment{Initialize cache with empty prefix and its objective lower bound}
\While {$Q$ not empty} \Comment{Optimization complete when the queue is empty}
	\State $\Prefix \gets Q$.pop(\,) \Comment{Remove a length-$K$ prefix~$\Prefix$ from the queue}
        \State $b(\Prefix, \x, \y) \gets C$.find$(\Prefix)$ \Comment{Look up $\Prefix$'s lower bound in the cache}
        \State $\mathbf{u} \gets \neg\,\Cap(\x, \Prefix)$ \Comment{Bit vector indicating data not captured by $\Prefix$}
        \For {$s$ in $\RuleSet$} \Comment{Evaluate all of $\Prefix$'s children}
            \If {$s$ not in $\Prefix$}
                \State $\PrefixB \gets (\Prefix, s)$ \Comment{\textbf{Branch}: Generate child $\PrefixB$}
                \State $\mathbf{v} \gets \mathbf{u} \wedge \Cap(\x, s)$ \Comment{Bit vector indicating data captured by $s$ in $\PrefixB$}
                \State $b(\PrefixB, \x, \y) \gets b(\Prefix, \x, \y) + \Reg~ + $ \Call{IncrementalLowerBound}{$\mathbf{v}, \y, N$} %\Comment{Eq.~\eqref{eq:inc-lb}}
                \If {$b(\PrefixB, \x, \y) < \CurrentObj$} \Comment{\textbf{Bound}: Apply bound from Theorem~\ref{thm:bound}}
                    \State $\Obj(\RLB, \x, \y) \gets b(\PrefixB, \x, \y)~ + $ \Call{IncrementalObjective}{$\mathbf{u}, \mathbf{v}, \y, N$} %\Comment{Eq.~\eqref{eq:inc-obj}}
                    \State $\RLB \gets (\PrefixB, \LabelsB, \DefaultB, K+1)$ \Comment{$\LabelsB, \DefaultB$ are set in the incremental functions}
                    \If {$\Obj(\RLB, \x, \y) < \CurrentObj$}
                        \State $(\CurrentRL, \CurrentObj) \gets (\RLB, \Obj(\RLB, \x, \y))$ \Comment{Update current best rule list and objective}
                    \EndIf
                    \State $Q$.push$(\PrefixB)$ \Comment{Add $\PrefixB$ to the queue}
                    \State $C$.insert$(\PrefixB, b(\PrefixB, \x, \y))$ \Comment{Add $\PrefixB$ and its lower bound to the cache}
                \EndIf
            \EndIf
        \EndFor
\EndWhile
\State $(\OptimalRL, \OptimalObj) \gets (\CurrentRL, \CurrentObj)$ \Comment{Identify provably optimal rule list and objective}
\end{algorithmic}
\end{algorithm}

\begin{algorithm}[t!]
  \caption{Incremental objective lower bound~\eqref{eq:inc-lb} used in Algorithm~\ref{alg:incremental}.}
\label{alg:incremental-lb}
\begin{algorithmic}
\normalsize
\State \textbf{Input:}
Bit vector~${\mathbf{v} \in \{0, 1\}^N}$ indicating data captured by $s$, the last antecedent in~$\PrefixB$,
bit vector of class labels~${\y \in \{0, 1\}^N}$,
number of observations~$N$
\State \textbf{Output:} Component of~$\RLB$'s misclassification error due to data captured by~$s$ \\

\Function{IncrementalLowerBound}{$\mathbf{v}, \y, N$}
    \State $n_v = \Count(\mathbf{v})$ \Comment{Number of data captured by $s$, the last antecedent in $\PrefixB$}
    \State $\mathbf{w} \gets \mathbf{v} \wedge \y$ \Comment{Bit vector indicating data captured by $s$ with label $1$}
    \State $n_w = \Count(\mathbf{w})$ \Comment{Number of data captured by $s$ with label $1$}
    \If {$n_w / n_v > 0.5$}
        \State \Return $(n_v - n_w) / N$ \Comment{Misclassification error of the rule $s \rightarrow 1$}
    \Else
        \State \Return $n_w / N$ \Comment{Misclassification error of the rule $s \rightarrow 0$}
    \EndIf
    \EndFunction
\end{algorithmic}
\end{algorithm}

\begin{algorithm}[t!]
  \caption{Incremental objective function~\eqref{eq:inc-obj} used in Algorithm~\ref{alg:incremental}.}
\label{alg:incremental-obj}
\begin{algorithmic}
\normalsize
\State \textbf{Input:}
Bit vector~${\mathbf{u} \in \{0, 1\}^N}$ indicating data not captured by~$\PrefixB$'s parent prefix,
bit vector~${\mathbf{v} \in \{0, 1\}^N}$ indicating data not captured by $s$, the last antecedent in~$\PrefixB$,
bit vector of class labels~${\y \in \{0, 1\}^N}$,
number of observations~$N$
\State \textbf{Output:} Component of~$\RLB$'s misclassification error due to its default rule \\

 \Function{IncrementalObjective}{$\mathbf{u}, \mathbf{v}, \y, N$}
    \State $\mathbf{f} \gets \mathbf{u} \wedge \neg\,\mathbf{v} $ \Comment{Bit vector indicating data not captured by $\PrefixB$}
    \State $n_f = \Count(\mathbf{f})$ \Comment{Number of data not captured by $\PrefixB$}
    \State $\mathbf{g} \gets \mathbf{f} \wedge \y$ \Comment{Bit vector indicating data not captured by $\PrefixB$ with label $1$}
    \State $n_g = \Count(\mathbf{g})$ \Comment{Number of data not captued by $\PrefixB$ with label $1$}
    \If {$n_g / n_f > 0.5$}
        \State \Return $(n_f - n_g) / N$ \Comment{Misclassification error of the default label prediction $1$}
    \Else
        \State \Return $n_g / N$ \Comment{Misclassification error of the default label prediction $0$}
    \EndIf
\EndFunction
\end{algorithmic}
\end{algorithm}

We present an incremental branch-and-bound procedure in
Algorithm~\ref{alg:incremental}, and show the incremental computations
of the objective lower bound~\eqref{eq:inc-lb} and objective~\eqref{eq:inc-obj}
as two separate functions in Algorithms~\ref{alg:incremental-lb}
and~\ref{alg:incremental-obj}, respectively.
In Algorithm~\ref{alg:incremental}, we use a cache to store
prefixes and their objective lower bounds.
Algorithm~\ref{alg:incremental} additionally reorganizes the structure
of Algorithm~\ref{alg:branch-and-bound} to group together the computations
associated with all children of a particular prefix.
This has two advantages.
The first is to consolidate cache queries: all children of the same
parent prefix compute their objective lower bounds with respect to
the parent's stored value, and we only require one cache `find' operation
for the entire group of children, instead of a separate query for each child.
The second is to shrink the queue's size:
instead of adding all of a prefix's children as separate queue elements,
we represent the entire group of children in the queue by a single element.
Since the number of children associated with each prefix
is close to the total number of possible antecedents,
both of these effects can yield significant savings.
For example, if we are trying to optimize over rule lists formed
from a set of 1000 antecedents, then the maximum queue size in
Algorithm~\ref{alg:incremental} will be smaller than that in
Algorithm~\ref{alg:branch-and-bound} by a factor of nearly 1000.

\end{arxiv}

\clearpage
\section{Implementation}
\label{sec:implementation}

We implement our algorithm using a collection of optimized data structures
that we describe in this section.
First, we explain how we use a prefix tree~(\S\ref{sec:trie})
to support the incremental computations that we motivated in~\S\ref{sec:incremental}.
Second, we describe several queue designs
that implement different search policies~(\S\ref{sec:queue}).
Third, we introduce a symmetry-aware map~(\S\ref{sec:pmap}) to support
symmetry-aware pruning~(Corollary~\ref{thm:permutation},~\S\ref{sec:permutation}).
Next, we summarize how these data structures interact throughout
our model of incremental execution~(\S\ref{sec:execution}).
In particular, Algorithms~\ref{alg:bounds} and~\ref{alg:pmap} illustrate many
of the computational details from CORELS' inner loop, highlighting each of
the bounds from~\S\ref{sec:framework} that we use to prune the search space.
We additionally describe how we garbage collect our data structures~(\S\ref{sec:gc}).
Finally, we explore how our queue can be used to support
custom scheduling policies designed to improve performance~(\S\ref{sec:scheduling}).

\subsection{Prefix Tree}
\label{sec:trie}

Our incremental computations (\S\ref{sec:incremental}) require a
cache to keep track of prefixes that we have already evaluated
and that are also still under consideration by the algorithm.
We implement this cache as a prefix tree, a data structure also known as a trie,
which allows us to efficiently represent structure shared between related prefixes.
Each node in the prefix tree encodes an individual rule ${r_k = p_k \rightarrow q_k}$.
Each path starting from the root represents a prefix, such that the final node
in the path also contains metadata associated with that prefix.
For a %rule list ${\RL = (\Prefix, \Labels, \Default, K)}$, with
prefix ${\Prefix = (p_1, \dots, p_K)}$,
let~$\varphi(\Prefix)$ denote the corresponding node in the trie.
The metadata at node~$\varphi(\Prefix)$ supports the incremental computation
and includes:
\begin{itemize}
\item An index encoding~$p_K$, the last antecedent.
\item The objective lower bound $b(\Prefix, \x, \y)$, defined in~\eqref{eq:lower-bound},
  the central bound in our framework (Theorem~\ref{thm:bound}).
\item The lower bound on the default rule misclassification error
  $b_0(\Prefix, \x, \y)$, defined in~\eqref{eq:lb-b0},
  to support our equivalent points bound (Theorem~\ref{thm:identical}).
\item An indicator denoting whether this node should be deleted (see~\S\ref{sec:gc}).
\item A representation of viable extensions of~$\Prefix$,
  \ie length ${K+1}$ prefix that start with~$\Prefix$ and have not been pruned.
\end{itemize}
For evaluation purposes and convenience, we store additional information in
the prefix tree; for a prefix~$\Prefix$ with corresponding rule list
${\RL = (\Prefix, \Labels, \Default, K)}$, the node~$\varphi(\Prefix)$ also stores:
\begin{itemize}
\item The length~$K$; equivalently, node~$\varphi(\Prefix)$'s depth in the trie. 
\item The label prediction~$q_K$ corresponding to antecedent~$p_K$.
\item The default rule label prediction~$\Default$.
\item $\NCap$, the number of samples captured by prefix $\Prefix$, as in~\eqref{eq:num-cap}.
\item The objective value $\Obj(\RL, \x, \y)$, defined in~\eqref{eq:objective}.
\end{itemize}
Finally, we note that we implement the prefix tree as a custom C++ class.

\subsection{Queue}
\label{sec:queue}

The queue is a worklist that orders exploration over the search space of possible
rule lists; every queue element corresponds to a leaf in the prefix tree, and vice versa.
In our implementation, each queue element points to a leaf;
when we pop an element off the queue, we use the leaf's metadata to
incrementally evaluate the corresponding prefix's children.

We order entries in the queue to implement several different search policies.
For example, a first-in-first-out~(FIFO) queue implements breadth-first search~(BFS),
and a priority queue implements best-first search.
In our experiments~(\S\ref{sec:experiments}), we use the C++ Standard Template Library~(STL)
queue and priority queue to implement BFS and best-first search, respectively.
For CORELS, priority queue policies of interest include ordering by the lower bound,
the objective, or more generally, any function that maps prefixes to real values;
stably ordering by prefix length and inverse prefix length implement
BFS and depth-first search (DFS), respectively.
In our released code, we present a unified implementation,
where we use the STL priority queue to support BFS, DFS,
and several best-first search policies.
As we demonstrate in our experiments~(\S\ref{sec:ablation}),
we find that using a custom search strategy,
such as ordering by the lower bound, usually leads to a faster runtime than BFS.

We motivate the design of additional custom search strategies in~\S\ref{sec:scheduling}.
In preliminary work (not shown), we also experimented with
stochastic exploration processes that bypass the need for a queue
by instead following random paths from the root to leaves;
developing such methods could be an interesting direction for future work.
We note that these search policies are referred to as node selection strategies
in the MIP literature.
Strategies such as best-first (best-bound) search and DFS are known as static methods, and the framework we present in~\S\ref{sec:scheduling}
has the spirit of estimate-based methods~\citep{Linderoth1999}.

\subsection{Symmetry-aware Map}
\label{sec:pmap}

The symmetry-aware map supports the symmetry-aware pruning justified in~\S\ref{sec:equivalent}.
In our implementation, we specifically leverage our permutation bound
(Corollary~\ref{thm:permutation}), though it is also possible to directly
exploit the more general equivalent support bound (Theorem~\ref{thm:equivalent}).
We use the C++ STL unordered map to keep track of the best known ordering
of each evaluated set of antecedents.
The keys of our symmetry-aware map encode antecedents in canonical order,
\ie antecedent indices in numerically sorted order,
and we associate all permutations of a set of antecedents with a single key.
Each key maps to a value that encodes the best known prefix in the permutation
group of the key's antecedents, as well as the objective lower bound of that prefix.

Before we consider adding a prefix~$\Prefix$ to the trie and queue, we check
whether the map already contains a permutation~$\pi(\Prefix)$ of that prefix.
If no such permutation exists, then we insert~$\Prefix$ into the map, trie, and queue.
Otherwise, if a permutation~$\pi(\Prefix)$ exists and the lower bound of~$\Prefix$ is better
than that of~$\pi(\Prefix)$, \ie ${b(\Prefix, \x, \y) <}$ ${b(\pi(\Prefix), \x, \y)}$,
then we update the map and remove~$\pi(\Prefix)$ and its entire subtree from the trie;
we also insert~$\Prefix$ into the trie and queue.
Otherwise, if there exists a permutation~$\pi(\Prefix)$ such that
${b(\pi(\Prefix), \x, \y) \le}$ ${b(\Prefix, \x, \y)}$,
then we do nothing, \ie we do not insert~$\Prefix$ into any data structures.

\subsection{Incremental Execution}
\label{sec:execution}

Mapping our algorithm to our data structures produces the following execution strategy,
which we also illustrate in Algorithms~\ref{alg:bounds} and~\ref{alg:pmap}.
We initialize the current best objective~$\CurrentObj$ and rule list~$\CurrentRL$.
While the trie contains unexplored leaves, a scheduling policy selects the next prefix~$\Prefix$
to extend; in our implementation, we pop elements from a (priority) queue, until the queue is empty.
Then, for every antecedent~$s$ that is not in~$\Prefix$,
we construct a new prefix~$\PrefixB$ by appending~$s$ to~$\Prefix$;
we incrementally calculate the lower bound~$b(\PrefixB, \x, \y)$,
the objective~$\Obj(\RLB, \x, \y)$, of the associated rule list~$\RLB$,
and other quantities used by our algorithm, summarized by the metadata fields of
the (potential) prefix tree node~$\varphi(\PrefixB)$.

If the objective~$\Obj(\RLB, \x, \y)$ is less than the current best objective~$\CurrentObj$,
then we update~$\CurrentObj$ and~$\CurrentRL$.
If the lower bound of the new prefix~$\PrefixB$ is less than the current best objective,
%\ie ${b(\Prefix', \x, \y) < \CurrentObj}$,
then as described in~\S\ref{sec:pmap}, we query the symmetry-aware map for~$\PrefixB$;
if we insert~$\Prefix'$ into the symmetry-aware map, then we also insert it into the trie and queue.
Otherwise, %\ie ${b(\Prefix', \x, \y) \ge \CurrentObj}$
then by our hierarchical lower bound (Theorem~\ref{thm:bound}),
no extension of~$\PrefixB$ could possibly lead to a rule list with objective
better than~$\CurrentObj$, thus we do not insert~$\PrefixB$ into the tree or queue.
We also leverage our other bounds from~\S\ref{sec:framework}
to aggressively prune the search space; we highlight each of these bounds
in Algorithms~\ref{alg:bounds} and~\ref{alg:pmap},
which summarize the computations and data structure operations performed in CORELS' inner loop.
When there are no more leaves to explore, \ie the queue is empty, we output the optimal rule list.
We can optionally terminate early according to some alternate condition,
\eg when the size of the prefix tree exceeds some threshold.

\begin{algorithm}[t!]
  \caption{The inner loop of CORELS, which evaluates all children of a prefix~$\Prefix$.}
%  For details about symmetry-aware map queries, see Algorithm~\ref{alg:pmap}.}
\label{alg:bounds}
\begin{algorithmic}
\small
\State Define $\mathbf{z} \in \{0, 1\}^N$, s.t. ${z_n = \sum_{u=1}^U \one [x_n \in e_u] [y_n = q_u]}$ \\
\Comment{$e_u$
is the equivalent points set containing~$x_n$ and $q_u$ is the minority class label of~$e_u$ (\S\ref{sec:identical})}
\State Define $b(\Prefix, \x, \y)$ and $\mathbf{u} = \neg\,\Cap(\x, \Prefix)$ \Comment{$\mathbf{u}$ is a bit vector indicating data not captured by $\Prefix$}
\vspace{1.5mm}
\For {$s$ in $\RuleSet$ \textbf{if} $s$ not in $\Prefix$ \textbf{then}} \Comment{Evaluate all of $\Prefix$'s children}
%    \If {$s$ not in $\Prefix$}
        \State $\PrefixB \gets (\Prefix, s)$ \Comment{\textbf{Branch}: Generate child $\PrefixB$}
        \State $\mathbf{v} \gets \mathbf{u} \wedge \Cap(\x, s)$ \Comment{Bit vector indicating data captured by $s$ in $\PrefixB$}
        \State $n_v = \Count(\mathbf{v})$ \Comment{Number of data captured by $s$, the last antecedent in $\PrefixB$}
        \If {$n_v / N < \Reg$}
            \State \Continue \Comment{\textbf{Lower bound on antecedent support (Theorem\ref{thm:min-capture})}}
        \EndIf
        \State $\mathbf{w} \gets \mathbf{v} \wedge \y$ \Comment{Bit vector indicating data captured by $s$ with label $1$}
        \State $n_w = \Count(\mathbf{w})$ \Comment{Number of data captured by $s$ with label $1$}
        \If {$n_w / n_v \ge 0.5$}
            \State $n_c \gets n_w$ \Comment{Number of correct predictions by the new rule $s \rightarrow 1$}
        \Else
            \State $n_c \gets n_v - n_w$ \Comment{Number of correct predictions by the new rule $s \rightarrow 0$}
        \EndIf
        \If {$n_c / N < \Reg$}
            \State \Continue \Comment{\textbf{Lower bound on accurate antecedent support (Theorem~\ref{thm:min-capture-correct})}}
        \EndIf
        \State $\delta_b \gets (n_v - n_c) / N$ \Comment{Misclassification error of the new rule}
       \State $b(\PrefixB, \x, \y) \gets b(\Prefix, \x, \y) + \Reg + \delta_b$ \Comment{Incremental lower bound~\eqref{eq:inc-lb}}
       \If {$b(\PrefixB, \x, \y) \ge \CurrentObj$} \Comment{\textbf{Hierarchical objective lower bound (Theorem~\ref{thm:bound})}}
           \State \Continue
       \EndIf
       \State $\mathbf{f} \gets \mathbf{u} \wedge \neg\,\mathbf{v} $ \Comment{Bit vector indicating data not captured by $\PrefixB$}
       \State $n_f = \Count(\mathbf{f})$ \Comment{Number of data not captured by $\PrefixB$}
       \State $\mathbf{g} \gets \mathbf{f} \wedge \y$ \Comment{Bit vector indicating data not captured by $\PrefixB$ with label $1$}
       \State $n_g = \Count(\mathbf{g})$ \Comment{Number of data not captued by $\PrefixB$ with label $1$}
       \If {$n_g / n_f \ge 0.5$}
           \State $\delta_\Obj \gets (n_f - n_g) / N$ \Comment{Misclassification error of the default label prediction $1$}
       \Else
           \State $\delta_\Obj \gets n_g / N$ \Comment{Misclassification error of the default label prediction $0$}
       \EndIf
       \State $\Obj(\RLB, \x, \y) \gets b(\PrefixB, \x, \y) + \delta_\Obj$ \Comment{Incremental objective~\eqref{eq:inc-obj}}
       %\State $\Obj(\RLB, \x, \y) \gets b(\PrefixB, \x, \y)~ + $ \Call{IncrementalObjective}{$\mathbf{u}, \mathbf{v}, \y, N$}  \Comment{Inc. objective~\eqref{eq:inc-lb}}
       \State $\RLB \gets (\PrefixB, \LabelsB, \DefaultB, K+1)$ \Comment{$\LabelsB, \DefaultB$ are set in the incremental functions}
       \If {$\Obj(\RLB, \x, \y) < \CurrentObj$}
            \State $(\CurrentRL, \CurrentObj) \gets (\RLB, \Obj(\RLB, \x, \y))$ \Comment{Update current best rule list and objective}
            \State \Call{GarbageCollectPrefixTree}{$\CurrentObj$} \Comment{Delete nodes with lower bound $\ge \CurrentObj - \Reg$ (\S\ref{sec:gc}),}
        \EndIf \hfill {using the \textbf{Lookahead bound (Lemma~\ref{lemma:lookahead})}}
        \State $b_0(\PrefixB, \x, \y) \gets \Count(\mathbf{f} \wedge \mathbf{z}) / N$ \Comment{Lower bound on the default rule misclassification}
        \State $b \gets b(\PrefixB, \x, \y) + b_0(\PrefixB, \x, \y)$ \hfill error defined in~\eqref{eq:lb-b0}
        \If {$b + \Reg \ge \CurrentObj$} \Comment{\textbf{Equivalent points bound (Theorem~\ref{thm:identical})}}
            \State \Continue \hfill {combined with the \textbf{Lookahead bound (Lemma~\ref{lemma:lookahead})}}
        \EndIf
        \State \Call{CheckMapAndInsert}{$\PrefixB, b$} \Comment{Check the \textbf{Permutation bound (Corollary~\ref{thm:permutation})} and}
\EndFor \hfill {possibly insert $\PrefixB$ into data structures (Algorithm~\ref{alg:pmap})}
\end{algorithmic}
\end{algorithm}

\begin{algorithm}[t!]
  \caption{Possibly insert a prefix into CORELS' data structures, after first
  checking the symmetry-aware map, which supports search space pruning
  triggered by the permutation bound (Corollary~\ref{thm:permutation}).
  For further context, see Algorithm~\ref{alg:bounds}.}
\label{alg:pmap}
\begin{algorithmic}
\State $T$ is the prefix tree (\S\ref{sec:trie})
\State $Q$ is the queue, for concreteness, a priority queue ordered by the lower bound (\S\ref{sec:queue})
\State $\PMap$ is the symmetry-aware map (\S\ref{sec:pmap}) \\

\Function{CheckMapAndInsert}{$\PrefixB, b$}
    \State $\pi_0 \gets$ sort($\PrefixB$) \Comment{$\PrefixB$'s antecedents in canonical order}
    \State $(D_\pi, b_\pi) \gets \PMap$.find($\pi_0$) \Comment{Look for a permutation of $\PrefixB$}
    \If {$D_\pi$ exists}
        \If {$b < b_\pi$} \Comment{$\PrefixB$ is better than $D_\pi$}
            \State $\PMap$.update($\pi_0, (\PrefixB, b)$) \Comment{Replace $D_\pi$ with $\PrefixB$ in the map}
            \State $T$.delete\_subtree($D_\pi$) \Comment{Delete $D_\pi$ and its subtree from the prefix tree}
            \State $T$.insert$(\varphi(\PrefixB))$ \Comment{Add node for $\PrefixB$ to the prefix tree}
            \State $Q$.push$(\PrefixB, b)$ \Comment{Add $\PrefixB$ to the queue}
        \Else
            \State \textbf{pass} \Comment{$\PrefixB$ is inferior to $D_\pi$, thus do not insert it into any data structures}
        \EndIf
    \Else
        \State $\PMap$.insert($\pi_0, (\PrefixB, b)$) \Comment{Add $\PrefixB$ to the map}
        \State $T$.insert$(\varphi(\PrefixB))$ \Comment{Add node for $\PrefixB$ to the prefix tree}
        \State $Q$.push$(\PrefixB, b)$ \Comment{Add $\PrefixB$ to the queue}
    \EndIf
\EndFunction
\end{algorithmic}
\end{algorithm}

\subsection{Garbage Collection}
\label{sec:gc}

During execution, we garbage collect the trie.
Each time we update the minimum objective,
we traverse the trie in a depth-first manner, deleting all subtrees
of any node with lower bound larger than the current minimum objective.
At other times, when we encounter a node with no children, we prune upwards,
deleting that node and recursively traversing the tree towards the root,
deleting any childless nodes.
This garbage collection allows us to constrain the trie's memory consumption, though in our
experiments we observe the minimum objective to decrease only a small number of times.

In our implementation, we cannot immediately delete prefix tree leaves
because each corresponds to a queue element that points to it.
The C++ STL priority queue is a wrapper container that prevents access to the
underlying data structure, and thus we cannot access elements in the middle of the queue,
even if we know the relevant identifying information.
We therefore have no way to update the queue without iterating over every element.
We address this by marking prefix tree leaves that we wish to delete (see~\S\ref{sec:trie}),
deleting the physical nodes lazily, after they are popped from the queue.
Later, in our section on experiments~(\S\ref{sec:experiments}),
we refer to two different queues that we define here: the physical queue
corresponds to the C++ queue, and thus all prefix tree leaves, and the logical queue
corresponds only to those prefix tree leaves that have not been marked for deletion.

\subsection{Custom Scheduling Policies}
\label{sec:scheduling}

In our setting, an ideal scheduling policy would immediately identify an optimal
rule list, and then certify its optimality by systematically eliminating the
remaining search space.
This motivates trying to design scheduling policies that tend to quickly find optimal rule lists.
When we use a priority queue to order the set of prefixes to evaluate next,
we are free to implement different scheduling policies via the ordering of
elements in the queue.
This motivates designing functions that assign higher priorities to `better'
prefixes that we believe are more likely to lead to optimal rule lists.
We follow the convention that priority queue elements are ordered
by keys, such that keys with smaller values have higher priorities.

We introduce a custom class of functions that we call \emph{curiosity} functions.
Broadly, we think of the curiosity of a rule list~$\RL$
as the expected objective value of another rule list~$\RL'$ that is related to~$\RL$;
different models of the relationship between~$\RL$ and~$\RL'$ lead to different
curiosity functions.
In general, the curiosity of~$\RL$ is, by definition, equal to the sum of the expected
misclassification error and the expected regularization penalty of~$\RL'$:
\begin{align}
\Curiosity(\Prefix, \x, \y) \equiv \E[ \Obj(\RL', \x, \y) ]
&= \E[\Loss(\Prefix', \Labels', \x, \y)] + \Reg \E[ K' ].
\label{eq:curiosity}
\end{align}

Next, we describe a simple curiosity function for a rule list~$\RL$ with prefix~$\Prefix$.
First, let~$\NCap$ denote the number of observations captured by~$\Prefix$, \ie
\begin{align}
\NCap \equiv \sum_{n=1}^N \Cap(x_n, \Prefix).
\label{eq:num-cap}
\end{align}
We now describe a model that generates another
rule list~${\RL' = (\Prefix', \Labels', \Default', K')}$ from~$\Prefix$.
Assume that prefix~$\Prefix'$ starts with~$\Prefix$ and captures all the data,
such that each additional antecedent in~$\Prefix'$
captures as many `new' observations as each antecedent in~$\Prefix$, on average;
then, the expected length of~$\Prefix'$ is
\begin{align}
\E[ K' ] = \frac{N}{\NCap / K}.
\label{eq:curiosity-length}
\end{align}
Furthermore, assume that each additional antecedent in~$\Prefix'$
makes as many mistakes as each antecedent in~$\Prefix$, on average,
thus the expected misclassification error of~$\Prefix'$ is
\begin{align}
\E[\Loss(\Prefix', \Labels', \x, \y)]
&= \E[\Loss_p(\Prefix', \Labels', \x, \y)] + \E[\Loss_0(\Prefix', \Default', \x, \y)] \nn \\
&= \E[\Loss_p(\Prefix', \Labels', \x, \y)]
=  \E[ K' ] \left(\frac{\Loss_p(\Prefix, \Labels, \x, \y)}{K}\right).
\label{eq:curiosity-error}
\end{align}
Note that the default rule misclassification error~$\Loss_0(\Prefix', \Default', \x, \y)$
is zero because we assume that~$\Prefix'$ captures all the data.
Inserting~\eqref{eq:curiosity-length} and~\eqref{eq:curiosity-error}
into~\eqref{eq:curiosity} thus gives curiosity for this model:
\begin{align*}
\Curiosity(\Prefix, \x, \y)
%= \E[\Loss_p(\Prefix', \Labels', \x, \y)] + \Reg \E[ K' ]
%&= \E[ K' ] \left(\frac{\Loss_p(\Prefix, \Labels, \x, \y)}{K}\right) + \Reg \E[ K' ] \\
%&=  \left(\frac{N}{\NCap / K}\right)
%  \left(\frac{\Loss_p(\Prefix, \Labels, \x, \y)}{K}\right)
%  + \Reg \left(\frac{N}{\NCap / K}\right) \nn \\
&= \left( \frac{N}{\NCap} \right) \biggl(\Loss_p(\Prefix, \Labels, \x, \y) + \Reg K \biggr) \\
&= \left( \frac{1}{N} \sum_{n=1}^N \Cap(x_n, \Prefix) \right)^{-1} b(\Prefix, \x, \y)
= \frac{b(\Prefix, \x, \y)}{\Supp(\Prefix, \x)},
\end{align*}
where for the second equality, we used the definitions in~\eqref{eq:num-cap} of~$\NCap$
and in~\eqref{eq:lower-bound} of~$\Prefix$'s lower bound, and for the last equality,
we used the definition in~\eqref{eq:support} of~$\Prefix$'s normalized support.

The curiosity for a prefix~$\Prefix$ is thus also equal to its objective lower bound,
scaled by the inverse of its normalized support.
For two prefixes with the same lower bound, curiosity gives higher priority to
the one that captures more data.
This is a well-motivated scheduling strategy if we model prefixes that extend
the prefix with smaller support as having more `potential' to make mistakes.
We note that using curiosity in practice does not introduce new bit vector
or other expensive computations; during execution, we can calculate curiosity
as a simple function of already derived quantities.

In preliminary experiments, we observe that using a priority queue ordered by
curiosity sometimes yields a dramatic reduction in execution time,
compared to using a priority queue ordered by the objective lower bound.
Thus far, we have observed significant benefits on specific small problems,
where the structure of the solutions happen to render curiosity particularly
effective (not shown).
Designing and studying other `curious' functions, that are effective in more
general settings, is an exciting direction for future work.

\section{Experiments}
\label{sec:experiments}

Our experimental analysis addresses five questions:
How does CORELS' predictive performance compare to that of COMPAS scores
and other algorithms? (\S\ref{sec:compas}, \S\ref{sec:frisk}, and~\S\ref{sec:sparsity})
How does CORELS' model size compare to that of other algorithms? (\S\ref{sec:sparsity})
How rapidly do the objective value and its lower bound converge,
for different values of the regularization parameter~$\Reg$? (\S\ref{sec:reg-param})
How much does each of the implementation optimizations contribute to CORELS' performance?~(\S\ref{sec:ablation})
How rapidly does CORELS prune the search space? (\S\ref{sec:reg-param} and \S\ref{sec:ablation})
Before proceeding, we first describe our computational environment (\S\ref{sec:environment}),
as well as the data sets and prediction problems we use (\S\ref{sec:datasets}),
and then in Section~\ref{sec:examples} show example optimal rule lists found by CORELS.

\subsection{Computational Environment}
\label{sec:environment}
All timed results ran on a server with an Intel Xeon E5-2699~v4 (55~MB cache, 2.20~GHz) processor and 264~GB RAM,
and we ran each timing measurement separately, on a single hardware thread, with nothing else running on the server.
Except where we mention a memory constraint, all experiments
can run comfortably on smaller machines, \eg a laptop with 16~GB~RAM.

\subsection{Data Sets and Prediction Problems}
\label{sec:datasets}
Our evaluation focuses on two socially-important prediction problems associated
with recent, publicly-available data sets.
Table~1 summarizes the data sets and prediction problems,
and Table~2 summarizes feature sets extracted from each data set,
as well as antecedent sets we mine from these feature sets.
We provide some details next.
For further details about data sets, preprocessing steps, and antecedent mining,
see Appendix~\ref{appendix:data}.

\begin{table}[t!]
\centering
\begin{tabular}{l|c|r|c|c|r|r}
Data set & Prediction problem & N~~~~ & Positive & Resample & Training & Test~~ \\
& & & fraction & training set & set size~ & set size \\
\hline
ProPublica & Two-year recidivism & 6,907 & 0.46 & No & 6,217 & 692 \\
NYPD & Weapon possession & 325,800 & 0.03 & Yes & 566,839 & 32,580 \\
NYCLU & Weapon possession & 29,595 & 0.05 & Yes & 50,743 & 2,959 \\
\end{tabular}
\caption{Summary of data sets and prediction problems.
The last five columns report the total number of observations,
the fraction of observations with the positive class label,
whether we resampled the training set due to class imbalance,
and the sizes of each training and test set in our 10-fold cross-validation studies.
}
\label{tab:datasets}
\end{table}

\begin{table}[t!]
\centering
\begin{tabular}{l|c|c|c|c|c|c}
Data set & Feature & Categorical & Binary & Mined & Max number & Negations \\
& set & attributes & features & antecedents & of clauses & \\
\hline
ProPublica & A & 6 & 13 & 122 & 2 & No \\
ProPublica & B & 7 & 17 & 189 & 2 & No \\
NYPD & C & 5 & 28 & ~28 & 1 & No \\
NYPD & D & 3 & 20 & ~20 & 1 & No \\
NYCLU & E & 5 & 28 & ~46 & 1 & Yes
\end{tabular}
\caption{Summary of feature sets and mined antecedents.
The last five columns report the number of categorical attributes,
the number of binary features, the average number of mined antecedents,
the maximum number of clauses in each antecedent,
and whether antecedents include negated clauses.
}
\label{tab:features}
\end{table}

\subsubsection{Recidivism Prediction}
For our first problem, we predict which individuals in the ProPublica COMPAS
data set~\citep{LarsonMaKiAn16} recidivate within two years.
This data set contains records for all offenders in Broward County, Florida
in 2013 and 2014 who were given a COMPAS score pre-trial.
Recidivism is defined as being charged with a new crime within two years
after receiving a COMPAS assessment; the article by \citet{LarsonMaKiAn16},
and their code,\footnote{Data and code used in the analysis by \citet{LarsonMaKiAn16} can be found at \url{https://github.com/propublica/compas-analysis}.}
provide more details about this definition.
From the original data set of records for 7,214 individuals,
we identify a subset of 6,907 records without missing data.
For the majority of our analysis, we extract a set of~13 binary features (Feature Set~A),
which our antecedent mining framework combines into ${M=122}$ antecedents,
on average (folds ranged from containing 121 to 123 antecedents).
We also consider a second, similar antecedent set in~\S\ref{sec:examples},
derived from a superset of Feature Set~A that includes~4 additional binary features (Feature Set~B).

\subsubsection{Weapon Prediction}
For our second problem, we use New York City
stop-and-frisk data to predict whether a weapon will be found on a stopped
individual who is frisked or searched.
For experiments in Sections~\ref{sec:examples} and~\ref{sec:frisk} and Appendix~\ref{appendix:cpw},
we compile data from a database maintained by the New York Police Department (NYPD)~\citep{nypd},
from years 2008-2012, following~\citet{Goel16}.
Starting from 2,941,390 records, each describing an incident involving
a stopped person, we first extract 376,488 records where the suspected
crime was criminal possession of a weapon (CPW).\footnote{We filter for records that
explicitly match the string `CPW'; we note that additional records, after converting to
lowercase, contain strings such as `cpw' or `c.p.w.'}
From these, we next identify a subset of 325,800 records for which the
individual was frisked and/or searched; of these, criminal possession of a weapon
was identified in only 10,885 instances (about 3.3\%).
Resampling due to class imbalance, for 10-fold cross-validation, yields training sets
that each contain 566,839 datapoints. (We form corresponding test sets without resampling.)
From a set of 5 categorical features, we form a set of~28 single-clause antecedents
corresponding to~28 binary features (Feature Set~C).
We also consider another, similar antecedent set, derived from a subset of Feature Set~C
that excludes~8 location-specific binary features (Feature Set~D).

In Sections~\ref{sec:examples}, \ref{sec:sparsity}, \ref{sec:reg-param}, and~\ref{sec:ablation},
we also use a smaller stop-and-frisk data set,
derived by the NYCLU from the NYPD's 2014 data~\citep{nyclu:2014}.
From the original data set of 45,787 records, each describing an incident involving
a stopped person, we identify a subset of 29,595 records for which the individual
was frisked and/or searched.
Of these, criminal possession of a weapon was identified in about 5\% of instances.
As with the larger NYPD data set, we resample the data to form training sets
(but not to form test sets).
From the same set of 5 categorical features as in Feature Set~C, we form a set of ${M=46}$
single-clause antecedents, including negations (Feature Set~E).

\subsection{Example Optimal Rule Lists}
\label{sec:examples}

To motivate Feature Set~A, described in Appendix~\ref{appendix:data},
which we used in most of our analysis of the ProPublica data set,
we first consider Feature Set~B, a larger superset of features.

\begin{figure}[t!]
\begin{algorithmic}
\State \bif $(age = 21-22) \band (priors = 2-3)$ \bthen $yes$
\State \belif $(age = 18-20)\band (sex = male)$ \bthen $yes$
\State \belif $(priors > 3)$ \bthen $yes$
\State \belse $no$
\end{algorithmic}
\vspace{1mm}
\begin{algorithmic}
\State \bif $(age = 23-25) \band (priors = 2-3)$ \bthen $yes$
\State \belif $(age = 18-20)\band (sex = male)$ \bthen $yes$
\State \belif $(age = 21-22) \band (priors = 2-3)$ \bthen $yes$
\State \belif $(priors > 3)$ \bthen $yes$
\State \belse $no$
\end{algorithmic}
\caption{Example optimal rule lists that predict two-year recidivism for the
ProPublica data set (Feature Set~B, ${M=189}$), found by CORELS (${\Reg = 0.005}$), across 10 cross-validation folds.
While some input antecedents contain features for race, no optimal rule list includes such an antecedent.
Every optimal rule list is the same or similar to one of these examples,
with prefixes containing the same rules, up to a permutation, and same default rule.
}
\label{fig:propublica}
\end{figure}

Figure~\ref{fig:propublica} shows optimal rule lists learned by CORELS,
using Feature Set~B, which additionally includes race categories from the ProPublica data set (African American, Caucasian, Hispanic,
Other\footnote{We grouped the original Native American ($<$0.003), Asian ($<$0.005), and Other ($<$0.06) categories.}).
For Feature Set~B, our antecedent mining procedure generated an average
of~189 antecedents, across folds.
None of the optimal rule lists contain antecedents that directly depend on race;
this motivated our choice to exclude race, by using Feature Set~A, in our subsequent analysis.
For both feature sets, we replaced the original ProPublica age categories ($<$25, 25-45, $>$45)
with a set that is more fine-grained for younger individuals (18-20, 21-22, 23-25, 26-45, $>$45).
Figure~\ref{fig:recidivism-all-folds} shows example optimal rule lists that CORELS learns
for the ProPublica data set (Feature Set~A, ${\Reg = 0.005}$), using 10-fold cross validation.

\begin{figure}[t!]
\vspace{2mm}
\begin{algorithmic}
\State \bif $(age = 18-20) \band (sex = male)$ \bthen $yes$
\State \belif $(age = 21-22) \band (priors = 2-3)$ \bthen $yes$
\State \belif $(priors > 3)$ \bthen $yes$
\State \belse $no$
\end{algorithmic}
\vspace{1mm}
\begin{algorithmic}
\State \bif $(age = 18-20) \band (sex = male)$ \bthen $yes$
\State \belif $(age = 21-22) \band (priors = 2-3)$ \bthen $yes$
\State \belif $(age = 23-25) \band (priors = 2-3)$ \bthen $yes$
\State \belif $(priors > 3)$ \bthen $yes$
\State \belse $no$
\end{algorithmic}
\caption{Example optimal rule lists that predict two-year recidivism for the
ProPublica data set (Feature Set~A, ${M=122}$), found by CORELS (${\Reg = 0.005}$), across 10 cross-validation folds.
Feature Set~A is a subset of Feature Set~B (Figure~\ref{fig:propublica}) that excludes race features.
Optimal rule lists found using the two feature sets are very similar.
The upper and lower rule lists are representative of~7 and~3 folds, respectively.
Each of the remaining~8 solutions is the same or similar to one of these,
with prefixes containing the same rules, up to a permutation, and the same default rule.
See Figure~\ref{fig:recidivism-rule-list-005} in Appendix~\ref{appendix:examples} for a complete listing.
}
\label{fig:recidivism-all-folds}
\end{figure}
\begin{figure}[t!]
\vspace{2mm}
\begin{algorithmic}
\State \bif $(location = transit~authority)$ \bthen $yes$
\State \belif $(stop~reason = suspicious~bulge)$ \bthen $yes$
\State \belif $(stop~reason = suspicious~object)$ \bthen $yes$
\State \belse $no$
\end{algorithmic}
\caption{An example rule list that predicts whether a weapon will be found on a
stopped individual who is frisked or searched, for the NYCLU stop-and-frisk data set.
Across 10 cross-validation folds, the other optimal rule lists found by CORELS~(${\Reg = 0.01}$)
contain the same or equivalent rules, up to a permutation.
See also Figure~\ref{fig:weapon-rule-list-04-01} in Appendix~\ref{appendix:examples}.
}
\label{fig:weapon-rule-list}
\end{figure}

\begin{figure}[t!]
%$ head *cpw_*0.01*opt* K = 2
%{cs_objcs:stop-reason=suspicious-object}~1;{location:transit-authority}~1;default~0 x 8
%{location:transit-authority}~1;{cs_objcs:stop-reason=suspicious-object}~1;default~0 x 2
\textbf{Weapon prediction $(\Reg = 0.01, \text{Feature Set~C})$}
\vspace{1mm}
\begin{algorithmic}
\State \bif $(stop~reason = suspicious~object)$ \bthen $yes$ %\Comment{Found by 8 folds}
\State \belif $(location = transit~authority)$ \bthen $yes$
\State \belse $no$
\end{algorithmic}
\vspace{1mm}
\textbf{Weapon prediction $(\Reg = 0.01, \text{Feature Set~D})$}
\vspace{1mm}
\begin{algorithmic}
\State \bif $(stop~reason = suspicious~object)$ \bthen $yes$ %\Comment{Found by 7 folds}
\State \belif $(inside~or~outside = outside)$ \bthen $no$
\State \belse $yes$
\end{algorithmic}
\vspace{1mm}
\textbf{Weapon prediction $(\Reg = 0.005, \text{Feature Set~C})$}
\begin{algorithmic}
\State \bif $(stop~reason = suspicious~object)$ \bthen $yes$ %\Comment{Found by 7 folds}
\State \belif $(location = transit~authority)$ \bthen $yes$
\State \belif $(location = housing~authority)$ \bthen $no$
\State \belif $(city = Manhattan)$ \bthen $yes$
\State \belse $no$
\end{algorithmic}
\vspace{1mm}
\textbf{Weapon prediction $(\Reg = 0.005, \text{Feature Set~D})$}
\vspace{1mm}
\begin{algorithmic}
\State \bif $(stop~reason = suspicious~object)$ \bthen $yes$ %\Comment{Found by 2 folds}
\State \belif $(stop~reason = acting~as~lookout)$ \bthen $no$
\State \belif $(stop~reason = fits~description)$ \bthen $no$
\State \belif $(stop~reason = furtive~movements)$ \bthen $no$
\State \belse $yes$
\end{algorithmic}
\caption{Example optimal rule lists for the NYPD stop-and-frisk data set, found by CORELS.
Feature Set~C contains attributes for `location' and `city', while Feature Set~D does not.
For each choice of regularization parameter and feature set, the rule lists learned by CORELS,
across all 10~cross-validation folds, contain the same or equivalent rules, up to a permutation,
with the exception of a single fold (Feature Set~C, ${\Reg = 0.005}$).
For a complete listing, see Figures~\ref{fig:cpw-rule-list} and~\ref{fig:cpw-noloc-rule-list}
in Appendix~\ref{appendix:examples}.
}
\label{fig:nypd}
\end{figure}

Figures~\ref{fig:weapon-rule-list} and~\ref{fig:nypd} show example optimal rule lists that
CORELS learns for the NYCLU (${\Reg = 0.01}$) and NYPD data sets.
Figure~\ref{fig:nypd} shows optimal rule lists that CORELS learns for the larger NYPD data set.

While our goal is to provide illustrative examples, and not to provide a
detailed analysis nor to advocate for the use of these specific models,
we note that these rule lists are short and easy to understand.
For the examples and regularization parameter choices in this section,
the optimal rule lists are relatively robust across cross-validation folds:
the rules are nearly the same, up to permutations of the prefix rules.
For smaller values of the regularization parameter, we observe less robustness,
as rule lists are allowed to grow in length.
For the sets of optimal rule lists represented in Figures~\ref{fig:propublica},
\ref{fig:recidivism-all-folds}, and~\ref{fig:weapon-rule-list},
each set could be equivalently expressed as a DNF rule;
\eg this is easy to see when the prefix rules all predict the positive class label
and the default rule predicts the negative class label.
Our objective is not designed to enforce any of these properties,
though some may be seen as desirable.

As we demonstrate in~\S\ref{sec:sparsity},
optimal rule lists learned by CORELS achieve accuracies that are competitive
with a suite of other models, including black box COMPAS scores.
See Appendix~\ref{appendix:examples} for additional listings of optimal rule lists found
by CORELS, for each of our prediction problems, across cross-validation folds,
for different regularization parameters~$\Reg$.

\subsection{Comparison of CORELS to the Black Box COMPAS Algorithm}
\label{sec:compas}

The accuracies of rule lists learned by CORELS are competitive with
scores generated by the black box COMPAS algorithm
at predicting two-year recidivism for the ProPublica data set (Figure~\ref{fig:compas-comparison}).
Across 10 cross-validation folds, optimal rule lists learned by CORELS
(Figure~\ref{fig:recidivism-all-folds}, ${\Reg = 0.005}$)
have a mean test accuracy of~0.665, with standard deviation~0.018.
The COMPAS algorithm outputs scores between~1 and~10,
representing low~(1-4), medium~(5-7), and high~(8-10) risk for recidivism.
As in the analysis by~\citet{LarsonMaKiAn16}, we interpret a medium or high score
as a positive prediction for two-year recidivism, and a low score as a negative prediction.
Across the 10 test sets, the COMPAS algorithm scores obtain
a mean accuracy of~0.660, with standard deviation~0.019.

\begin{figure}[t!]
% left lower right upper
\begin{center}
\includegraphics[trim={35mm, 0mm, 40mm, 0mm},
width=0.9\textwidth]{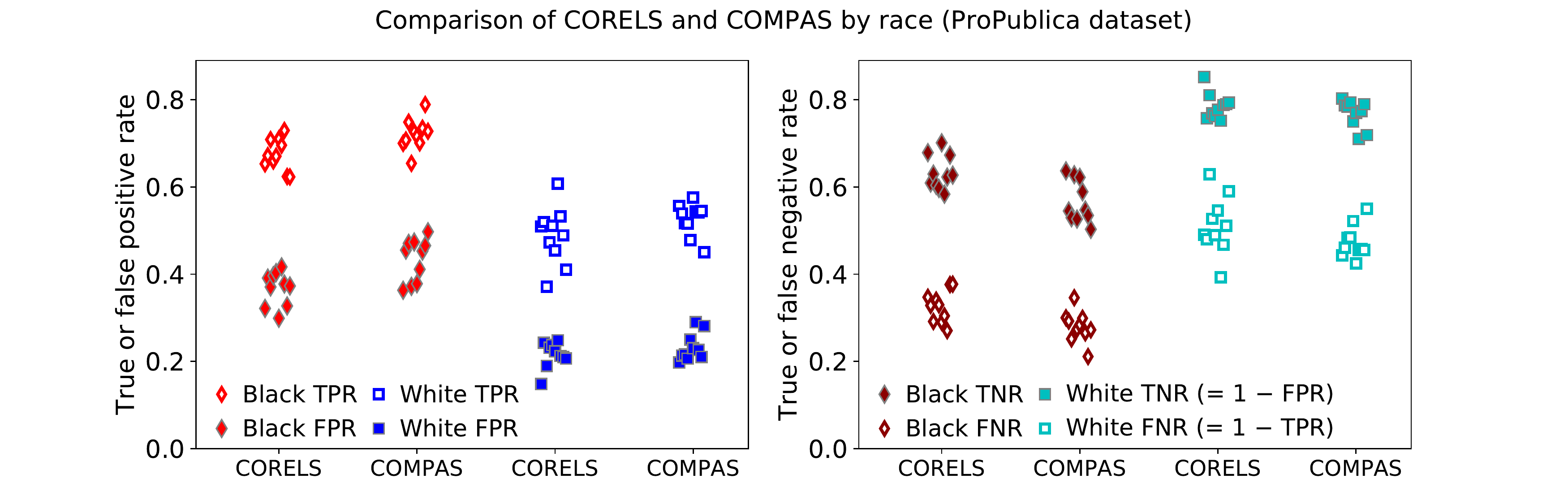}
\end{center}
\caption{Comparison of TPR and FPR (left), as well as TNR and FNR (right),
for different races in the ProPublica data set, for CORELS and COMPAS,
across 10 cross-validation folds.
%
%CORELS has a lower FPR for black individuals in the data set than COMPAS, as well as a lower TPR.
%
%Both CORELS and COMPAS have a significantly higher FPR rate for black indviduals than white individuals---despite CORELS producing optimal rule lists without explicit race features.
%
}
\label{fig:tpr-fpr}
\end{figure}

Figure~\ref{fig:tpr-fpr} shows that CORELS and COMPAS perform similarly across both black and white individuals.
Both algorithms have much higher true positive rates (TPR's) and false positive rates (FPR's) for blacks than whites (left), and higher true negative rates (TNR's) and false negative rates (FNR's) for whites than blacks (right).
The fact that COMPAS has higher FPR's for blacks and higher FNR's for whites was a central observation motivating ProPublica's claim that COMPAS is racially biased~\citep{LarsonMaKiAn16}.
The fact that CORELS' models are so simple, with almost the same results as COMPAS, and contain only counts of past crimes, age, and gender, indicates possible explanations for the uneven predictions of both COMPAS and CORELS among blacks and whites.
In particular, blacks evaluated within Broward County tend to be younger and have longer criminal histories within the data set,
(on average, 4.4 crimes for blacks versus~2.6 crimes for whites)
leading to higher FPR's for blacks and higher FNR's for whites.
This aspect of the data could help to explain why ProPublica concluded that COMPAS was racially biased. 

Similar observations have been reported for other datasets, namely that complex machine learning models do not have an advantage over simpler transparent models \citep{tollenaar2013method,bushway2013there,ZengUsRu2017}.
There are many definitions of fairness, and it is not clear whether CORELS' models are fair either, but it is much easier to debate about the fairness of a model when it is transparent.
Additional fairness constraints or transparency constraints can be placed on CORELS' models if desired, though one would need to edit our bounds~(\S\ref{sec:framework}) and implementation~(\S\ref{sec:implementation}) to impose more constraints.

Regardless of whether COMPAS is racially biased (which our analysis does not indicate is necessarily true as long as criminal history and age are allowed to be considered as features),
COMPAS may have many other fairness defects that might be considered serious.
Many of COMPAS's survey questions are direct inquiries about socioeconomic status.
For instance, a sample COMPAS survey\footnote{A sample COMPAS survey contributed
by Julia Angwin, ProPublica, can be found at
\url{https://www.documentcloud.org/documents/2702103-Sample-Risk-Assessment-COMPAS-CORE.html}.} asks:
``Is it easy to get drugs in your neighborhood?,''
``How often do you have barely enough money to get by?,''
``Do you frequently get jobs that don't pay more than minimum wage?,''
``How often have you moved in the last~12 months?''
COMPAS's survey questions also ask about events that were not caused by the person who is being evaluated, such as:
``If you lived with both parents and they later separated, how old were you at the time?,''
``Was one of your parents ever sent to jail or prison?,''
``Was your mother ever arrested, that you know of?"

The fact that COMPAS requires over~130 questions to be answered, many of whose answers may not be verifiable, means that the computation of the COMPAS score is prone to errors.
Even the Arnold Foundation's ``public-safety assessment'' (PSA) score---which is completely transparent, and has only~9 factors---has been miscalculated in serious criminal trials, leading to a recent lawsuit~\citep{npr-bail:2017}.
It is substantially more difficult to obtain the information required to calculate COMPAS scores than PSA scores (with over~14 times the number of survey questions).
This significant discrepancy suggests that COMPAS scores are more fallible than PSA scores, as well as even simpler models, like those produced by CORELS.
Some of these problems could be alleviated by using only data within electronic records that can be automatically calculated, instead of using information entered by hand and/or collected via subjective surveys.

The United States government pays Northpointe (now called Equivant) to use COMPAS.
In light of our observations that CORELS is as accurate as COMPAS on a real-world data set where COMPAS is used in practice, CORELS predicts similarly to COMPAS for both blacks and whites, and CORELS' models are completely transparent, it is not clear what value COMPAS scores possess.
Our experiments also indicate that the proprietary survey data required to compute COMPAS scores has not boosted its prediction accuracy above that of transparent models in practice.

Risk predictions are important for the integrity of the judicial system; judges cannot be expected to keep entire databases in their heads to calculate risks, whereas models (when used correctly) can help to ensure equity.
Risk prediction models also have the potential to heavily impact how efficient the judicial system is, in terms of bail and parole decisions; efficiency in this case means that dangerous individuals are not released, whereas non-dangerous individuals are granted bail or parole.
High stakes decisions, such as these, are ideal applications for machine learning algorithms that produce transparent models from high dimensional data.

Currently, justice system data does not support highly accurate risk predictions, but current risk models are useful in practice, and these risk predictions will become more accurate as more and higher quality data are made available.

\subsection{Comparison of CORELS to a Heuristic Model for Weapon Prediction}
\label{sec:frisk}

CORELS generates simple, accurate models for the task of weapon prediction,
using the NYPD stop-and-frisk data set.
Our approach offers a principled alternative to heuristic models proposed by~\citet{Goel16},
who develop a series of regression models to analyze racial disparities
in New York City's stop-and-frisk policy for a related, larger data set.
In particular, the authors arrive at a heuristic that they suggest
could potentially help police officers more effectively decide when to
frisk and/or search stopped individuals, \ie when such
interventions are likely to discover criminal possession of a weapon (CPW).
Starting from a full regression model with~7,705 variables, the authors reduce this to a
smaller model with 98 variables; from this, they keep three variables with the largest coefficients.
This gives a heuristic model of the form ${ax + by + cz \ge T}$,
where
\begin{align}
x &= \one[stop~reason = suspicious~object] \nn \\
y &= \one[stop~reason = suspicious~bulge] \nn \\
z &= \one[additional~circumstances = sights~and~sounds~of~criminal~activity], \nn
\end{align}
and~$T$ is a threshold, such that the model predicts CPW when the threshold is met or exceeded.
We focus on their approach that uses a single threshold, rather than precinct-specific thresholds.
To increase ease-of-use, the authors round the coefficients
to the nearest integers, which gives ${(a, b, c) = (3, 1, 1)}$;
this constrains the threshold to take one of six values, ${T \in \{0, 1, 2, 3, 4, 5\}}$.
To employ this heuristic model in the field,
``\dots officers simply need to add at most three small, positive integers \dots
and check whether the sum exceeds a fixed threshold\dots'' \citep{Goel16}.

\begin{figure}[t!]
\begin{center}
% left lower right upper
\includegraphics[trim={12mm, 0mm, 24mm, 5mm},
width=0.71\textwidth]{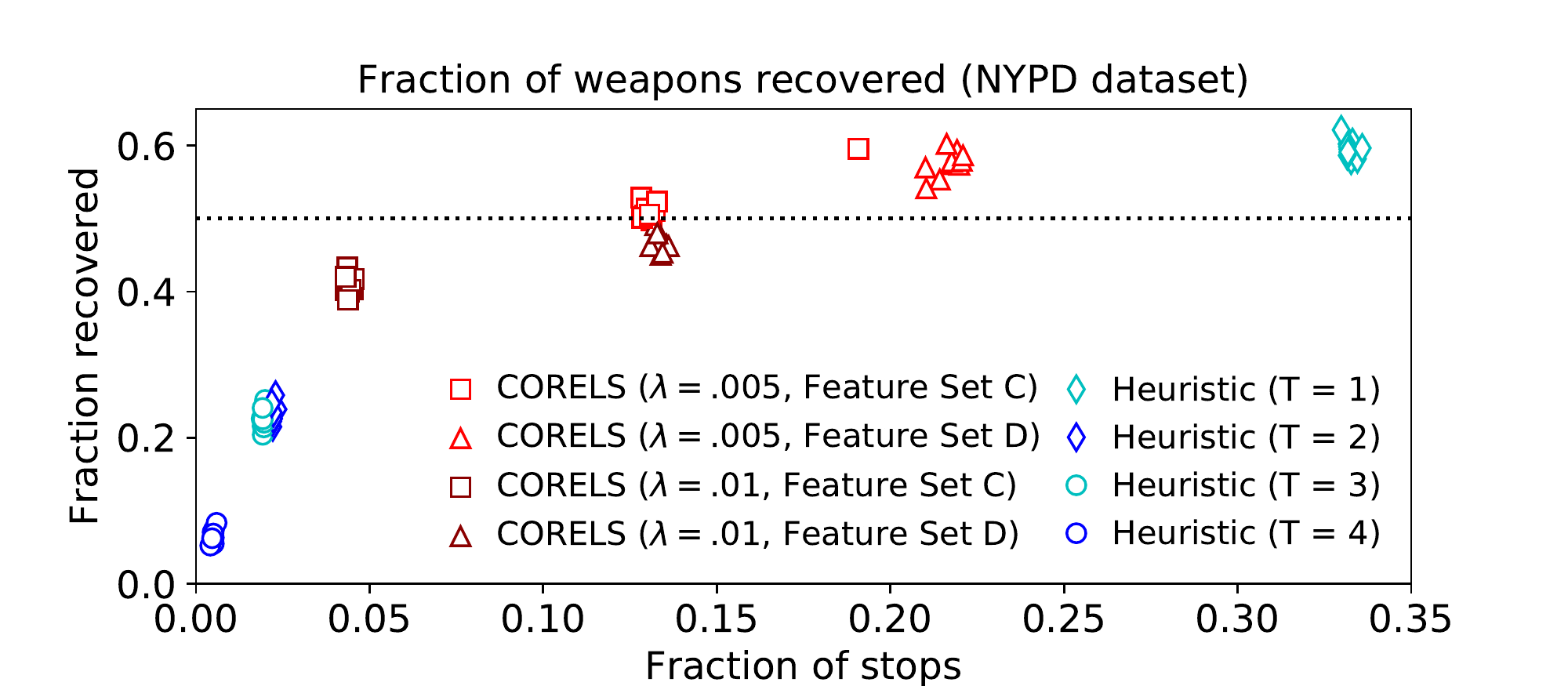}
\includegraphics[trim={12mm, 5mm, 24mm, 0mm},
width=0.71\textwidth]{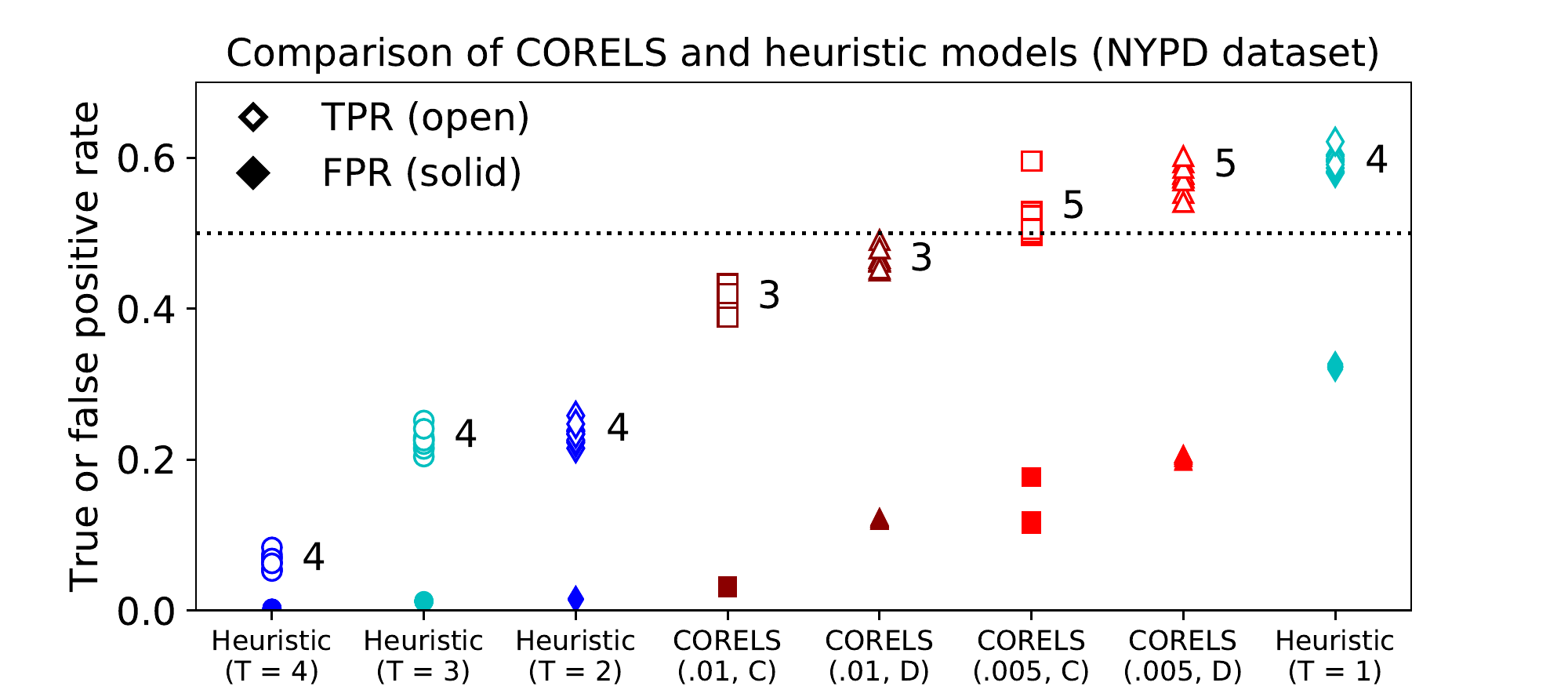}
\end{center}
\caption{Weapon prediction with the NYPD stop-and-frisk data set,
for various models learned by CORELS and the heuristic model by~\citet{Goel16},
across~10 cross-validation folds.
Note that the fraction of weapons recovered (top) is equal to the TPR (bottom, open markers).
Markers above the dotted horizontal lines at the value~0.5 correspond to models that
recover a majority of weapons (that are known in the data set).
Top: Fraction of weapons recovered as a function of the fraction of stops
where the individual was frisked and/or searched.
In the legend, entries for CORELS (red markers) indicate the regularization parameter~$(\Reg)$
and whether or not extra location features were used (``location'');
entries for the heuristic model (blue markers) indicate the threshold value~$(T)$.
The results we report for the heuristic model
are our reproduction of the results reported in Figure 9 by~\citet{Goel16}
(first four open circles in that figure, from left to right; we exclude the trivial open circle
showing 100\% of weapons recovered at 100\% of stops, obtained by setting the threshold at 0).
Bottom: Comparison of TPR (open markers) and FPR (solid markers) for various
CORELS and heuristic models.
Models are sorted left-to-right by TPR.
Markers and abbreviated horizontal tick labels correspond to the legend in the top~figure.
Numbers in the plot label model size; there was no variation in model size across folds,
except for a single fold for CORELS (${\Reg = 0.005}$, Feature Set~C), which found a model of size~6.
}
\label{fig:frisk}
\end{figure}

Figure~\ref{fig:frisk} directly compares various models learned by CORELS to the heuristic models,
using the same data set as~\citet{Goel16} and 10-fold cross-validation.
Recall that we train on resampled data to correct for class imbalance;
we evaluate with respect to test sets that have been formed without resampling.
For CORELS, the models correspond to the rule lists illustrated in Figure~\ref{fig:nypd}
from Section~\ref{sec:examples}, and Figures~\ref{fig:cpw-rule-list} and~\ref{fig:cpw-noloc-rule-list}
in Appendix~\ref{appendix:examples}, we consider both Feature Sets~C and~D
and both regularization parameters ${\Reg = 0.005}$ and~0.01.
The top panel plots the fraction of weapons recovered as a function of the fraction of
stops where the individual was frisked and/or searched.
\citet{Goel16} target models that efficiently recover a majority of weapons
(while also minimizing racial disparities, which we do not address here).
Interestingly, the models learned by CORELS span a significant region that is not available
to the heuristic model, which would require larger or non-integer parameters to access the region.
The region is possibly desirable, since it includes models (${\Reg = 0.005}$, bright red)
that recover a majority (${\ge 50\%}$) of weapons (that are known in the data set).
More generally, CORELS' models all recover at least ${40\%}$ of weapons
on average, \ie more weapons than any of the heuristic models with~${T \ge 2}$,
which recover less than 25\% of weapons on average.
At the same time, CORELS' models all require well under 25\% of stops---significantly
less than the heuristic model with~${T = 1}$, which requires over 30\% of stops
to recover a fraction of weapons comparable to the CORELS model that recovers the most weapons.

The bottom panel in Figure~\ref{fig:frisk} plots both TPR and FPR and labels model size,
for each of the models in the top panel.
For the heuristic, we define model size as the number of model parameters;
for CORELS, we use the number of rules in the rule list,
which is equal to the number of leaves when we view a rule list as a decision tree.
The heuristic models all have~4 parameters, while the different CORELS models have
either~3 or approximately~5 rules.
CORELS' models are thus approximately as small, interpretable, and transparent
as the heuristic models; furthermore, their predictions are straightforward
to compute, without even requiring arithmetic.

\subsection{Predictive Performance and Model Size for CORELS and Other Algorithms}
\label{sec:sparsity}

\begin{figure}[t!]
\begin{center}
%\includegraphics[width=0.75\textwidth]{figs/sketch-comparison.png}
% left lower right upper
%\includegraphics[trim={2mm, 10mm, 2mm, 0mm}, width=\textwidth]{figs/compare-compas-weapon.pdf}
\includegraphics[trim={2mm, 0mm, 101mm, 0mm}, clip,
width=0.55\textwidth]{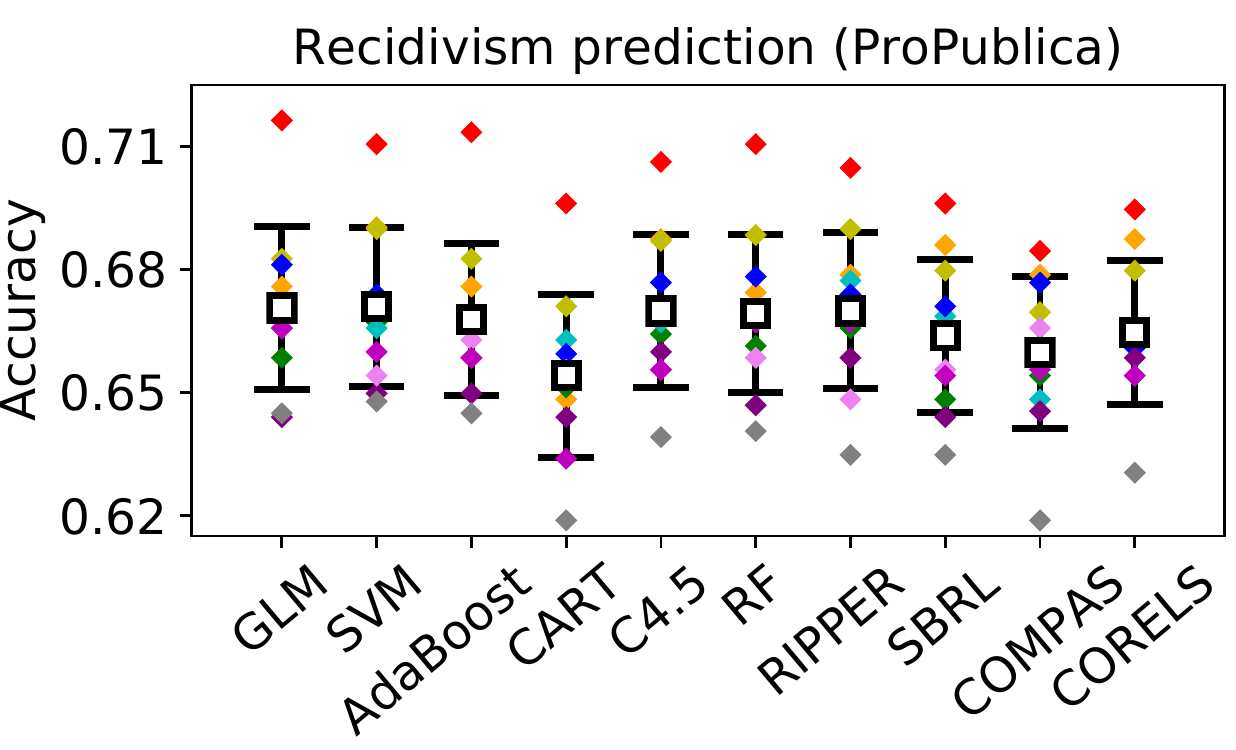}
\vspace{-7mm}
\end{center}
\caption{Two-year recidivism prediction for the ProPublica COMPAS data set.
Comparison of CORELS and a panel of nine other algorithms:
logistic regression~(GLM), support vector machines~(SVM),
AdaBoost, CART, C4.5, random forests~(RF), RIPPER,
scalable Bayesian rule lists~(SBRL), and COMPAS.
For CORELS, we use regularization parameter~${\Reg=0.005}$.
}
\label{fig:compas-comparison}
\end{figure}

We ran a 10-fold cross validation experiment using CORELS
and eight other algorithms:
logistic regression, support vector machines (SVM), AdaBoost, CART, C4.5,
random forest~(RF), RIPPER, and scalable Bayesian rule lists (SBRL).\footnote{For
SBRL, we use the C implementation at \url{https://github.com/Hongyuy/sbrlmod}.
By default, SBRL sets ${\eta = 3}$, ${\lambda = 9}$,
the number of chains to 11 and iterations to 1,000.}
We use standard R packages, with default parameter settings,
for the first seven algorithms.\footnote{For CART, C4.5 (J48), and RIPPER,
we use the R packages rpart, RWeka, and caret, respectively.
By default, CART uses complexity parameter ${cp = 0.01}$ and C4.5 uses complexity parameter ${C = 0.25}$.
}
We use the same antecedent sets as input to the two rule list learning algorithms, CORELS and SBRL;
for the other algorithms, the inputs are binary feature sets corresponding to the
single clause antecedents in the aforementioned antecedent sets (see Appendix~\ref{appendix:data}).

Figure~\ref{fig:compas-comparison} shows that for the ProPublica data set,
there were no statistically significant differences in test accuracies across algorithms,
the difference between folds was far larger than the difference between algorithms.
These algorithms also all perform similarly to the black box COMPAS algorithm.
Figure~\ref{fig:weapon-comparison} shows that for the NYCLU data set,
logistic regression, SVM, and AdaBoost have the highest TPR's and also the highest FPR's;
we show TPR and FPR due to class imbalance.
For this problem, CORELS obtains an intermediate TPR, compared to other algorithms,
while achieving a relatively low FPR.
We conclude that CORELS produces models whose predictive performance is comparable to or better than
those found via other algorithms.

\begin{figure}[t!]
\begin{center}
%\includegraphics[width=0.75\textwidth]{figs/sketch-comparison.png}
% left lower right upper
\includegraphics[trim={2mm, 10mm, 2mm, 0mm},
width=0.9\textwidth]{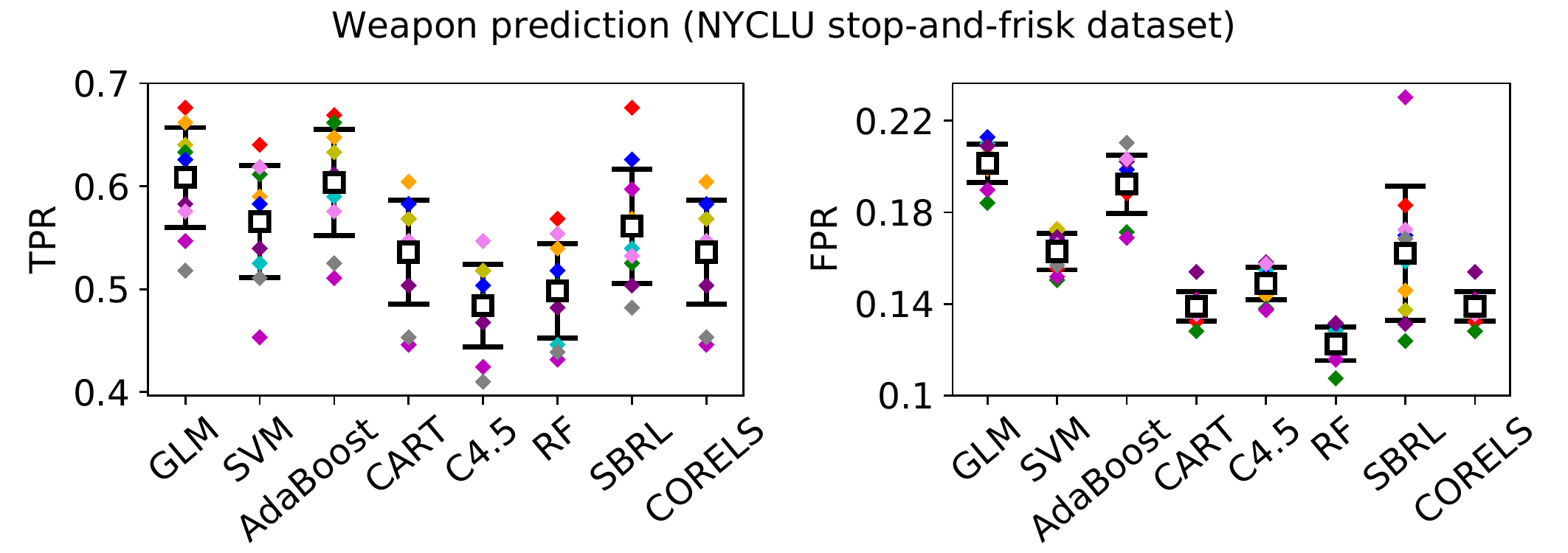}
\end{center}
\caption{TPR (left) and FPR (right) for the test set,
for CORELS and a panel of seven other algorithms,
for the weapon prediction problem with the NYCLU stop-and-frisk data set.
Means (white squares),
standard deviations (error bars),
and values (colors correspond to folds),
for 10-fold cross-validation experiments.
For CORELS, we use~${\Reg=0.01}$.
Note that we were unable to execute RIPPER for the NYCLU problem.
}
\label{fig:weapon-comparison}
\end{figure}
\begin{figure}[t!]
\begin{center}
%\includegraphics[width=0.75\textwidth]{figs/sketch-comparison.png}
% left lower right upper
\includegraphics[trim={12mm, 5mm, 24mm, 5mm},
width=0.7\textwidth]{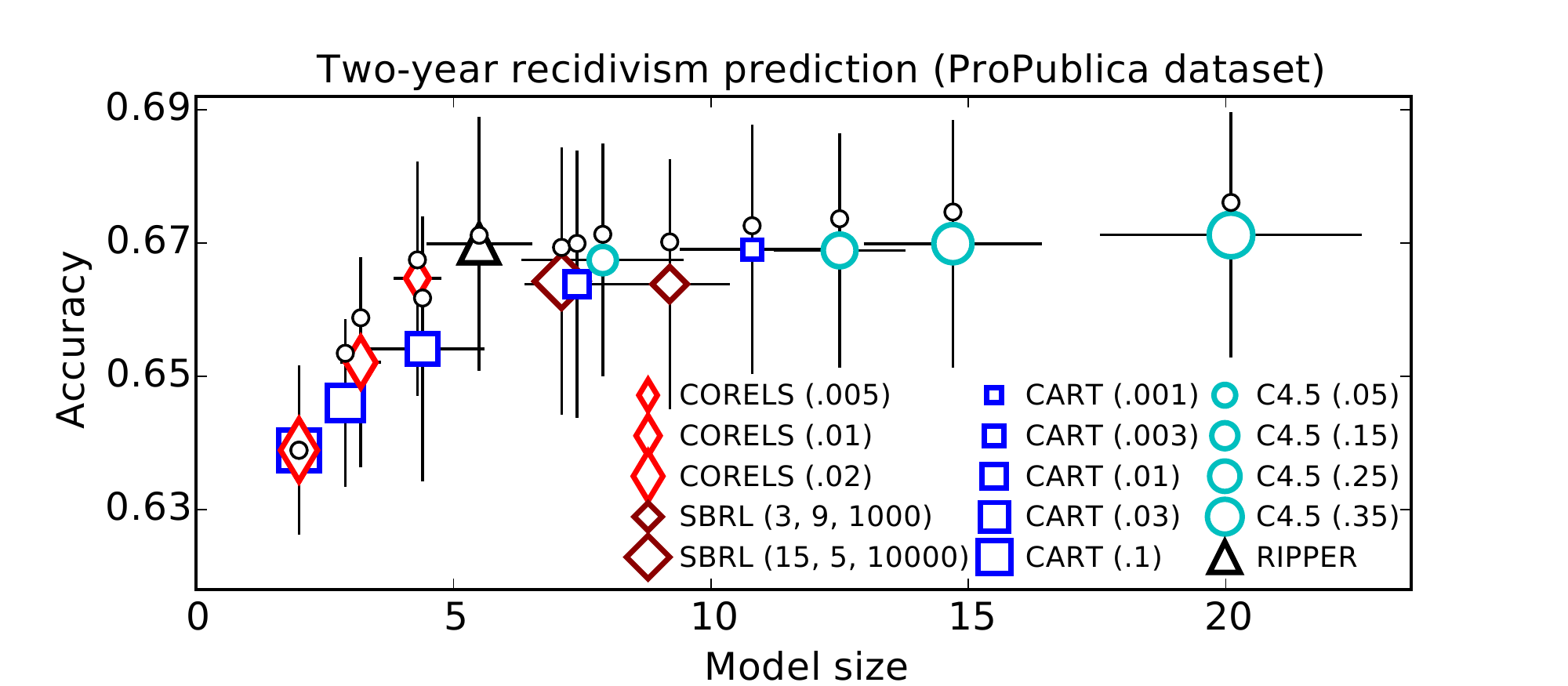}
\end{center}
\caption{Training and test accuracy as a function of model size, across different methods,
for two-year recidivism prediction with the ProPublica COMPAS data set.
In the legend, numbers in parentheses are algorithm parameters that we vary
for CORELS~($\Reg$), CART~($cp$), C4.5~($C$), and SBRL ($\eta$, $\lambda$, $i$),
where~$i$ is the number of iterations.
Legend markers and error bars indicate means and standard deviations,
respectively, across cross-validation folds.
Small circles mark training accuracy means.
None of the models exhibit significant overfitting;
mean training accuracy never exceeds mean test accuracy
by more than about 0.01.
}
\label{fig:sparsity-compas}
\end{figure}

Figures~\ref{fig:sparsity-compas} and~\ref{fig:sparsity-weapon} summarize differences
in predictive performance and model size
for CORELS and other tree (CART, C4.5) and rule list (RIPPER, SBRL) learning algorithms.
Here, we vary different algorithm parameters, and increase the number of iterations for SBRL to 10,000.
For two-year recidivism prediction with the ProPublica data set (Figure~\ref{fig:sparsity-compas}),
we plot both training and test accuracy,
as a function of the number of leaves in the learned model.
Due to class imbalance for the weapon prediction problem with the NYCLU stop-and-frisk data set
(Figure~\ref{fig:sparsity-weapon}), we plot both true positive rate (TPR) and false positive rate (FPR),
again as a function of the number of leaves.
For both problems, CORELS can learn short rule lists without sacrificing predictive performance.
For listings of example optimal rule lists that correspond to the results
for CORELS summarized here, see Appendix~\ref{appendix:examples}.
Also see Figure~\ref{fig:sparsity-cpw} in Appendix~\ref{appendix:cpw}; it uses the larger
NYPD data set and is similar to Figure~\ref{fig:sparsity-weapon}.

\begin{figure}[t!]
\begin{center}
%\includegraphics[width=0.7\textwidth]{figs/sketch-comparison.png}
% left lower right upper
\includegraphics[trim={17mm, 0mm, 27mm, 0mm},
width=0.7\textwidth]{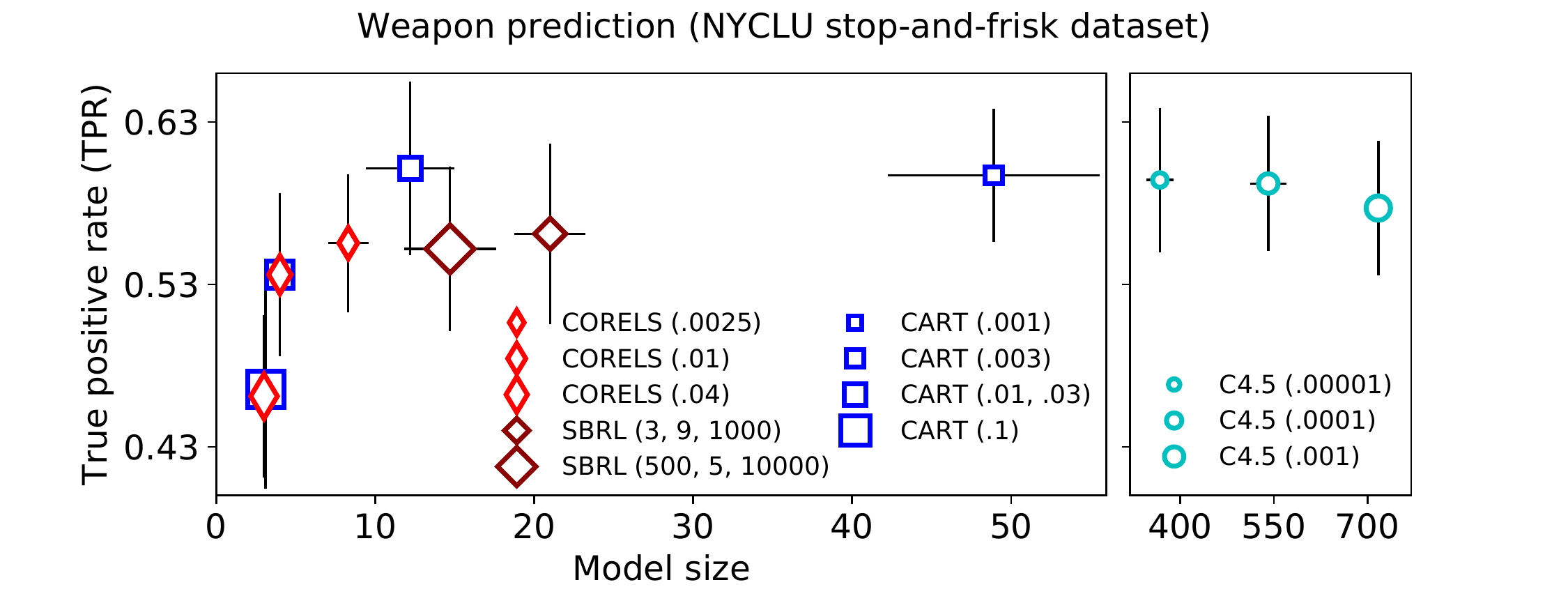}
\includegraphics[trim={17mm, 10mm, 27mm, 4mm},
width=0.7\textwidth]{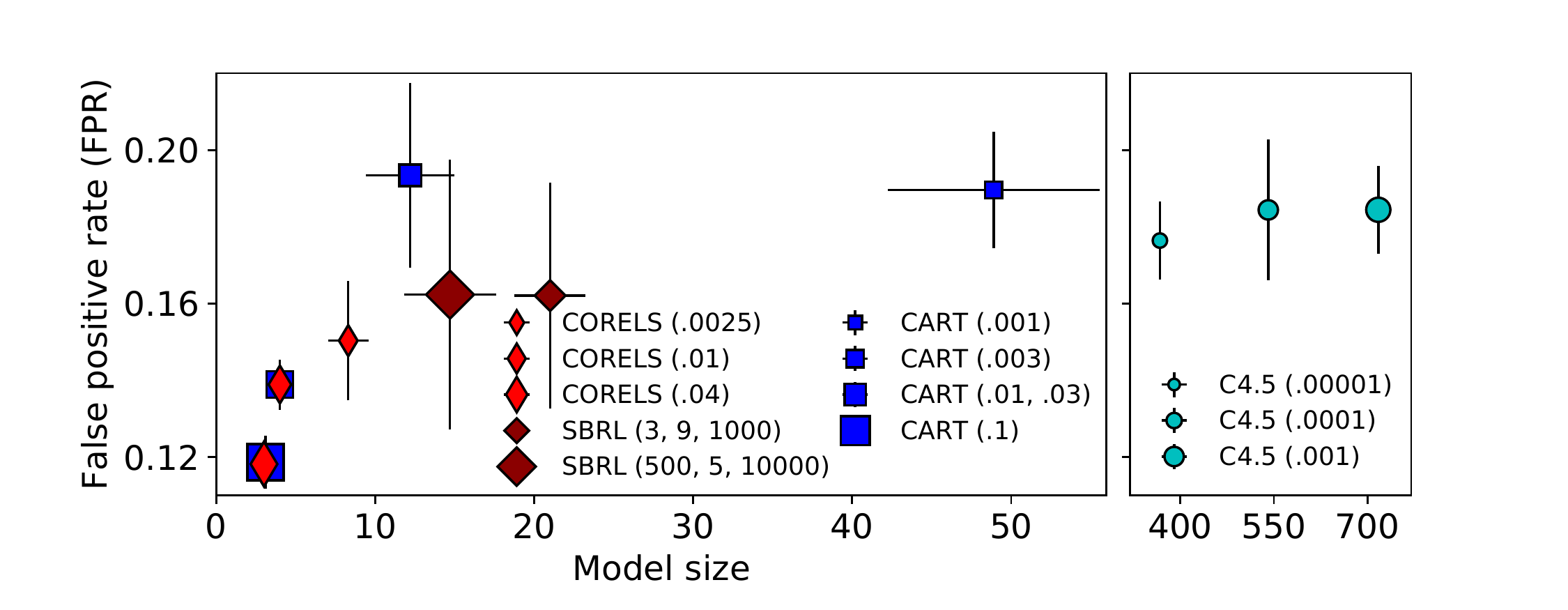}
\end{center}
\caption{TPR (top) and FPR (bottom)
for the test set, as a function of model size, across different methods,
for weapon prediction with the NYCLU stop-and-frisk data set.
In the legend, numbers in parentheses are algorithm parameters,
as in Figure~\ref{fig:sparsity-compas}.
Legend markers and error bars indicate means and standard deviations,
respectively, across cross-validation folds.
%
%For CORELS and SBRL, we use ${M = 28}$ antecedents.
%
C4.5 finds large models for all tested parameters.
}
\label{fig:sparsity-weapon}
\end{figure}

In Figure~\ref{fig:recidivism-all-folds}, we used CORELS to identify short
rule lists, depending on only three features---age, prior convictions, and sex---that
achieve test accuracy comparable to COMPAS
\citep[Figure~\ref{fig:compas-comparison}, also see][]{AngelinoLaAlSeRu17-kdd}.
If we restrict CORELS to search the space of rule lists formed from only age
and prior convictions (${\Reg=0.005}$), the optimal rule lists it finds achieve
test accuracy that is again comparable to COMPAS.
CORELS identifies the same rule list across all 10 folds of 10-fold
cross-validation experiments (Figure~\ref{fig:compas-two-features}).
In work subsequent to ours~\citep{AngelinoLaAlSeRu17-kdd}, \citet{Dressel2018}
confirmed this result, in the sense that they used logistic regression to
construct a linear classifier with age and prior convictions,
and also achieved similar accuracy to COMPAS.
However, computing a logistic regression model requires multiplication and
addition, and their model cannot easily be computed, in the sense that it
requires a calculator (and thus is potentially error-prone).
Our rule lists require no such computation.

\begin{figure}[t!]
%\vspace{2mm}
\begin{algorithmic}
\State \bif $(priors > 3)$ \bthen $yes$
\State \belif $(age < 25) \band (priors = 2-3)$ \bthen $yes$
\State \belse $no$
\end{algorithmic}
%\vspace{1mm}
\caption{When restricted to two features (age, priors), CORELS (${\Reg=0.005}$)
finds the same rule list across 10 cross-validation folds.}
\label{fig:compas-two-features}
\end{figure}

\subsection{CORELS Execution Traces, for Different Regularization Parameters}
\label{sec:reg-param}
In this section, we illustrate several views of CORELS execution traces,
for the NYCLU stop-and-frisk data set with ${M = 46}$ antecedents,
for the same three regularization parameters (${\Reg = .04, .01, .025}$)
as in Figure~\ref{fig:sparsity-weapon}.

\begin{table}[t!]
\centering
%\begin{centering}
CORELS with different regularization parameters (NYCLU stop-and-frisk data set) \\
%\end{centering}
\vspace{2mm}
\begin{tabular}{l | c | c | c | c}
& Total & Time to & Max evaluated & Optimal \\
$\lambda$ & time (s) & optimum (s) & prefix length & prefix length \\
\hline
.04 & ~.61 (.03) & .002 (.001) & 6 & 2 \\
.01 & 70 (6) & .008 (.002) & 11 & 3 \\
.0025 & 1600 (100) & 56 (74) & 16-17 & 6-10 \\
\hline
\end{tabular}
\begin{tabular}{l | c | c | c}
\hline
& Lower bound & Total queue &  Max queue \\
$\lambda$ &~ evaluations ($\times 10^6$) ~&~ insertions ($\times 10^3$) ~&~ size ($\times 10^3$) \\
\hline
.04 & ~.070 (.004) & 2.2 (.1) & ~.9 (.1) \\
.01 & 7.5 (.6) & 210 (20) & 130 (10) \\
.0025 & 150 (10) & 4400 (300) & 2500 (170) \\
\end{tabular}
%\vspace{4mm}
\caption{Summary of CORELS executions, for the NYCLU stop-and-frisk data set (${M = 46}$),
for same three regularization parameter ($\Reg$) values as in Figure~\ref{fig:sparsity-weapon}.
The columns report the total execution time,
time to optimum, maximum evaluated prefix length, optimal prefix length,
number of times we completely evaluate a prefix~$\Prefix$'s lower bound~$b(\Prefix, \x, \y)$,
total number of queue insertions (this number is equal to the number of cache insertions),
and the maximum queue size.
For prefix lengths, we report single values or ranges corresponding to the minimum and maximum observed values;
in the other columns, we report means (and standard deviations) over 10 cross-validation folds.
See also Figures~\ref{fig:weapon-reg-execution} and~\ref{fig:queue-weapon-reg}.
}
\vspace{4mm}
\label{tab:weapon-reg}
\end{table}

Table~\ref{tab:weapon-reg} summarizes execution traces across all 10 cross-validation folds.
For each value of~$\Reg$, CORELS achieves the optimum in a small fraction of the total execution time.
As~$\Reg$ decreases, these times increase because the search problems become more difficult,
as is summarized by the observation that CORELS must evaluate longer prefixes;
consequently, our data structures grow in size.
We report the total number of elements inserted into the queue and the maximum queue size;
recall from~\S\ref{sec:implementation} that the queue elements correspond to the trie's leaves,
and that the symmetry-aware map elements correspond to the trie's nodes.

\begin{figure}[t!]
\begin{center}
% left lower right upper
\includegraphics[trim={35mm 0mm 35mm 15mm},
width=0.81\textwidth]{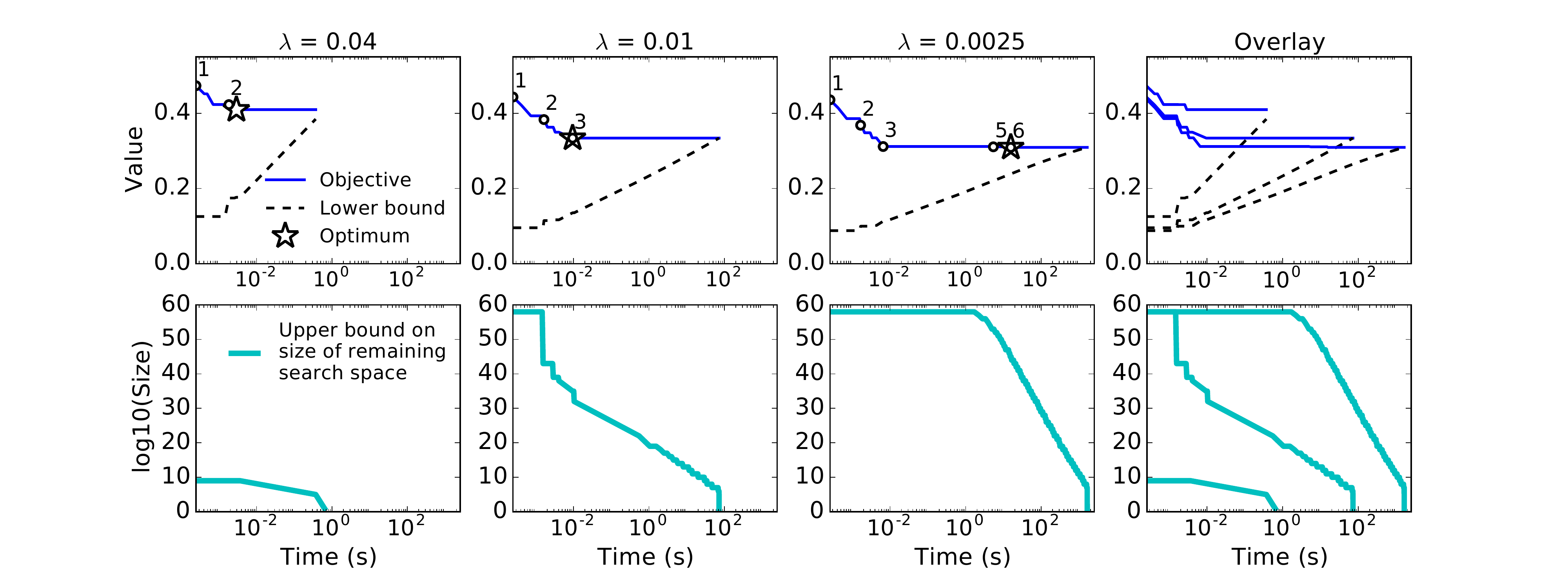}
\end{center}
\vspace{-5mm}
\caption{Example executions of CORELS, for the NYCLU stop-and-frisk data set (${M = 46}$).
See also Table~\ref{tab:weapon-reg} and Figure~\ref{fig:queue-weapon-reg}.
Top: Objective value (solid line) and lower bound (dashed line) for CORELS,
as a function of wall clock time (log scale).
Numbered points along the trace of the objective value
indicate when the length of the best known rule list changes
and are labeled by the new length.
For each value of~$\Reg$, a star marks the optimum objective value
and time at which it was achieved.
Bottom: $\lfloor \log_{10} \Remaining(\CurrentObj, \Queue) \rfloor$,
as a function of wall clock time (log scale),
where~$\Remaining(\CurrentObj, \Queue)$
is the upper bound on remaining search space size
(Theorem~\ref{thm:remaining-eval-fine}).
Rightmost panels: For visual comparison, we overlay the execution traces
from the panels to the left, for the three different values of~$\Reg$.
}
\label{fig:weapon-reg-execution}
\end{figure}
\begin{figure}[t!]
\begin{center}
% left lower right upper
\includegraphics[trim={30mm 0mm 30mm 3mm},
width=0.92\textwidth]{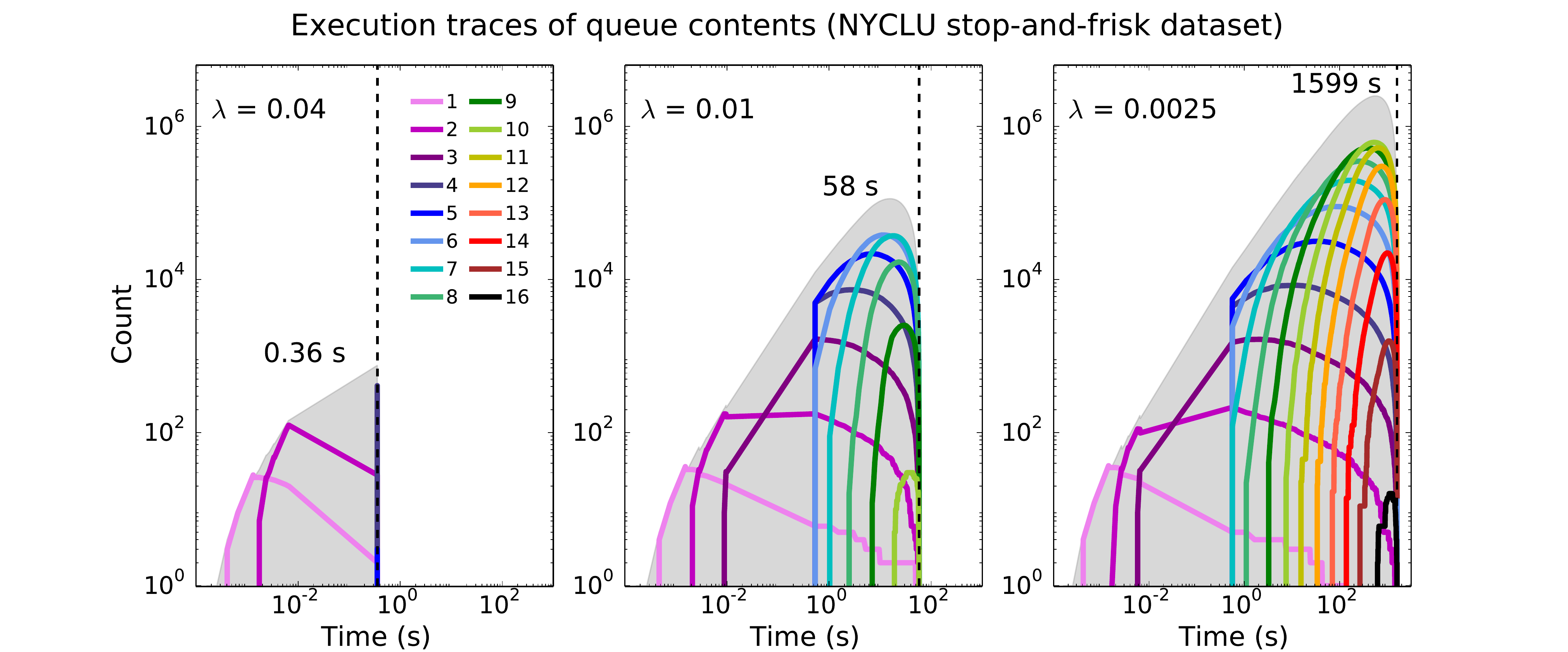}
\end{center}
\vspace{-5mm}
\caption{Summary of CORELS' logical queue,
%\ie the nodes in the physical queue that have not been marked for deletion,
for the NYCLU stop-and-frisk data set (${M = 46}$),
for same three regularization parameters as in Figure~\ref{fig:sparsity-weapon}
and Table~\ref{tab:weapon-reg}.
Solid lines plot the numbers of prefixes in the logical queue (log scale), colored by length (legend),
as a function of wall clock time (log scale).
All plots are generated using a single, representative cross-validation training set.
For each execution, the gray shading fills in the area beneath the total number
of queue elements, \ie the sum over all lengths;
we also annotate the total time in seconds, marked with a dashed vertical line.
}
\label{fig:queue-weapon-reg}
\end{figure}

The upper panels in Figure~\ref{fig:weapon-reg-execution} plot example execution traces,
from a single cross-validation fold, of both the current best objective value~$\CurrentObj$
and the lower bound~$b(\Prefix, \x, \y)$ of the prefix~$\Prefix$ being evaluated.
These plots illustrate that CORELS certifies optimality
when the lower bound matches the objective value.
The lower panels in Figure~\ref{fig:weapon-reg-execution} plot corresponding traces of
an upper bound on the size of the remaining search space (Theorem~\ref{thm:remaining-eval-fine}),
and illustrate that as~$\Reg$ decreases, it becomes more difficult to eliminate regions of the search~space.
For Figure~\ref{fig:weapon-reg-execution}, we dynamically and incrementally
calculate~$\lfloor \log_{10} \Remaining(\CurrentObj, \Queue) \rfloor$,
which adds some computational overhead; we do not calculate this elsewhere unless noted.

Figure~\ref{fig:queue-weapon-reg} visualizes the elements in CORELS' logical queue,
for each of the executions in Figure~\ref{fig:weapon-reg-execution}.
Recall from~\S\ref{sec:gc} that the logical queue corresponds to elements in the
(physical) queue that have not been garbage collected from the trie; these are prefixes that
CORELS has already evaluated and whose children the algorithm plans to evaluate next.
As an execution progresses, longer prefixes are placed in the queue;
as~$\Reg$ decreases, the algorithm must spend more time evaluating longer and longer prefixes.

%\newpage
\subsection{Efficacy of CORELS Algorithm Optimizations}
\label{sec:ablation}

\begin{table}[t!]
\centering
%\begin{centering}
Per-component performance improvement (ProPublica data set) \\
%\end{centering}
\vspace{1mm}
\begin{tabular}{l | c  r | c | c}
& Total time & Slow- & Time to & Max evaluated \\
Algorithm variant & (min) & down & optimum (s) & prefix length \\
\hline
CORELS & ~.98 (.6) & --- & ~~1 (1) & 5 \\
No priority queue (BFS) & 1.03 (.6) & 1.1$\times$ & ~~2 (4) & 5 \\
No support bounds & ~1.5 (.9) & 1.5$\times$ & ~~1 (2) & 5 \\
No lookahead bound & ~12.3 (6.2) & 13.3$\times$ & ~~1 (1) & 6 \\
No symmetry-aware map & ~~~9.1 (6.4) & 8.4$\times$ & ~~2 (3) & 5 \\
No equivalent points bound* & $>$130 (2.6) & $>$180$\times$ & $>$1400 (2000) & $\ge$11 \\
\hline
\end{tabular}
\begin{tabular}{l | c | c | c}
\hline
 & Lower bound & Total queue &  Max queue \\
Algorithm variant & evaluations ($\times 10^6$) & insertions ($\times 10^6$) &~ size ($\times 10^6$) \\
\hline
CORELS & 26 (15) & .29 (.2) & .24 (.1) \\
No priority queue (BFS) & 27 (16) & .33 (.2) & .20 (.1) \\
No support bounds & 42 (25) & .40 (.2) & .33 (.2) \\
No lookahead bound & 320 (160) & 3.6 (1.8) & 3.0 (1.5) \\
No symmetry-aware map & 250 (180) & 2.5 (1.7) & 2.4 (1.7) \\
No equivalent points bound* & $>$940 (5)~~~~~ & $>$510 (1.1)~~~ & $>$500 (1.2)~~~ \\
\end{tabular}
%\vspace{4mm}
\caption{Per-component performance improvement, for the ProPublica data set
(${\Reg = 0.005}$, ${M = 122}$).
The columns report the total execution time,
time to optimum, maximum evaluated prefix length,
number of times we completely evaluate a prefix~$\Prefix$'s lower bound~$b(\Prefix, \x, \y)$,
total number of queue insertions (which is equal to the number of cache insertions),
and maximum logical queue size.
The first row shows CORELS; subsequent rows show variants
that each remove a specific implementation optimization or bound.
(We are not measuring the cumulative effects of removing a sequence of components.)
All rows represent complete executions that certify optimality,
except those labeled `No equivalent points bound,'
for which each execution was terminated due to memory constraints,
once the size of the cache reached ${5 \times 10^8}$ elements,
after consuming $\sim$250GB RAM.
In all but the final row and column, we report means
(and standard deviations) over 10 cross-validation folds.
We also report the  mean slowdown in total execution time,
with respect to CORELS.
In the final row, we report the mean (and standard deviation) of the
incomplete execution time and corresponding slowdown,
and a lower bound on the mean time to optimum;
in the remaining fields, we report minimum values across folds.
See also Figure~\ref{fig:queue}. \\
*~Only 7 out of 10 folds achieve the optimum before being terminated.
}
\vspace{4mm}
\label{tab:ablation}
\end{table}

This section examines the efficacy of each of our bounds and data structure optimizations.
We remove a single bound or data structure optimization from our final implementation and measure
how the performance of our algorithm changes.
We examine these performance traces on both the NYCLU and the ProPublica data sets,
and highlight the result that on different problems, the relative performance improvements
of our optimizations can vary.

\begin{figure}[t!]
\begin{center}
% left lower right upper
\includegraphics[trim={0mm 0mm 0mm 15mm}, width=\textwidth]{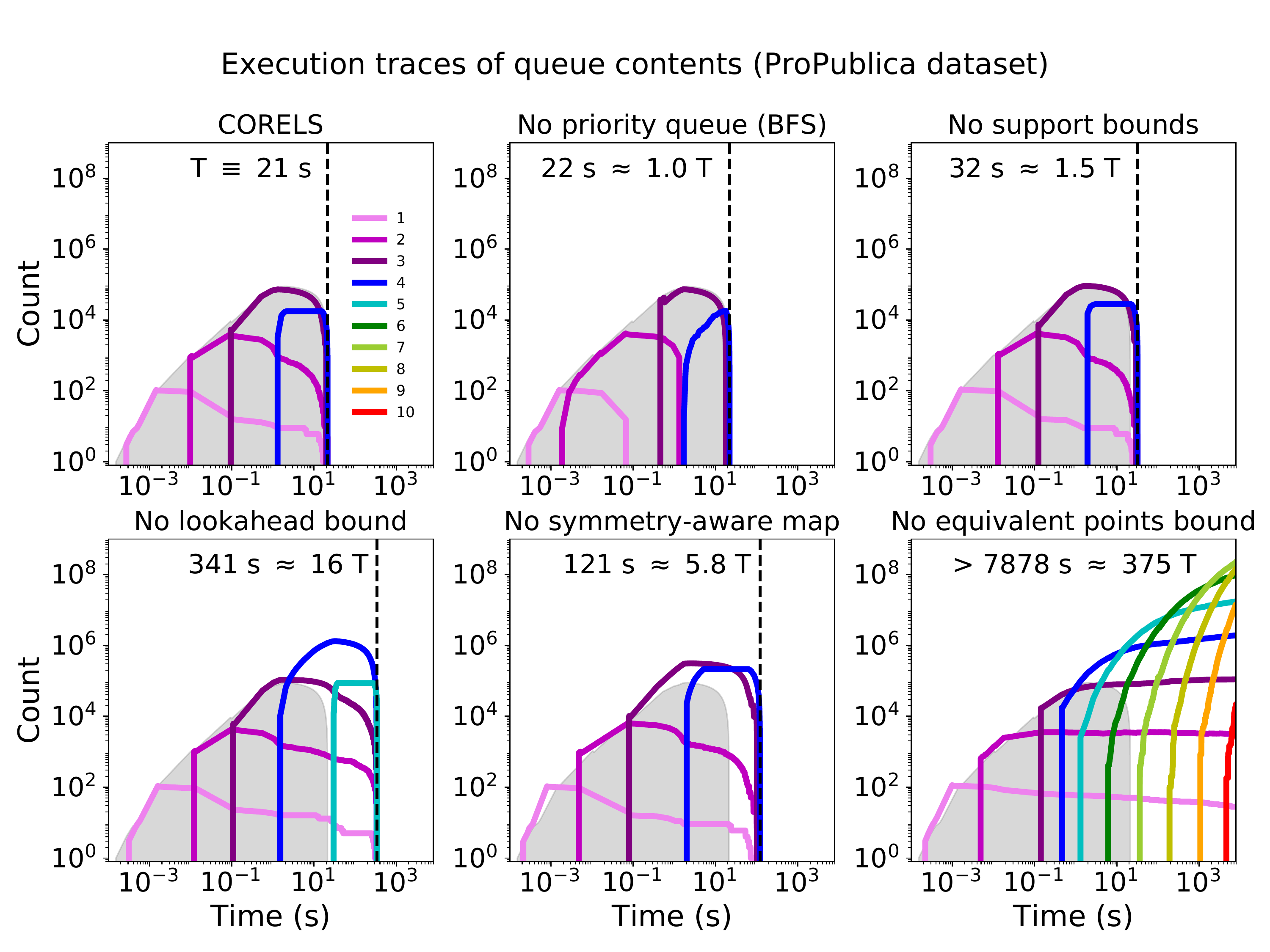}
\end{center}
\vspace{-5mm}
\caption{Summary of the logical queue's contents, for full CORELS (top left)
%\ie the nodes in the physical queue that have not been marked for deletion,
and five variants that each remove a specific implementation optimization or bound,
for the ProPublica data set (${\Reg = 0.005}$, ${M = 122}$).  See also Table~\ref{tab:ablation}.
Solid lines plot the numbers of prefixes in the logical queue (log scale), colored by length (legend),
as a function of wall clock time (log scale).
All plots are generated using a single, representative cross-validation training set.
The gray shading fills in the area beneath the total number of
queue elements for CORELS,
\ie the sum over all lengths in the top left figure.
For comparison, we replicate the same gray region
in the other five subfigures.
For each execution, we indicate the total time in seconds,
relative to the full CORELS implementation (T = 21 s),
and with a dashed vertical line.
The execution without the equivalent points bound (bottom right) is incomplete.
}
\label{fig:queue}
\end{figure}

Table~\ref{tab:ablation} provides summary statistics for experiments using the full
CORELS implementation (first row) and five variants (subsequent rows) that each remove a specific optimization:
(1)~Instead of a priority queue (\S\ref{sec:queue}) ordered by the objective lower bound,
we use a queue that implements breadth-first search (BFS).
(2)~We remove checks that would trigger pruning via
our lower bounds on antecedent support (Theorem~\ref{thm:min-capture})
and accurate antecedent support (Theorem~\ref{thm:min-capture-correct}).
(3)~We remove the effect of our lookahead bound (Lemma~\ref{lemma:lookahead}),
which otherwise tightens the objective lower bound by an amount equal to the regularization parameter~$\Reg$.
(4)~We disable the symmetry-aware map (\S\ref{sec:pmap}), our data structure
that enables pruning triggered by the permutation bound (Corollary~\ref{thm:permutation}).
(5)~We do not identify sets of equivalent points, which we otherwise use to tighten the
objective lower bound via the equivalent points bound (Theorem~\ref{thm:identical}).

\begin{table}[t!]
\centering
%\begin{centering}
Per-component performance improvement (NYCLU stop-and-frisk data set) \\
%\end{centering}
\vspace{1mm}
\begin{tabular}{l | c  r | c | c}
& Total & Slow- & Time to & Max evaluated \\
Algorithm variant & time (min) & down & optimum ($\mu$s) & prefix length \\
\hline
CORELS & ~~~1.1 (.1) & --- & 8.9 (.1) & 11 \\
No priority queue (BFS) & ~~~2.2 (.2) & 2.0$\times$ & 110 (10) & 11 \\
No support bounds & ~~~1.2 (.1) & 1.1$\times$ & 8.8 (.8) & 11 \\
No lookahead bound & ~~~1.7 (.2) & 1.6$\times$ & ~7.3 (1.8) & 11-12 \\
No symmetry-aware map & $>$ 73 (5) & $>$ 68$\times$ & $>$ 7.6 (.4)~~~~ & $>$ 10~~~ \\
No equivalent points bound & ~~~~~4 (.3) & 3.8$\times$ & 6.4 (.9) & 14 \\
\hline
\end{tabular}
\begin{tabular}{l | c | c | c}
\hline
 & Lower bound & Total queue &  Max queue \\
Algorithm variant & evaluations ($\times 10^6$) & insertions ($\times 10^5$) &~ size ($\times 10^5$) \\
\hline
CORELS & ~7 (1) & 2.0 (.2) & 1.3 (.1) \\
No priority queue (BFS) & 14 (1) & 4.1 (.4) & 1.4 (.1) \\
No support bounds & ~8 (1) & 2.1 (.2) & 1.3 (.1) \\
No lookahead bound & 11 (1) & 3.2 (.3) & 2.1 (.2) \\
No symmetry-aware map & $>$ 390 (40)~~~~ & $>$ 1000 (0)~~~~~~ & $>$ 900 (10)~~~ \\
No equivalent points bound & 33 (2) & 9.4 (.7) & 6.0 (.4) \\
\end{tabular}
%\vspace{4mm}
\caption{Per-component performance improvement, as in Table~\ref{tab:ablation},
for the NYCLU stop-and-frisk data set (${\Reg = 0.01}$, ${M = 46}$).
All rows except those labeled `No symmetry-aware map' represent complete executions.
A single fold running without a symmetry-aware map required over 2 days to complete,
so in order to run all 10 folds above, we terminated execution after the prefix tree~(\S\ref{sec:trie})
reached~$10^8$ nodes.
See Table~\ref{tab:ablation} for a detailed caption,
and also Figure~\ref{fig:queue-weapon}.
}
\vspace{4mm}
\label{tab:ablation-weapon}
\end{table}
\begin{figure}[t!]
\begin{center}
% left lower right upper
\includegraphics[trim={0mm 0mm 0mm 15mm}, width=\textwidth]{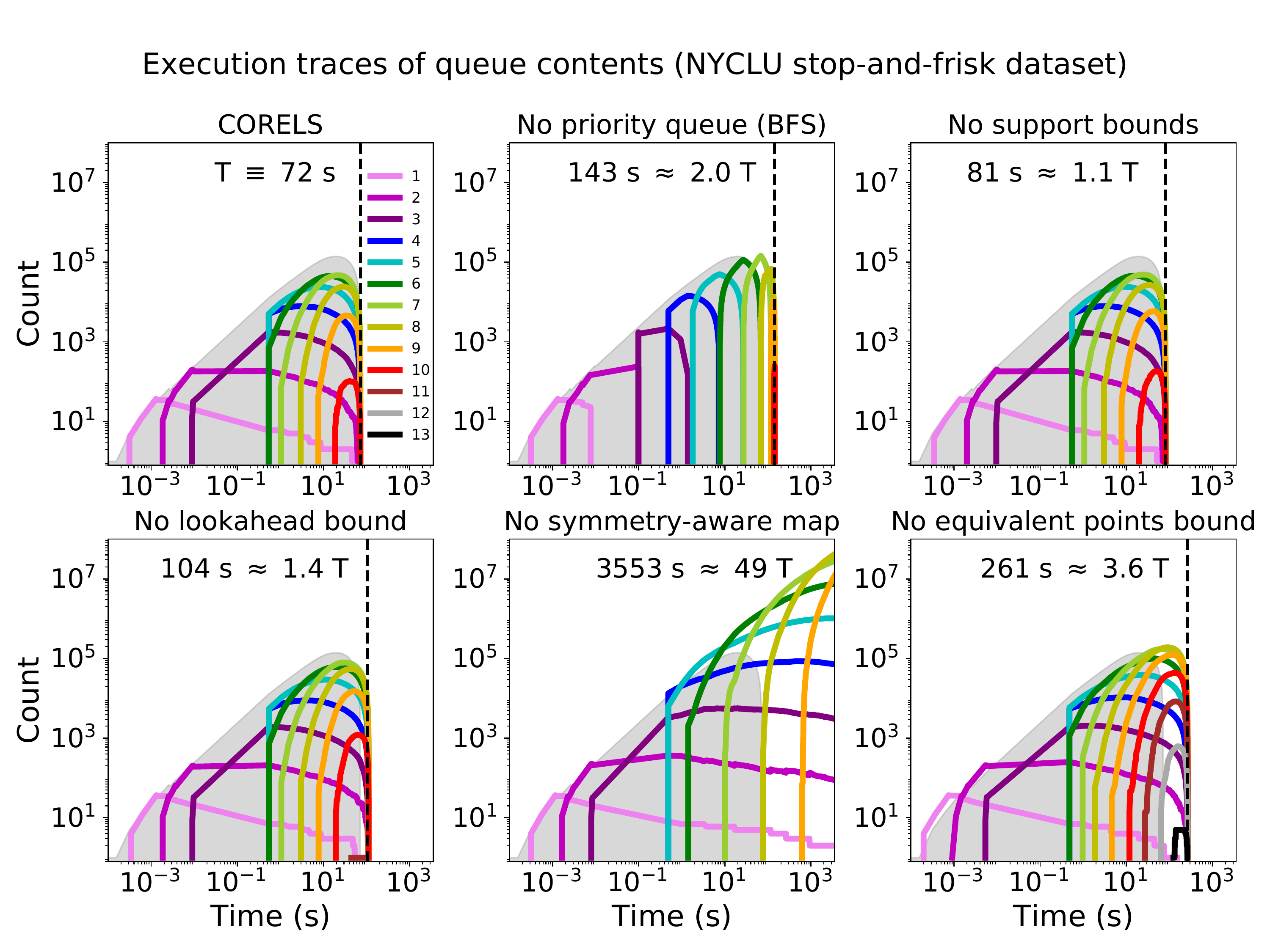}
\end{center}
\vspace{-5mm}
\caption{Summary of the logical queue's contents, for full CORELS (top left)
and five variants that each remove a specific implementation optimization or bound,
for the NYCLU stop-and-frisk data set (${\Reg = 0.01}$, ${M = 46}$), as in Table~\ref{tab:ablation-weapon}.
The execution without the symmetry-aware map (bottom center) is incomplete.
See Figure~\ref{fig:queue} for a detailed caption.
}
\label{fig:queue-weapon}
\end{figure}

Removing any single optimization increases total execution time,
by varying amounts across these optimizations.
Similar to our experiments in~\S\ref{sec:reg-param}, we always encounter the
optimal rule list in far less time than it takes to certify optimality.
As in Table~\ref{tab:weapon-reg}, we report metrics that are all proxies
for how much computational work our algorithm must perform;
these metrics thus scale with the overall slowdown with respect to CORELS execution time
(Table~\ref{tab:ablation}, first column).

Figure~\ref{fig:queue} visualizes execution traces of the elements in CORELS' logical queue,
similar to Figure~\ref{fig:queue-weapon-reg},
for a single, representative cross-validation fold.
Panels correspond to different removed optimizations, as in Table~\ref{tab:ablation}.
These plots demonstrate that our optimizations reduce the number of evaluated prefixes
and are especially effective at limiting the number of longer evaluated prefixes.
For the ProPublica data set, the most important optimization is the equivalent points
bound---without it, we place prefixes of at least length~10 in our queue,
and must terminate these executions before they are complete.
In contrast, CORELS and most other variants evaluate only prefixes up to at most length~5,
except for the variant without the lookahead bound, which evaluates prefixes up to length~6.

Table~\ref{tab:ablation-weapon} and Figure~\ref{fig:queue-weapon} summarize an
analogous set of experiments for the NYCLU data set.
Note that while the equivalent points bound proved to be the most important optimization for the ProPublica data set,
the symmetry-aware map is the crucial optimization for the NYCLU data set.

\begin{figure}[t!]
\begin{center}
%\includegraphics[width=0.65\textwidth]{figs/sketch-objective.png}
% left lower right upper
\includegraphics[trim={30mm, 20mm, 30mm, 20mm},
width=0.95\textwidth]{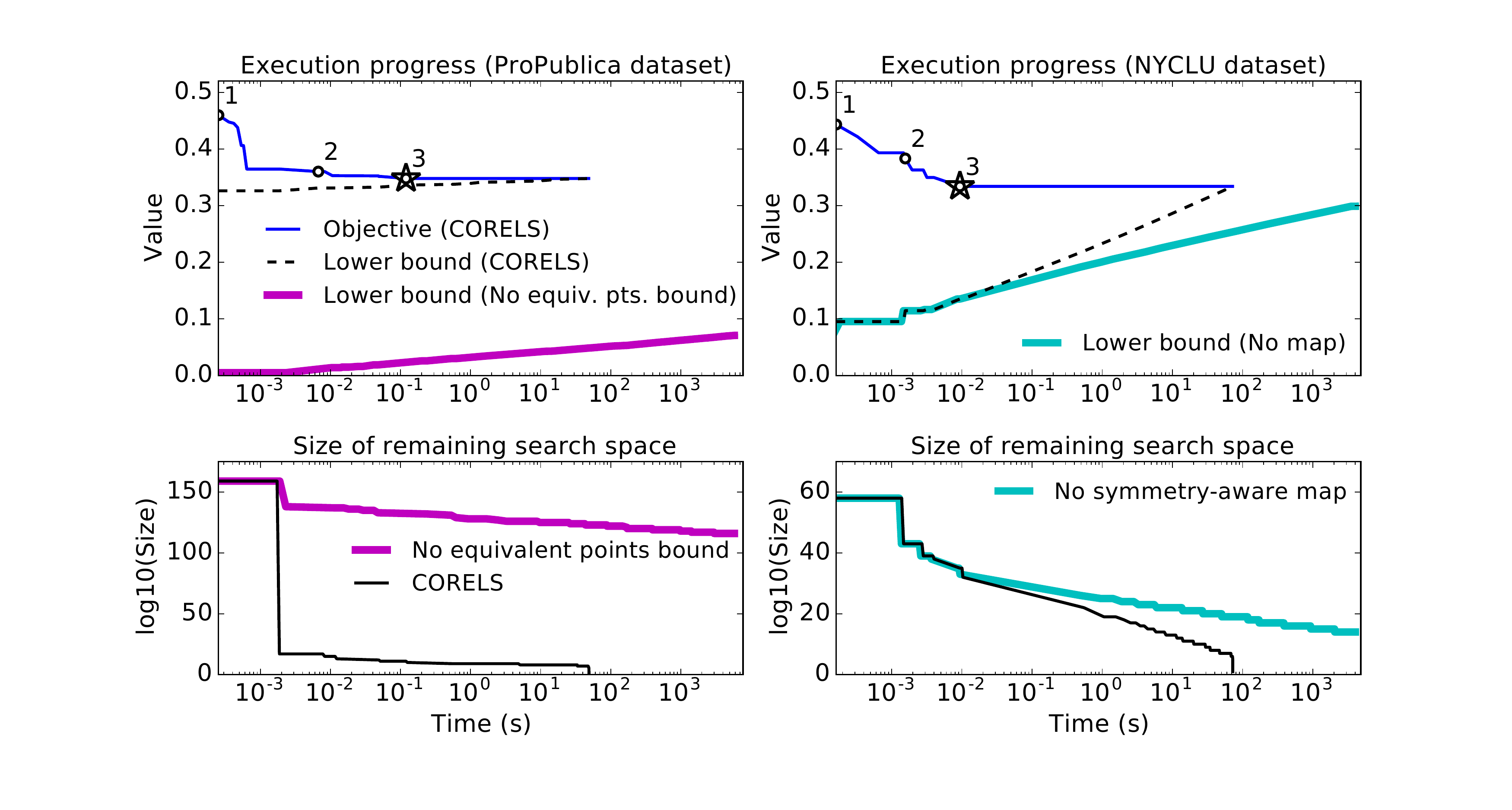}
\end{center}
\caption{Execution progress of CORELS and selected variants,
for the ProPublica (${\Reg = 0.005}$, ${M = 122}$) (left)
and NYCLU (${\Reg = 0.01}$, ${M = 46}$) (right) data sets.
Top: Objective value (thin solid lines) and lower bound (dashed lines) for CORELS,
as a function of wall clock time (log scale).
Numbered points along the trace of the objective value
indicate when the length of the best known rule list changes,
and are labeled by the new length.
CORELS quickly achieves the optimal value (star markers),
and certifies optimality when the lower bound matches the objective value.
On the left, a separate and significantly longer execution of CORELS
without the equivalent points  (Theorem~\ref{thm:identical}) bound remains
far from complete, and its lower bound (thick solid line) far from the optimum.
On the right, a separate execution of CORELS without the permutation bound
(Corollary~\ref{thm:permutation}), and thus the symmetry-aware map,
requires orders of magnitude more time to complete.
Bottom: $\lfloor \log_{10} \Remaining(\CurrentObj, \Queue) \rfloor$,
as a function of wall clock time (log scale),
where~$\Remaining(\CurrentObj, \Queue)$
is the upper bound~\eqref{eq:remaining} on remaining search space size
(Theorem~\ref{thm:remaining-eval-fine}).
For these problems, the equivalent points (left) and
permutation (right) bounds are responsible for the ability of
CORELS to quickly eliminate most of the search space (thin solid lines);
the remaining search space decays much more slowly without these bounds (thick solid lines).
}
\label{fig:objective}
\end{figure}

Finally, Figure~\ref{fig:objective} highlights the most significant
algorithm optimizations for our prediction problems:
the equivalent points bound for the ProPublica data set (left)
and the symmetry-aware map for the NYCLU data set (right).
For CORELS (thin lines) with the ProPublica recidivism data set (left),
the objective drops quickly, achieving the optimal value within a second.
CORELS certifies optimality in about a minute---the objective lower bound
steadily converges to the optimal objective (top) as the search space shrinks (bottom).
As in Figure~\ref{fig:weapon-reg-execution}, we dynamically and incrementally
calculate~$\lfloor \log_{10} \Remaining(\CurrentObj, \Queue) \rfloor$,
where~$\Remaining(\CurrentObj, \Queue)$
is the upper bound~\eqref{eq:remaining} on remaining search space size
(Theorem~\ref{thm:remaining-eval-fine}); this adds some computational overhead.
In the same plots (left), we additionally highlight
a separate execution of CORELS without the equivalent points bound
(Theorem~\ref{thm:identical}) (thick lines).
After more than 2 hours, the execution is still far from complete;
in particular, the lower bound is far from the optimum objective value (top)
and much of the search space remains unexplored~(bottom).
For the NYCLU stop-and-frisk data set (right),
CORELS achieves the optimum objective in well under a second,
and certifies optimality in a little over a minute.
CORELS without the permutation bound (Corollary~\ref{thm:permutation}),
and thus the symmetry-aware map,
requires more than an hour, \ie orders of magnitude more time, to complete (thick lines).

%\clearpage
\subsection{Algorithmic Speedup}
\label{sec:speedup}

\begin{table}[t!]
%\vspace{5mm}
\centering
\begin{tabular}{l|c|c|r l}
Algorithmic & Max evaluated & Lower bound & Predicted&runtime \\
approach & prefix length & evaluations & \\
\hline
CORELS & 5  & $ 2.8 \times 10^7$ & $36$ seconds& \\
Brute force & 5 & ~$2.5 \times 10^{10}$ & $9.0$ hours&$\approx 3.3 \times 10^4$ s \\
Brute force & 10 & ~$5.0 \times 10^{20}$ & $21 \times 10^6$ years&$\approx 6.5 \times 10^{14}$ s \\
CORELS (1984) & 5 & $2.8 \times 10^7$ & $13.5$ days&$\approx 1.2 \times 10^6$ s \\
\end{tabular}
\caption{Algorithmic speedup for the ProPublica data set (${\Reg = 0.005}$, ${M = 122}$).
Solving this problem using brute force is impractical due to the inability to explore rule lists of reasonable lengths.
Removing only the equivalent points bound requires exploring prefixes of up to length 10 (see Table \ref{tab:ablation}), a clearly intractable problem.
Even with all of our improvements, however, it is only recently that processors have been fast enough for this type of discrete optimization algorithm to succeed.
}
%\vspace{4mm}
\label{tab:speedup}
\end{table}

Table \ref{tab:speedup} shows the overall speedup of CORELS compared to a na\"ive implementation
and demonstrates the infeasibility of running our algorithm 30 years ago.
Consider an execution of CORELS for the ProPublica data set, with ${M = 122}$ antecedents,
that evaluates prefixes up to length~5 in order to certify optimality (Table~\ref{tab:ablation}).
A brute force implementation that na\"ively considers all prefixes of up to length~5
would evaluate ${2.5 \times 10^{10}}$ prefixes.
As shown in Figure~\ref{fig:recidivism-all-folds}, the optimal rule list has prefix length~3,
thus the brute force algorithm would identify the optimal rule list.
However, for this approach to certify optimality, it would have to consider far longer prefixes.
Without our equivalent points bound, but with all of our other optimizations,
we evaluate prefixes up to at least length~10
(see Table~\ref{tab:ablation} and Figure~\ref{fig:queue})---thus a brute force algorithm
would have to evaluate prefixes of length~10 or longer.
Na\"ively evaluating all prefixes up to length~10 would require looking at ${5.0 \times 10^{20}}$ different prefixes.

However, CORELS examines only 28 million prefixes in total---a reduction of~893$\times$ compared to examining
all prefixes up to length~5 and a reduction of ${1.8 \times 10^{13}}$ for the case of length~10.
On a laptop, we require about 1.3~$\mu$s to evaluate a single prefix (given by dividing the number of lower bound evaluations by the total time
in Table~\ref{tab:ablation}).
Our runtime is only about 36 seconds, but the na\"ive solutions of examining all prefixes up to
lengths~5 and~10 would take 9 hours and 21 million years, respectively.
It is clear that brute force would not scale to larger problems.

We compare our current computing circumstances to those of 1984, the year when CART was published.
Moore's law holds that computing power doubled every 18 months from 1984 to 2006.
This is a period of 264 months, which means computing power has gone up by at least a factor of 32,000 since 1984.
Thus, even with our algorithmic and data structural improvements,
CORELS would have required about 13.5 days in 1984---an unreasonable amount of time.
Our advances are meaningful only because we can run them on a modern system.
Combining our algorithmic improvements with the increase in modern processor speeds,
our algorithm runs more than $10^{13}$ times faster than a na\"ive implementation would have in 1984.
This helps explain why neither our algorithm, nor other branch-and-bound variants, had been developed before now.

\begin{arxiv}

\section{Summary and Future Work on Bounds}
\label{sec:practical}

Here, we highlight our most significant bounds, as well as directions for future work
based on bounds that we have yet to leverage in practice.

In empirical studies, we found our equivalent support (\S\ref{sec:equivalent}, Theorem~\ref{thm:equivalent})
and equivalent points (\S\ref{sec:identical}, Theorem~\ref{thm:identical}) bounds
to yield the most significant improvements in algorithm performance.
In fact, they sometimes proved critical for finding solutions and proving optimality,
even on small problems.

Accordingly, we would hope that our similar support bound (\S\ref{sec:similar}, Theorem~\ref{thm:similar}) could be useful;
understanding how to efficiently exploit this result in practice
represents an important direction for future work. In particular, this type of bound may lead to
principled approximate variants of our~approach.

We presented several sets of bounds in which at least one bound was strictly tighter than the other(s).
For example, the lower bound on accurate antecedent support (Theorem~\ref{thm:min-capture-correct})
is strictly tighter than the lower bound on accurate support (Theorem~\ref{thm:min-capture}).
It might seem that we should only use this tighter bound, but in practice, we can use
both---the looser bound can be checked before completing the calculation required to
check the tighter bound.
Similarly, the equivalent support bound (Theorem~\ref{thm:equivalent})
is more general than the special case of the permutation bound (Corollary~\ref{thm:permutation}).
We have implemented data structures, which we call symmetry-aware maps,
to support both of these bounds, but have not yet identified an efficient approach
for supporting the more general equivalent points bound.
A good solution may be related to the challenge of designing an efficient data structure
to support the similar support bound.

We also presented results on antecedent rejection
that unify our understanding of our lower~(\S\ref{sec:lb-support})
and upper bounds (\S\ref{sec:ub-support}) on antecedent support.
In a preliminary implementation (not described here), we experimented with special data structures
to support the direct use of our observation that antecedent rejection propagates
(\S\ref{sec:reject}, Theorem~\ref{thm:reject}).
We leave the design of efficient data structures for this task as future work.

During execution, we find it useful to calculate an upper bound on the
size of the remaining search space---\eg via Theorem~\ref{thm:remaining-eval-fine},
or the looser Proposition~\ref{prop:remaining-eval-coarse}, which incurs less
computational overhead---since these provide meaningful information about algorithm
progress and allow us to estimate the remaining execution time.
As we illustrated in Section~\ref{sec:ablation},
these calculations also help us qualify the impact of different algorithmic bounds,
\eg by comparing executions that keep or remove~bounds.

When our algorithm terminates, it outputs an optimal solution of the training optimization problem, with a certificate of optimality.
On a practical note, our approach can also provide useful results even for incomplete executions.
As shown earlier, we have empirically observed that our algorithm often identifies the optimal rule list
very quickly, compared to the total time required to prove optimality, \eg in seconds,
versus minutes, respectively.
Furthermore, our objective's lower bounds allow us to place an upper bound on the size of the remaining search space,
%and other algorithm states, such information about explored versus unexplored prefixes,
and provides guarantees on the quality of a solution output by an incomplete execution.

The order in which we evaluate prefixes can impact the rate at which we
prune the search space, and thus the total runtime. We think that it is possible to design search policies that significantly improve performance.

\end{arxiv}

\section{Conclusion and More Possible Directions for Future Work}

\begin{kdd}
CORELS is an efficient and accurate algorithm for constructing provably optimal rule lists.
Optimality is particularly important in domains where model interpretability
has social consequences, \eg recidivism prediction.
While achieving optimality on such discrete optimization problems is
computationally hard in general, we aggressively prune our problem's search space
via a suite of bounds.
This makes realistically sized problems tractable.
CORELS is amenable to parallelization, which should allow it to scale to
even larger problems than those we presented here.
\end{kdd}

\begin{arxiv}

Finally, we would like to clarify some limitations of CORELS.
As far as we can tell, CORELS is the current best algorithm for solving a
specialized optimal decision tree problem.
While our approach scales well to large numbers of observations,
it could have difficulty proving optimality
for problems with many possibly relevant features that are highly correlated,
when large regions of the search space might be challenging to exclude.

CORELS is not designed for raw image processing or other problems where the
features themselves are not interpretable.
It could instead be used as a final classifier for image processing problems
where the features were created beforehand; for instance, one could create classifiers
for each part of an image, and use CORELS to create a final combined classifier. The notions of interpretability used in image classification tend to be completely different from those for structured data where each feature is separately meaningful \citep[e.g., see][]{LiEtAl18}. For structured data, decision trees, along with scoring systems, tend to be popular forms of transparent models. Scoring systems are sparse linear models with integer coefficients, and they can also be created from data \citep{UstunRu2017KDD,UstunRu2016SLIM}.

In some of our experiments, CORELS produces a DNF formula by coincidence, but it might be possible to create a much simpler version of CORELS that only produces DNF formulae. This could build off previous algorithms for creating an optimal DNF formula \citep{Rijnbeek10,WangEtAl16,WangEtAl2017}.

CORELS does not automatically rank the subgroups in order of the likelihood of a
positive outcome; doing so would require an algorithm such as Falling Rule Lists \citep{WangRu15,ChenRu18}, which forces the estimated probabilities to decrease along the list.
Furthermore, while CORELS does not technically produce estimates of ${\P(Y=1 \given x)}$,
one could form such an estimate by computing the empirical
proportion ${\hat{\P}(Y=1 \given x \textrm{ obeys } p_k)}$ for each antecedent~$p_k$.
CORELS is also not designed to assist with causal inference applications, since
it does not estimate the effect of a treatment via the conditional difference
${\P(Y=1 \given \textrm{treatment} =}$ ${\textrm{True}, x)}$ $-$
${\P(Y=1 \given \textrm{treatment} =}$ ${\textrm{False}, x)}$.
Alternative algorithms that estimate conditional differences with interpretable
rule lists include Causal Falling Rule Lists \citep{WangRu15CFRL},
Cost-Effective Interpretable Treatment Regimes (CITR) \citep{LakkarajuRu17},
and an approach by \citet{ZhangEtAl15} for constructing
interpretable and parsimonious treatment regimes. Alternatively, one could use a complex machine learning model to predict outcomes for the treatment group and a separate complex model for the control group that would allow counterfactuals to be estimated for each observation; from there, CORELS could be applied to produce a transparent model for personalized treatment effects. A similar approach to this was taken by \citet{GohRu18causal}, who use CORELS to understand a black box causal model.

CORELS could be adapted to handle cost-sensitive learning or weighted regularization.
This would require creating more general versions of our theorems,
which would be an extension of this work.
%\citep[Or, one could use the more general but computationally heavy mixed-integer programming algorithm of][which allows easy customization.]{ErtekinRu17}
%

While CORELS does not directly handle continuous variables, we have found that it is not difficult in practice to construct a rule set that is sufficient for creating a useful model. It may be possible to use techniques such as Fast Boxes \citep{GohRu14} to discover useful and interpretable rules for continuous data that can be used within CORELS.

An interesting direction for future research would be to create a hybrid interpretable/ black box model in the style of \citet{Wang2018}, where the rule list would eliminate large parts of the space away from the decision boundary, and the observations that remain are evaluated by a black box model rather than a default rule.

Lastly, CORELS does not create generic single-variable-split decision trees.
CORELS optimizes over rule lists, which are one-sided decision trees;
in our setting, the leaves of these `trees' are conjunctions.
It may be possible to generalize ideas from our approach to handle generic
decision trees, which could be an interesting project for future work. There are more symmetries to handle in that case, since there would be many equivalent decision trees, leading to challenges in developing symmetry-aware data structures.

\end{arxiv}

% Acknowledgements should go at the end, before appendices and references
\acks{E.A. conducted most of this work while supported by the Miller Institute for Basic Research
in Science, University of California, Berkeley, and hosted by Prof. M.I. Jordan at RISELab.
C.D.R. is supported in part by MIT-Lincoln Labs and the National Science Foundation under IIS-1053407.
E.A. would like to thank E.~Jonas, E.~Kohler, and S.~Tu for early implementation
guidance, A.~D'Amour for pointing out the work by~\citet{Goel16}, V.~Kanade, S.~McCurdy,
J.~Schleier-Smith and E.~Thewalt for helpful conversations, and members of RISELab,
SAIL, and the UC Berkeley Database Group for their support and feedback.
We thank H.~Yang and B.~Letham for sharing advice and code for processing data
and mining rules, B.~Coker for his critical advice on using the ProPublica COMPAS data set,
as well as V.~Kaxiras and A.~Saligrama for their recent contributions to our implementation
and for creating the CORELS website. We are very grateful to our editor and anonymous reviewers.
}

\vskip 0.2in

\renewcommand{\theHsection}{A\arabic{section}}
\appendix

\section{Excessive Antecedent Support}
\label{appendix:ub-supp}

\begin{theorem}[Upper bound on antecedent support]
\label{thm:ub-support}
Let ${\OptimalRL = (\Prefix, \Labels, \Default, K)}$
be any optimal rule list with objective~$\OptimalObj$, \ie
${\OptimalRL \in \argmin_\RL \Obj(\RL, \x, \y)}$,
and let ${\Prefix = (p_1, \dots, p_{k-1},}$
${p_k, \dots, p_K)}$ be its prefix.
For each ${k \le K}$, antecedent~$p_k$ in~$\Prefix$
has support less than or equal to
the fraction of all data not captured by preceding antecedents,
by an amount greater than the regularization parameter~$\Reg$:
\begin{align}
\Supp(p_k, \x \given \Prefix) \le 1 - \Supp(\Prefix^{k-1}, \x) - \Reg,
\label{eq:ub-support}
\end{align}
where ${\Prefix^{k-1} = (p_1, \dots, p_{k-1})}$.
For the last antecedent, \ie when ${p_k = p_K}$, equality implies
that there also exists a shorter optimal rule list
${\RL' = (\Prefix^{K-1}, \Labels', \Default', K - 1) \in}$ ${\argmin_\RL \Obj(\RL, \x, \y)}$.
%with prefix~$\Prefix^{K-1}$.
\end{theorem}

\begin{proof}
First, we focus on the last antecedent~$p_{K+1}$ in a rule list~$\RL'$.
Let ${\RL = (\Prefix, \Labels, \Default, K)}$
be a rule list with prefix ${\Prefix = (p_1, \dots, p_K)}$
and objective ${\Obj(\RL, \x, \y) \ge \OptimalObj}$, where
${\OptimalObj \equiv}$ ${\min_{\RLB} \Obj(\RLB, \x, \y)}$
is the optimal objective.
Let ${\RL' = (\Prefix', \Labels', \Default', K + 1)}$
be a rule list whose prefix ${\Prefix' = (p_1, \dots, p_K, p_{K+1})}$
starts with~$\Prefix$ and ends with a new antecedent~$p_{K+1}$.
Suppose~$p_{K+1}$ in the context of~$\Prefix'$ captures nearly all
data not captured by~$\Prefix$, except for a fraction~$\epsilon$
upper bounded by the regularization parameter~$\Reg$:
\begin{align*}
1 - \Supp(\Prefix, \x) - \Supp(p_{K+1}, \x \given \Prefix') \equiv \epsilon \le \Reg.
\end{align*}
Since~$\Prefix'$ starts with~$\Prefix$,
its prefix misclassification error is at least as great;
the only discrepancy between the misclassification errors
of~$\RL$ and~$\RL'$ can come from the difference between the support of
the set of data not captured by~$\Prefix$ and the support of~$p_{K+1}$:
\begin{align*}
| \Loss(\RL', \x, \y) - \Loss(\RL, \x, \y) | \le
1 - \Supp(\Prefix, \x) - \Supp(p_{K+1}, \x \given \Prefix') = \epsilon.
\end{align*}
The best outcome for~$\RL'$ would occur if its misclassification
error were smaller than that of~$\RL$ by~$\epsilon$, therefore
\begin{align*}
\Obj(\RL', \x, \y) &= \Loss(\RL', \x, \y) + \Reg (K+1) \\
&\ge \Loss(\RL, \x, \y) - \epsilon + \Reg(K+1)
= \Obj(\RL, \x, \y) - \epsilon + \Reg \ge \Obj(\RL, \x, \y) \ge \OptimalObj.
\end{align*}
$\RL'$ is an optimal rule list,
\ie ${\RL' \in \argmin_{\RLB} \Obj(\RLB, \x, y)}$,
if and only if ${\Obj(\RL', \x, \y) = \Obj(\RL, \x, \y) =}$ ${\OptimalObj}$,
which requires ${\epsilon = \Reg}$.
Otherwise, ${\epsilon < \Reg}$, in which case
\begin{align*}
\Obj(\RL', \x, \y) \ge \Obj(\RL, \x, \y) - \epsilon + \Reg
> \Obj(\RL, \x, \y) \ge \OptimalObj,
\end{align*}
therefore $\RL'$ is not optimal, \ie  ${\RL' \notin \argmin_{\RLB} \Obj(\RLB, \x, \y)}$.
This demonstrates the desired result for ${k = K}$.

In the remainder, we prove the bound in~\eqref{eq:ub-support} by contradiction,
in the context of a rule list~$\RL''$.
Let~$\RL$ and~$\RL'$ retain their definitions from above,
thus as before, that the data not captured by~$\Prefix'$
has normalized support~${\epsilon \le \Reg}$, \ie
\begin{align*}
1 - \Supp(\Prefix', \x) = 1 - \Supp(\Prefix, \x) - \Supp(p_{K+1}, \x \given \Prefix') = \epsilon \le \Reg.
\end{align*}
Thus for any rule list~$\RL''$ whose prefix
$\Prefix'' = (p_1, \dots, p_{K+1}, \dots, p_{K'})$ starts
with~$\Prefix'$ and ends with one or more additional rules,
each additional rule~$p_k$ has support
${\Supp(p_k, \x \given \Prefix'') \le}$ ${\epsilon \le \Reg}$,
for all~${k > K+1}$.
By Theorem~\ref{thm:min-capture},
all of the additional rules have insufficient support,
therefore~$\Prefix''$ cannot be optimal,
\ie ${\RL'' \notin \argmin_{\RLB} \Obj(\RLB, \x, \y)}$.
\end{proof}

Similar to Theorem~\ref{thm:min-capture}, our lower bound on
antecedent support, we can apply Theorem~\ref{thm:ub-support}
in the contexts of both constructing rule lists and
rule mining~(\S\ref{sec:setup}).
Theorem~\ref{thm:ub-support} implies that if we only seek a single
optimal rule list, then during branch-and-bound execution,
we can prune a prefix if we ever add an antecedent with support
too similar to the support of the set of data not captured by the
preceding antecedents.
One way to view this result is that if
${\RL = (\Prefix, \Labels, \Default, K)}$
and ${\RL' = (\Prefix', \Labels', \Default', K + 1)}$
are rule lists such that~$\Prefix'$ starts with~$\Prefix$
and ends with an antecedent that captures all or nearly all
data not captured by~$\Prefix$, then the new rule in~$\RL'$
behaves similar to the default rule of~$\RL$.
As a result, the misclassification error of~$\RL'$ must be
similar to that of~$\RL$, and any reduction may not be
sufficient to offset the penalty for longer prefixes.

\begin{proposition}[Excessive antecedent support propagates]
\label{prop:ub-support}
Define~$\StartContains(\Prefix)$ as in~\eqref{eq:start-contains},
and let ${\Prefix = (p_1, \dots, p_{K})}$ be a prefix,
such that its last antecedent~$p_{K}$ has excessive support,
\ie the opposite of the bound in~\eqref{eq:ub-support}:
\begin{align*}
\Supp(p_K, \x \given \Prefix) > 1 - \Supp(\Prefix^{K-1}, \x) - \Reg,
\end{align*}
where ${\Prefix^{K-1} = (p_1, \dots, p_{K-1})}$.
Let ${\RLB = (\PrefixB, \LabelsB, \DefaultB, \kappa)}$
be any rule list with prefix
${\PrefixB =}$ ${(P_1, \dots, P_{\kappa})}$
such that~$\PrefixB$ starts with ${\PrefixB^{K'-1} =}$
${(P_1, \dots, P_{K'-1}) \in \StartContains(\Prefix^{K-1})}$
and~${P_{K'} = p_{K}}$.
It follows that~$P_{K'}$ has excessive support in prefix~$\PrefixB$,
and furthermore, ${\RLB \notin \argmin_{\RL} \Obj(\RL, \x, \y)}$.
\end{proposition}

\begin{proof}
Since ${\PrefixB^{K'} = (P_1, \dots, P_{K'})}$
contains all the antecedents in~$\Prefix$, we have that
\begin{align*}
\Supp(\PrefixB^{K'}, \x) \ge \Supp(\Prefix, \x).
\end{align*}
Expanding these two terms gives
\begin{align*}
\Supp(\PrefixB^{K'}, \x)
&= \Supp(\PrefixB^{K'-1}, \x) + \Supp(P_{K'}, \x \given \PrefixB) \\
&\ge \Supp(\Prefix, \x)
= \Supp(\Prefix^{K-1}, \x) + \Supp(p_K, \x \given \Prefix)
> 1 - \Reg.
\end{align*}
Rearranging gives
\begin{align*}
\Supp(P_{K'}, \x \given \PrefixB)
> 1 - \Supp(\PrefixB^{K'-1}, \x) - \Reg,
\end{align*}
thus~$P_{K'}$ has excessive support in~$\PrefixB$.
By Theorem~\ref{thm:ub-support},
${\RLB \notin \argmin_{\RL} \Obj(\RL, \x, \y)}$.
\end{proof}

\section{Proof of Theorem~\ref{thm:equivalent} (Equivalent Support Bound)}
\label{appendix:equiv-supp}

We begin by defining four related rule lists.
First, let ${\RL = (\Prefix, \Labels, \Default, K)}$
be a rule list with prefix ${\Prefix = (p_1, \dots, p_K)}$
and labels ${\Labels = (q_1, \dots, q_K)}$.
Second, let ${\RLB = (\PrefixB, \LabelsB, \DefaultB, \kappa)}$
be a rule list with prefix ${\PrefixB = (P_1, \dots, P_\kappa)}$
that captures the same data as~$\Prefix$,
and labels ${\LabelsB = (Q_1, \dots, Q_\kappa)}$.
Third, let ${\RL' = (\Prefix', \Labels', \Default', K') \in}$
${\StartsWith(\Prefix)}$ be any rule list
whose prefix starts with~$\Prefix$, such that~${K' \ge K}$.
Denote the prefix and labels of~$\RL'$ by
${\Prefix' = (p_1, \dots, p_K,}$ ${p_{K+1}, \dots, p_{K'})}$
and ${\Labels = (q_1, \dots, q_{K'})}$, respectively.
Finally, define ${\RLB' = (\PrefixB', \LabelsB', \DefaultB', \kappa') \in}$
${\StartsWith(\PrefixB)}$ to be the `analogous' rule list, \ie whose prefix
${\PrefixB' = (P_1, \dots, P_\kappa, P_{\kappa+1}, \dots, P_{\kappa'}) =}$
${(P_1, \dots, P_\kappa, p_{K+1}, \dots, p_{K'})}$
starts with~$\PrefixB$ and ends with the same ${K'-K}$
antecedents as~$\Prefix'$.
Let ${\LabelsB' = (Q_1, \dots, Q_{\kappa'})}$
denote the labels of~$\RLB'$.

Next, we claim that the difference in the objectives
of rule lists~$\RL'$ and~$\RL$ is the same as the difference
in the objectives of rule lists~$\RLB'$ and~$\RLB$.
Let us expand the first difference as
\begin{align}
&\Obj(\RL', \x, \y) - \Obj(\RL, \x, \y)
  = \Loss(\RL', \x, \y) + \Reg K' - \Loss(\RL, \x, \y) - \Reg K \nn \\
&= \Loss_p(\Prefix', \Labels', \x, \y) + \Loss_0(\Prefix', \Default', \x, \y)
  - \Loss_p(\Prefix, \Labels, \x, \y) - \Loss_0(\Prefix, \Default, \x, \y)
  + \Reg (K' - K). \nn
\end{align}
Similarly, let us expand the second difference as
\begin{align}
&\Obj(\RLB', \x, \y) - \Obj(\RLB, \x, \y)
  = \Loss(\RLB', \x, \y) + \Reg \kappa' - \Loss(\RLB, \x, \y) - \Reg \kappa \nn \\
&= \Loss_p(\PrefixB', \LabelsB', \x, \y) + \Loss_0(\PrefixB', \DefaultB', \x, \y)
  - \Loss_p(\PrefixB, \LabelsB, \x, \y) - \Loss_0(\PrefixB, \DefaultB, \x, \y)
  + \Reg (K' - K), \nn
\end{align}
where we have used the fact that ${\kappa' - \kappa = K' - K}$.

The prefixes~$\Prefix$ and~$\PrefixB$ capture the same data.
Equivalently, the set of data that is not captured by~$\Prefix$
is the same as the set of data that is not captured by~$\PrefixB$, \ie
\begin{align}
\{x_n : \neg\, \Cap(x_n, \Prefix)\} = \{x_n : \neg\, \Cap(x_n, \PrefixB)\}. \nn
\end{align}
Thus, the corresponding rule lists~$\RL$ and~$\RLB$
share the same default rule, \ie ${\Default = \DefaultB}$,
yielding the same default rule misclassification error:
\begin{align}
\Loss_0(\Prefix, \Default, \x, \y) = \Loss_0(\PrefixB, \DefaultB, \x, \y). \nn
\end{align}
Similarly, prefixes~$\Prefix'$ and~$\PrefixB'$ capture
the same data, and thus rule lists~$\RL'$ and~$\RLB'$
have the same default rule misclassification error:
\begin{align}
\Loss_0(\Prefix, \Default, \x, \y) = \Loss_0(\PrefixB, \DefaultB, \x, \y). \nn
\end{align}

At this point, to demonstrate our claim relating the objectives
of~$\RL$, $\RL'$, $\RLB$, and~$\RLB'$, what remains is to
show that the difference in the misclassification errors
of prefixes~$\Prefix'$ and~$\Prefix$ is the same as that
between~$\PrefixB'$ and~$\PrefixB$.
We can expand the first difference as
\begin{align}
\Loss_p(\Prefix', \Labels', \x, \y) - \Loss_p(\Prefix, \Labels, \x, \y)
%&= \frac{1}{N} \sum_{n=1}^N \sum_{k=1}^{K'}
%  \one [ \Cap(x_n, p_k \given \Prefix') \wedge (q_k \neq y_n) ]
%  - \frac{1}{N} \sum_{n=1}^N \sum_{k=1}^K
%  \one [ \Cap(x_n, p_k \given \Prefix) \wedge (q_k \neq y_n) ] \\
&= \frac{1}{N} \sum_{n=1}^N \sum_{k=K+1}^{K'}
  \Cap(x_n, p_k \given \Prefix') \wedge \one [ q_k \neq y_n ], \nn
\end{align}
where we have used the fact that since~$\Prefix'$
starts with~$\Prefix$, the first~$K$ rules in~$\Prefix'$
make the same mistakes as those in~$\Prefix$.
Similarly, we can expand the second difference as
\begin{align}
\Loss_p(\PrefixB', \LabelsB', \x, \y) - \Loss_p(\PrefixB, \LabelsB, \x, \y)
&= \frac{1}{N} \sum_{n=1}^N \sum_{k=\kappa+1}^{\kappa'}
  \Cap(x_n, P_k \given \PrefixB') \wedge \one [ Q_k \neq y_n ] \nn \\
&= \frac{1}{N} \sum_{n=1}^N \sum_{k=K+1}^{K'}
  \Cap(x_n, p_k \given \PrefixB') \wedge \one [ Q_k \neq y_n ] \nn \\
&= \frac{1}{N} \sum_{n=1}^N \sum_{k=K+1}^{K'}
  \Cap(x_n, p_k \given \Prefix') \wedge \one [ q_k \neq y_n ] \label{eq:third} \\
&= \Loss_p(\Prefix', \Labels', \x, \y) - \Loss_p(\Prefix, \Labels, \x, \y) \nn.
\end{align}
To justify the equality in~\eqref{eq:third}, we observe first that
prefixes~$\PrefixB'$ and~$\Prefix'$ start with~$\kappa$ and~$K$
antecedents, respectively, that capture the same data.
Second, prefixes~$\PrefixB'$ and~$\Prefix'$ end with exactly
the same ordered list of~${K' - K}$ antecedents,
therefore for any~${k = 1, \dots, K' - K}$,
antecedent ${P_{\kappa + k} = p_{K + k}}$ in~$\PrefixB'$
captures the same data as~$p_{K + k}$ captures in~$\Prefix'$.
It follows that the corresponding labels are all equivalent, \ie
${Q_{\kappa + k} = q_{K + k}}$, for all~${k = 1, \dots, K' - K}$,
and consequently, the prefix misclassification error associated
with the last~${K' - K}$ antecedents of~$\Prefix'$ is the same
as that of~$\PrefixB'$.
We have therefore shown that the difference between the objectives
of~$\RL'$ and~$\RL$ is the same as that between~$\RLB'$ and~$\RLB$, \ie
\begin{align}
\Obj(\RL', \x, \y) - \Obj(\RL, \x, \y)
= \Obj(\RLB', \x, \y) - \Obj(\RLB, \x, \y).
\label{eq:equiv-analogous}
\end{align}

Next, suppose that the objective lower bounds of~$\RL$ and~$\RLB$
obey ${b(\Prefix, \x, \y) \le b(\PrefixB, \x, \y)}$, therefore
\begin{align}
\Obj(\RL, \x, \y)
&= \Loss_p(\Prefix, \Labels, \x, \y) + \Loss_0(\Prefix, \Default, \x, \y) + \Reg K \nn \\
&= b(\Prefix, \x, \y) + \Loss_0(\Prefix, \Default, \x, \y) \nn \\
&\le b(\PrefixB, \x, \y) + \Loss_0(\Prefix, \Default, \x, \y)
= b(\PrefixB, \x, \y) + \Loss_0(\PrefixB, \DefaultB, \x, \y)
= \Obj(\RLB, \x, \y).
\label{eq:equiv-ineq}
\end{align}
Now let~$\RL^*$ be an optimal rule list with prefix
constrained to start with~$\Prefix$,
\begin{align}
\RL^* \in \argmin_{\RL^\dagger \in \StartsWith(\Prefix)} \Obj(\RL^\dagger, \x, \y), \nn
\end{align}
and let~$K^*$ be the length of~$\RL^*$.
Let~$\RLB^*$ be the analogous $\kappa^*$-rule list whose prefix starts
with~$\PrefixB$ and ends with the same~${K^* - K}$ antecedents as~$\RL^*$,
where~${\kappa^* = \kappa + K^* - K}$.
By~\eqref{eq:equiv-analogous},
\begin{align}
\Obj(\RL^*, \x, \y) - \Obj(\RL, \x, \y)
= \Obj(\RLB^*, \x, \y) - \Obj(\RLB, \x, \y).
\label{eq:equiv-diff}
\end{align}
Furthermore, we claim that~$\RLB^*$ is an optimal rule list
with prefix constrained to start with~$\PrefixB$,
\begin{align}
\RLB^* \in \argmin_{\RLB^\dagger \in \StartsWith(\PrefixB)} \Obj(\RLB^\dagger, \x, \y).
\label{eq:equiv-analogous-optimal}
\end{align}

To demonstrate~\eqref{eq:equiv-analogous-optimal},
we consider two separate scenarios.
In the first scenario, prefixes~$\Prefix$ and~$\PrefixB$
are composed of the same antecedents,
\ie the two prefixes are equivalent up to a permutation of
their antecedents, and as a consequence,
${\kappa = K}$ and~${\kappa^* = K^*}$.
Here, every rule list~${\RL'' \in \StartsWith(\Prefix)}$
that starts with~$\Prefix$
has an analogue~${\RLB'' \in \StartsWith(\PrefixB)}$
that starts with~$\PrefixB$,
such that~$\RL''$ and~$\RLB''$ obey~\eqref{eq:equiv-analogous},
and vice versa, and thus~\eqref{eq:equiv-analogous-optimal}
is a direct consequence of~\eqref{eq:equiv-diff}.

In the second scenario, prefixes~$\Prefix$ and~$\PrefixB$
are not composed of the same antecedents.
Define~${\phi = \{p_k : (p_k \in \Prefix) \wedge (p_k \notin \PrefixB)\}}$
to be the set of antecedents in~$\Prefix$ that are not in~$\PrefixB$,
and define~${\Phi = \{P_k : (P_k \in \PrefixB) \wedge (P_k \notin \Prefix)\}}$
to be the set of antecedents in~$\PrefixB$ that are not in~$\Prefix$;
either~${\phi \neq \emptyset}$, or~${\Phi \neq \emptyset}$, or both.

Suppose~${\phi \neq \emptyset}$, and let~${p \in \phi}$
be an antecedent in~$\phi$.
It follows that there exists a subset of rule lists
in~$\StartsWith(\PrefixB)$ that do not have analogues
in~$\StartsWith(\Prefix)$.
Let~${\RLB'' \in \StartsWith(\PrefixB)}$ be such a rule list,
such that its prefix ${\PrefixB'' = (P_1, \dots, P_\kappa, \dots, p, \dots)}$
starts with~$\PrefixB$ and contains~$p$ among its remaining antecedents.
Since~$p$ captures a subset of the data that~$\Prefix$ captures,
and~$\PrefixB$ captures the same data as~$\Prefix$,
it follows that~$p$ does not capture any data in~$\PrefixB''$, \ie
\begin{align}
\frac{1}{N} \sum_{n=1}^N \Cap(x_n, p \given \PrefixB'') = 0 \le \Reg. \nn
\end{align}
By Theorem~\ref{thm:min-capture}, antecedent~$p$ has insufficient
support in~$\RLB''$, and thus~$\RLB''$ cannot be optimal, \ie
${\RLB'' \notin}$ ${\argmin_{\RLB^\dagger \in \StartsWith(\PrefixB)} \Obj(\RLB^\dagger, \x, \y)}$.
By a similar argument, if~${\Phi \neq \emptyset}$
and~${P \in \Phi}$, and~${\RL'' \in \StartsWith(\Prefix)}$
is any rule list whose prefix starts with~$\Prefix$
and contains antecedent~$P$, then~$\RL''$ cannot be optimal, \ie
${\RL'' \notin \argmin_{\RL^\dagger \in \StartsWith(\Prefix)} \Obj(\RL^\dagger, \x, \y)}$.

To finish justifying claim~\eqref{eq:equiv-analogous-optimal}
for the second scenario, first define
\begin{align}
\tau(\Prefix, \Phi) \equiv
  \{\RL'' = (\Prefix'', \Labels'', \Default'', K'') :
    \RL'' \in \StartsWith(\Prefix) \textnormal{ and }
    p_k \notin \Phi, \forall p_k \in \Prefix''\} \subset \StartsWith(\Prefix) \nn
\end{align}
to be the set of all rule lists whose prefixes start with~$\Prefix$
and do not contain any antecedents in~$\Phi$.
Now, recognize that the optimal prefixes in~$\tau(\Prefix, \Phi)$
and~$\StartsWith(\Prefix)$ are the same, \ie
\begin{align}
\argmin_{\RL^\dagger \in \tau(\Prefix, \Phi)} \Obj(\RL^\dagger, \x, \y)
= \argmin_{\RL^\dagger \in \StartsWith(\Prefix)} \Obj(\RL^\dagger, \x, \y), \nn
\end{align}
and similarly, the optimal prefixes in~$\tau(\PrefixB, \phi)$
and~$\StartsWith(\PrefixB)$ are the same, \ie
\begin{align}
\argmin_{\RLB^\dagger \in \tau(\PrefixB, \phi)} \Obj(\RLB^\dagger, \x, \y)
= \argmin_{\RLB^\dagger \in \StartsWith(\PrefixB)} \Obj(\RLB^\dagger, \x, \y). \nn
\end{align}
Since we have shown that every~${\RL'' \in \tau(\Prefix, \Phi)}$
has a direct analogue~${\RLB'' \in \tau(\PrefixB, \phi)}$,
such that~$\RL''$ and~$\RLB''$ obey~\eqref{eq:equiv-analogous},
and vice versa, we again have~\eqref{eq:equiv-analogous-optimal}
as a consequence of~\eqref{eq:equiv-diff}.

We can now finally combine~\eqref{eq:equiv-ineq}
and~\eqref{eq:equiv-analogous-optimal} to obtain the desired inequality in~\eqref{eq:equivalent}:
\begin{align}
\min_{\RL' \in \StartsWith(\Prefix)} \Obj(\RL', \x, \y)
= \Obj(\RL^*, \x, \y) \le \Obj(\RLB^*, \x, \y)
= \min_{\RLB' \in \StartsWith(\PrefixB)} \Obj(\RLB', \x, \y). \nn
\end{align}

%\clearpage
\section{Proof of Theorem~\ref{thm:similar} (Similar Support Bound)}
\label{appendix:similar-supp}

We begin by defining four related rule lists.
First, let ${\RL = (\Prefix, \Labels, \Default, K)}$
be a rule list with prefix ${\Prefix = (p_1, \dots, p_K)}$
and labels ${\Labels = (q_1, \dots, q_K)}$.
Second, let ${\RLB = (\PrefixB, \LabelsB, \DefaultB, \kappa)}$
be a rule list with prefix ${\PrefixB = (P_1, \dots, P_\kappa)}$
and labels ${\LabelsB = (Q_1, \dots, Q_\kappa)}$.
Define~$\omega$ as in~\eqref{eq:omega}
and~$\Omega$ as in~\eqref{eq:big-omega},
and require that~${\omega, \Omega \le \Reg}$.
Third, let ${\RL' = (\Prefix', \Labels', \Default', K') \in}$
${\StartsWith(\Prefix)}$ be any rule list
whose prefix starts with~$\Prefix$, such that~${K' \ge K}$.
Denote the prefix and labels of~$\RL'$ by
${\Prefix' = (p_1, \dots, p_K, p_{K+1}, \dots, p_{K'})}$
and ${\Labels = (q_1, \dots, q_{K'})}$,
respectively.
Finally, define
${\RLB' = (\PrefixB', \LabelsB', \DefaultB', \kappa') \in \StartsWith(\PrefixB)}$
to be the `analogous' rule list, \ie whose prefix
${\PrefixB' =}$ ${(P_1, \dots, P_\kappa, P_{\kappa+1}, \dots, P_{\kappa'})
= (P_1, \dots, P_\kappa, p_{K+1}, \dots, p_{K'})}$
starts with~$\PrefixB$ and ends with the same ${K'-K}$
antecedents as~$\Prefix'$.
Let ${\LabelsB' = (Q_1, \dots, Q_{\kappa'})}$
denote the labels of~$\RLB'$.

%Suppose that the lower bounds of~$\Prefix$ and~$\PrefixB$
%obey ${b(\Prefix, \x, \y) < b(\PrefixB, \x, \y)}$.
%
The smallest possible objective for~$\RLB'$, in relation
to the objective of~$\RL'$, reflects both the difference
between the objective lower bounds of~$\RLB$ and~$\RL$
and the largest possible discrepancy between the
objectives of~$\RL'$ and~$\RLB'$.
The latter would occur if~$\RL'$ misclassified all the data
corresponding to both~$\omega$ and~$\Omega$ while~$\RLB'$
correctly classified this same data, thus
\begin{align}
\Obj(\RLB', \x, \y) \ge \Obj(\RL', \x, \y)
  + b(\PrefixB, \x, \y) - b(\Prefix, \x, \y) - \omega - \Omega.
\label{eq:similar-analogous}
\end{align}
Now let~$\RLB^*$ be an optimal rule list with prefix
constrained to start with~$\PrefixB$,
\begin{align*}
\RLB^* \in \argmin_{\RLB^\dagger \in \StartsWith(\PrefixB)} \Obj(\RLB^\dagger, \x, \y),
\end{align*}
and let~$\kappa^*$ be the length of~$\RLB^*$.
Also let~$\RL^*$ be the analogous $K^*$-rule list whose prefix
starts with~$\Prefix$ and ends with the same~${\kappa^* - \kappa}$
antecedents as~$\RLB^*$, where~${K^* = K + \kappa^* - \kappa}$.
By~\eqref{eq:similar-analogous}, we obtain the desired inequality in~\eqref{eq:similar}:
\begin{align*}
\min_{\RLB^\dagger \in \StartsWith(\PrefixB)} \Obj(\RLB^\dagger, \x, \y)
&= \Obj(\RLB^*, \x, \y) \\
&\ge \Obj(\RL^*, \x, \y)
  + b(\PrefixB, \x, \y) - b(\Prefix, \x, \y) - \omega - \Omega \\
&\ge \min_{\RL^\dagger \in \StartsWith(\Prefix)} \Obj(\RL^\dagger, \x, \y)
  + b(\PrefixB, \x, \y) - b(\Prefix, \x, \y) - \omega - \Omega.
\end{align*}

\section{Proof of Theorem~\ref{thm:identical} (Equivalent Points Bound)}
\label{appendix:equiv-pts}

We derive a lower bound on the default rule
misclassification error~$\Loss_0(\Prefix, \Default, \x, \y)$,
analogous to the lower bound~\eqref{eq:lb-equiv-pts} on the misclassification
error~$\Loss(\RL, \x, \y)$ in the proof of Proposition~\ref{prop:identical}.
As before, we sum over all sets of equivalent points, and then for each such set,
we count differences between class labels and the minority class label of the set,
instead of counting mistakes made by the default rule:
\begin{align}
\Loss_0(\Prefix, \Default, \x, \y)
&= \frac{1}{N} \sum_{n=1}^N \neg\, \Cap(x_n, \Prefix) \wedge \one [ q_0 \neq y_n ] \nn \\
&= \frac{1}{N} \sum_{u=1}^U \sum_{n=1}^N \neg\, \Cap(x_n, \Prefix) \wedge
  \one [ q_0 \neq y_n ]\, \one [ x_n \in e_u ] \nn \\
&\ge \frac{1}{N} \sum_{u=1}^U \sum_{n=1}^N \neg\, \Cap(x_n, \Prefix) \wedge
  \one [ y_n = q_u ]\, \one [ x_n \in e_u ] = b_0(\Prefix, \x, \y),
\label{eq:lb-equiv-pts-uncap}
\end{align}
where the final equality comes from the definition of~$b_0(\Prefix, \x, \y)$ in~\eqref{eq:lb-b0}.
Since we can write the objective~$\Obj(\RL, \x, \y)$
as the sum of the objective lower bound~$b(\Prefix, \x, \y)$ and
default rule misclassification error~$\Loss_0(\Prefix, \Default, \x, \y)$,
applying~\eqref{eq:lb-equiv-pts-uncap} gives a lower bound on~$\Obj(\RL, \x, \y)$:
\begin{align}
\Obj(\RL, \x, \y)
= \Loss_p(\Prefix, \Labels, \x, \y) + \Loss_0(\Prefix, \Default, \x, \y) + \Reg K
&= b(\Prefix, \x, \y) + \Loss_0(\Prefix, \Default, \x, \y) \nn \\
&\ge b(\Prefix, \x, \y) + b_0(\Prefix, \x, \y).
\label{eq:equiv-pts-base}
\end{align}
It follows that for any rule list~${\RL' \in \StartsWith(\RL)}$ whose prefix~$\Prefix'$
starts with~$\Prefix$, we have
\begin{align}
\Obj(\RL', \x, \y) \ge b(\Prefix', \x, \y) + b_0(\Prefix', \x, \y).
\label{eq:equiv-pts-extend}
\end{align}

Finally, we show that the lower bound
on~${\Obj(\RL, \x, \y)}$ in~\eqref{eq:equiv-pts-base} is not greater than
the lower bound on~${\Obj(\RL', \x, \y)}$ in~\eqref{eq:equiv-pts-extend}.
First, let us define
\begin{align}
\Upsilon(\Prefix', K, \x, \y) \equiv \frac{1}{N} \sum_{u=1}^U \sum_{n=1}^N
    \sum_{k=K+1}^{K'} \Cap(x_n, p_k \given \Prefix') \wedge \one [ x_n \in e_u ]\, \one [ y_n = q_u ].
\label{eq:upsilon}
\end{align}
Now, we write a lower bound on~${b(\Prefix', \x, \y)}$ with respect to~${b(\Prefix, \x, \y)}$:
\begin{align}
&b(\Prefix', \x, \y) = \Loss_p(\Prefix', \Labels, \x, \y) + \Reg K'
= \frac{1}{N} \sum_{n=1}^N \sum_{k=1}^{K'} \Cap(x_n, p_k \given \Prefix') \wedge \one [ q_k \neq y_n ] + \Reg K' \nn \\
&= \Loss_p(\Prefix, \Labels, \x, \y) + \Reg K + \frac{1}{N} \sum_{n=1}^N \sum_{k=K}^{K'} \Cap(x_n, p_k \given \Prefix') \wedge \one [ q_k \neq y_n ] + \Reg (K' - K) \nn \\
&= b(\Prefix, \x, \y) + \frac{1}{N} \sum_{n=1}^N \sum_{k=K+1}^{K'} \Cap(x_n, p_k \given \Prefix') \wedge \one [ q_k \neq y_n ] + \Reg (K' - K) \nn \\
&= b(\Prefix, \x, \y) + \frac{1}{N} \sum_{u=1}^U \sum_{n=1}^N \sum_{k=K+1}^{K'} \Cap(x_n, p_k \given \Prefix')
  \wedge \one [ q_k \neq y_n ]\, \one [ x_n \in e_u ] + \Reg (K' - K) \nn \\
&\ge b(\Prefix, \x, \y) + \frac{1}{N} \sum_{u=1}^U \sum_{n=1}^N \sum_{k=K+1}^{K'} \Cap(x_n, p_k \given \Prefix')
  \wedge \one [ y_n = q_u ]\, \one [ x_n \in e_u ] + \Reg (K' - K) \nn \\
&= b(\Prefix, \x, \y) + \Upsilon(\Prefix', K, \x, \y) + \Reg (K' - K),
\label{eq:equiv-pts-b}
\end{align}
where the last equality uses~\eqref{eq:upsilon}.
Next, we write ${b_0(\Prefix, \x, \y)}$ with respect to~${b_0(\Prefix', \x, \y)}$,
\begin{align}
&b_0(\Prefix, \x, \y) = \frac{1}{N} \sum_{u=1}^U \sum_{n=1}^N
    \neg\, \Cap(x_n, \Prefix) \wedge \one [ x_n \in e_u ]\, \one [ y_n = q_u ] \nn \\
&= \frac{1}{N} \sum_{u=1}^U \sum_{n=1}^N
    \left(\neg\, \Cap(x_n, \Prefix')+ \sum_{k=K+1}^{K'} \Cap(x_n, p_k \given \Prefix') \right)
    \wedge \one [ x_n \in e_u ]\, \one [ y_n = q_u ] \nn \\
&= b_0(\Prefix', \x, \y) + \frac{1}{N} \sum_{u=1}^U \sum_{n=1}^N
    \sum_{k=K+1}^{K'} \Cap(x_n, p_k \given \Prefix') \wedge \one [ x_n \in e_u ]\, \one [ y_n = q_u ].
\label{eq:b0}
\end{align}
Rearranging~\eqref{eq:b0} gives
\begin{align}
b_0(\Prefix', \x, \y) &= b_0(\Prefix, \x, y) - \Upsilon(\Prefix', K, \x, \y).
\label{eq:equiv-pts-b0}
\end{align}
Combining~\eqref{eq:equiv-pts-extend} with first~\eqref{eq:equiv-pts-b0}
and then~\eqref{eq:equiv-pts-b} gives the desired inequality in~\eqref{eq:identical}:
\begin{align}
\Obj(\RL', \x, \y) &\ge b(\Prefix', \x, \y) + b_0(\Prefix', \x, \y) \nn \\
&= b(\Prefix', \x, \y) + b_0(\Prefix, \x, y) - \Upsilon(\Prefix', K, \x, \y) \nn \\
&\ge b(\Prefix, \x, \y) + \Upsilon(\Prefix', K, \x, \y) + \Reg (K' - K) + b_0(\Prefix, \x, y) - \Upsilon(\Prefix', K, \x, \y) \nn \\
&= b(\Prefix, \x, \y) + b_0(\Prefix, \x, y) + \Reg (K' - K)
\ge b(\Prefix, \x, \y) + b_0(\Prefix, \x, \y). \nn
\end{align}

%\clearpage
\section{Data Processing Details and Antecedent Mining}
\label{appendix:data}

In this appendix, we provide details regarding datasets used in our experiments
(Section~\ref{sec:experiments}).

\subsection{ProPublica Recidivism Data Set}
Table~\ref{tab:recidivism-data} shows the~6 attributes
and corresponding~17 categorical values
that we use for the ProPublica data set.
From these, we construct~17 single-clause antecedents, for example,
${(age = 23-25)}$.
We then combine pairs of these antecedents as conjunctions to form
two-clause antecedents, \eg ${(age = 23-25) \wedge (priors = 2-3)}$.
By virtue of our lower bound on antecedent support,
(Theorem~\ref{thm:min-capture},~\S\ref{sec:lb-support}),
we eliminate antecedents with support less than~0.005 or greater than~0.995,
since ${\Reg = 0.005}$ is the smallest regularization parameter value
we study for this problem.
With this filtering step, we generate between~121 and~123 antecedents for each fold;
without it, we would instead generate about~130 antecedents as input to our algorithm.

Note that we exclude the `current charge' attribute (which has two categorical values,
`misdemeanor' and `felony'); for individuals in the data set booked on multiple charges,
this attribute does not appear to consistently reflect the most serious charge.

\begin{table}[h!]
\centering
%$Q_\text{max}$ $K_\text{max}$
\begin{tabular}{l  | c  c  c}
Feature & Value range & Categorical values & Count \\
\hline
sex & --- & male, female & 2 \\
age & 18-96 & 18-20, 21-22, 23-25, 26-45, $>$45  & 5 \\
juvenile felonies & 0-20 & 0, $>$0 & 2 \\
juvenile misdemeanors & 0-13 & 0, $>$0 & 2 \\
%other juvenile crimes & 0-17 & --- & 2 \\
juvenile crimes & 0-21 & 0, $>$0 & 2 \\
priors & 0-38 & 0, 1, 2-3, $>$3 & 4
\end{tabular}
%\vspace{4mm}
\caption{Categorical features (6 attributes, 17 values) from the ProPublica data set.
We construct the feature \emph{juvenile crimes} from the sum of
\emph{juvenile felonies}, \emph{juvenile misdemeanors}, and
the number of juvenile crimes that were neither felonies nor misdemeanors (not shown).
}
\vspace{4mm}
\label{tab:recidivism-data}
\end{table}

\subsection{NYPD Stop-and-frisk Data Set}
This data set is larger than, but similar to the NYCLU stop-and-frisk data set, 
described next.

%\newpage
\subsection{NYCLU Stop-and-frisk Data Set}
The original data set contains 45,787 records,
each describing an incident involving a stopped person; the individual
was frisked in 30,345 (66.3\%) of records and and searched in 7,283 (15.9\%).
In 30,961 records, the individual was frisked and/or searched (67.6\%); of those,
a criminal possession of a weapon was identified 1,445 times (4.7\% of these records).
We remove 1,929 records with missing data, as well as a small number with extreme values
for the individual's age---we eliminate those with age~${< 12}$ or~${>89}$.
%we also assume that one age marked `366' is a typo, and we replace it with `36'.
%
This yields a set of 29,595 records in which the individual was frisked and/or searched.
To address the class imbalance for this problem, we sample records from the
smaller class with replacement.
We generate cross-validation folds first, and then resample within each fold.
In our 10-fold cross-validation experiments, each training set contains 50,743 observations.
Table~\ref{tab:frisk-data} shows the 5 categorical attributes that we use,
corresponding to a total of 28 values.
Our experiments use these antecedents,
as well as negations of the 18 antecedents corresponding to the two features
\emph{stop reason} and \emph{additional circumstances},
which gives a total of 46 antecedents.

\begin{table}[h!]
\centering
%$Q_\text{max}$ $K_\text{max}$
\begin{tabular}{l | c  c}
Feature & Values & Count \\
\hline
stop reason & suspicious object, fits description, casing, & 9 \\
& acting as lookout, suspicious clothing, & \\
& drug transaction, furtive movements, & \\
& actions of violent crime, suspicious bulge \\
\hline
additional & proximity to crime scene, evasive response,  & 9 \\
circumstances & associating with criminals, changed direction, & \\
& high crime area, time of day,  & \\
& sights and sounds of criminal activity, & \\
& witness report, ongoing investigation & \\
\hline
city & Queens,  Manhattan, Brooklyn, Staten Island, Bronx & 5 \\
\hline
location & housing authority, transit authority, & 3 \\
& neither housing nor transit authority & \\
\hline
inside or outside & inside, outside & 2 \\
\end{tabular}
%\vspace{4mm}
\caption{Categorical features (5 attributes, 28 values) from the NYCLU data set.}
\vspace{4mm}
\label{tab:frisk-data}
\end{table}

%\clearpage
\section{Example Optimal Rule Lists, for Different Values of~$\Reg$}
\label{appendix:examples}

For each of our prediction problems, we provide listings of
optimal rule lists found by CORELS, across 10 cross-validation folds,
for different values of the regularization parameter~$\Reg$.
These rule lists correspond to the results for CORELS summarized
in Figures~\ref{fig:sparsity-compas} and~\ref{fig:sparsity-weapon}~(\S\ref{sec:sparsity}).
Recall that as~$\Reg$ decreases, optimal rule lists tend to grow in length. \\

\subsection{ProPublica Recidivism Data Set}
We show example optimal rule lists that predict two-year recidivism.
Figure~\ref{fig:recidivism-rule-list-02-01} shows examples for
regularization parameters~${\Reg = 0.02}$ and~0.01.
Figure~\ref{fig:recidivism-rule-list-005} shows examples for~${\Reg = 0.005}$;
Figure~\ref{fig:recidivism-all-folds}~(\S\ref{sec:examples}) showed two representative examples.

For the largest regularization parameter~${\Reg = 0.02}$~(Figure~\ref{fig:recidivism-rule-list-02-01}),
we observe that all folds identify the same length-1 rule list.
For the intermediate value~${\Reg = 0.01}$ (Figure~\ref{fig:recidivism-rule-list-02-01}),
the folds identify optimal 2-rule or 3-rule lists that contain the nearly same prefix rules,
up to permutations.
For the smallest value~${\Reg = 0.005}$~(Figure~\ref{fig:recidivism-rule-list-005}),
the folds identify optimal 3-rule or 4-rule lists that contain the nearly same prefix rules,
up to permutations.
Across all three regularization parameter values and all folds,
the prefix rules always predict the positive class label,
and the default rule always predicts the negative class label.
We note that our objective is not designed to enforce any of these properties.
%though some may be seen as desirable.

% tail -n 1 ../logs/for-compas_*_train.out-curious_lb-with_prefix_perm_map-minor-removed=none-max_num_nodes=10000000-c=0.0200000-v=progress-f=1000-opt.txt
%{priors:>3}~1;default~0
%
%$ tail -n 1 ../logs/for-compas_*_train.out-curious_lb-with_prefix_perm_map-minor-removed=none-max_num_nodes=10000000-c=0.0100000-v=progress-f=1000-opt.txt
%{priors:>3}~1;{sex:Male,juvenile-crimes:>0}~1;default~0 x 3
%{sex:Male,juvenile-crimes:>0}~1;{priors:>3}~1;default~0 x 2
%{sex:Male,age:18-20}~1;{priors:>3}~1;default~0 x 2
%{age:21-22,priors:2-3}~1;{priors:>3}~1;{sex:Male,age:18-20}~1;default~0 x 2
%{priors:>3}~1;{sex:Male,age:18-20}~1;default~0
\begin{figure}[h!]
\textbf{Two-year recidivism prediction $(\Reg = 0.02)$}
\vspace{1mm}
\begin{algorithmic}
\State \bif $(priors > 3)$ \bthen $yes$ \Comment{Found by all 10 folds}
\State \belse $no$
\end{algorithmic}
\vspace{5mm}
\textbf{Two-year recidivism prediction $(\Reg = 0.01)$}
\vspace{1mm}
\begin{algorithmic}
\State \bif $(priors > 3)$ \bthen $yes$ \Comment{Found by 3 folds}
\State \belif $(sex = male) \band (juvenile~crimes > 0)$ \bthen $yes$
\State \belse $no$
\end{algorithmic}
\vspace{1mm}
\begin{algorithmic}
\State \bif $(sex = male) \band (juvenile~crimes > 0)$ \bthen $yes$ \Comment{Found by 2 folds}
\State \belif $(priors > 3)$ \bthen $yes$
\State \belse $no$
\end{algorithmic}
\vspace{1mm}
\begin{algorithmic}
\State \bif $(age = 21-22) \band (priors = 2-3)$ \bthen $yes$ \Comment{Found by 2 folds}
\State \belif $(priors > 3)$ \bthen $yes$
\State \belif $(age = 18-20) \band (sex = male)$ \bthen $yes$
\State \belse $no$
\end{algorithmic}
\vspace{1mm}
\begin{algorithmic}
\State \bif $(age = 18-20) \band (sex = male)$ \bthen $yes$ \Comment{Found by 2 folds}
\State \belif $(priors > 3)$ \bthen $yes$
\State \belse $no$
\end{algorithmic}
\vspace{1mm}
\begin{algorithmic}
\State \bif $(priors > 3)$ \bthen $yes$ \Comment{Found by 1 fold}
\State \belif $(age = 18-20) \band (sex = male)$ \bthen $yes$
\State \belse $no$
\end{algorithmic}
\caption{Example optimal rule lists for the ProPublica data set,
found by CORELS with regularization parameters~${\Reg = 0.02}$~(top),
and~0.01~(bottom) across 10 cross-validation folds.
}
\label{fig:recidivism-rule-list-02-01}
\end{figure}

%logs in jmlr/ from Nicholas 9/27
%$ tail -n 1 *compas*curious*with*-minor*none-m*1000000000*0.005*opt.txt
%{sex:Male,age:18-20}~1;{age:21-22,priors:2-3}~1;{priors:>3}~1;default~0 x 4
%{age:21-22,priors:2-3}~1;{priors:>3}~1;{sex:Male,age:18-20}~1;default~0 x 2
%{sex:Male,age:18-20}~1;{priors:>3}~1;{age:21-22,priors:2-3}~1;default~0
%{sex:Male,age:18-20}~1;{age:21-22,priors:2-3}~1;{age:23-25,priors:2-3}~1;{priors:>3}~1;default~0
%{sex:Male,age:18-20}~1;{age:21-22,priors:2-3}~1;{priors:>3}~1;{age:23-25,priors:2-3}~1;default~0
%{age:21-22,priors:2-3}~1;{age:23-25,priors:2-3}~1;{priors:>3}~1;{sex:Male,age:18-20}~1;default~0
\begin{figure}[h!]
\textbf{Two-year recidivism prediction $(\Reg = 0.005)$}
\vspace{1mm}
\begin{algorithmic}
\State \bif $(age = 18-20) \band (sex = male)$ \bthen $yes$ \Comment{Found by 4 folds}
\State \belif $(age = 21-22) \band (priors = 2-3)$ \bthen $yes$
\State \belif $(priors > 3)$ \bthen $yes$
\State \belse $no$
\end{algorithmic}
\vspace{1mm}
\begin{algorithmic}
\State \bif $(age = 21-22) \band (priors = 2-3)$ \bthen $yes$  \Comment{Found by 2 folds}
\State \belif $(priors > 3)$ \bthen $yes$
\State \belif $(age = 18-20) \band (sex = male)$ \bthen $yes$
\State \belse $no$
\end{algorithmic}
\vspace{1mm}
\begin{algorithmic}
\State \bif $(age = 18-20) \band (sex = male)$ \bthen $yes$ \Comment{Found by 1 fold}
\State \belif $(priors > 3)$ \bthen $yes$
\State \belif $(age = 21-22) \band (priors = 2-3)$ \bthen $yes$
\State \belse $no$
\end{algorithmic}
\vspace{1mm}
\begin{algorithmic}
\State \bif $(age = 18-20) \band (sex = male)$ \bthen $yes$ \Comment{Found by 1 fold}
\State \belif $(age = 21-22) \band (priors = 2-3)$ \bthen $yes$
\State \belif $(age = 23-25) \band (priors = 2-3)$ \bthen $yes$
\State \belif $(priors > 3)$ \bthen $yes$
\State \belse $no$
\end{algorithmic}
\vspace{1mm}
\begin{algorithmic}
\State \bif $(age = 18-20) \band (sex = male)$ \bthen $yes$ \Comment{Found by 1 fold}
\State \belif $(age = 21-22) \band (priors = 2-3)$ \bthen $yes$
\State \belif $(priors > 3)$ \bthen $yes$
\State \belif $(age = 23-25) \band (priors = 2-3)$ \bthen $yes$
\State \belse $no$
\end{algorithmic}
\vspace{1mm}
\begin{algorithmic}
\State \bif $(age = 21-22) \band (priors = 2-3)$ \bthen $yes$ \Comment{Found by 1 fold}
\State \belif $(age = 23-25) \band (priors = 2-3)$ \bthen $yes$
\State \belif $(priors > 3)$ \bthen $yes$
\State \belif $(age = 18-20) \band (sex = male)$ \bthen $yes$
\State \belse $no$
\end{algorithmic}
\caption{Example optimal rule lists for the ProPublica data set,
found by CORELS with regularization parameters~${\Reg = 0.005}$,
across 10 cross-validation folds.
}
\label{fig:recidivism-rule-list-005}
\end{figure}

\clearpage
\subsection{NYPD Stop-and-frisk Data Set}
We show example optimal rule lists that predict whether a weapon
will be found on a stopped individual who is frisked or searched, learned from the NYPD data set.

\begin{figure}[b!]
%$ head *cpw_*0.01*opt* K = 2
%{cs_objcs:stop-reason=suspicious-object}~1;{location:transit-authority}~1;default~0 x 8
%{location:transit-authority}~1;{cs_objcs:stop-reason=suspicious-object}~1;default~0 x 2
\textbf{Weapon prediction $(\Reg = 0.01, \text{Feature Set~C})$}
\vspace{1mm}
\begin{algorithmic}
\State \bif $(stop~reason = suspicious~object)$ \bthen $yes$ \Comment{Found by 8 folds}
\State \belif $(location = transit~authority)$ \bthen $yes$
\State \belse $no$
\end{algorithmic}
\vspace{1mm}
\begin{algorithmic}
\State \bif $(location = transit~authority)$ \bthen $yes$ \Comment{Found by 2 folds}
\State \belif $(stop~reason = suspicious~object)$ \bthen $yes$
\State \belse $no$
\end{algorithmic}
\vspace{5mm}
%$ head *cpw_*0.005*opt* K = 4 or 5
%{cs_objcs:stop-reason=suspicious-object}~1;{location:transit-authority}~1;{location:housing-authority}~0;{city:MANHATTAN}~1;default~0 x 7
%{cs_objcs:stop-reason=suspicious-object}~1;{location:housing-authority}~0;{location:transit-authority}~1;{city:MANHATTAN}~1;default~0
%{cs_objcs:stop-reason=suspicious-object}~1;{location:housing-authority}~0;{city:MANHATTAN}~1;{location:transit-authority}~1;default~0
%{cs_objcs:stop-reason=suspicious-object}~1;{location:transit-authority}~1;{city:BRONX}~0;{location:housing-authority}~0;{cs_furtv:stop-reason=furtive-movements}~0;default~1
\textbf{Weapon prediction $(\Reg = 0.005, \text{Feature Set~C})$}
\begin{algorithmic}
\State \bif $(stop~reason = suspicious~object)$ \bthen $yes$ \Comment{Found by 7 folds}
\State \belif $(location = transit~authority)$ \bthen $yes$
\State \belif $(location = housing~authority)$ \bthen $no$
\State \belif $(city = Manhattan)$ \bthen $yes$
\State \belse $no$
\end{algorithmic}
\vspace{1mm}
\begin{algorithmic}
\State \bif $(stop~reason = suspicious~object)$ \bthen $yes$ \Comment{Found by 1 fold}
\State \belif $(location = housing~authority)$ \bthen $no$
\State \belif $(location = transit~authority)$ \bthen $yes$
\State \belif $(city = Manhattan)$ \bthen $yes$
\State \belse $no$
\end{algorithmic}
\vspace{1mm}
\begin{algorithmic}
\State \bif $(stop~reason = suspicious~object)$ \bthen $yes$ \Comment{Found by 1 fold}
\State \belif $(location = housing~authority)$ \bthen $no$
\State \belif $(city = Manhattan)$ \bthen $yes$
\State \belif $(location = transit~authority)$ \bthen $yes$
\State \belse $no$
\end{algorithmic}
\vspace{1mm}
\begin{algorithmic}
\State \bif $(stop~reason = suspicious~object)$ \bthen $yes$ \Comment{Found by 1 fold}
\State \belif $(location = transit~authority)$ \bthen $yes$
\State \belif $(city = Bronx)$ \bthen $no$
\State \belif $(location = housing~authority)$ \bthen $no$
\State \belif $(stop~reason = furtive~movements)$ \bthen $no$
\State \belse $yes$
\end{algorithmic}
\caption{Example optimal rule lists for the NYPD stop-and-frisk data set,
found by CORELS with regularization parameters~${\Reg = 0.01}$~(top) and~0.005~(bottom),
across 10 cross-validation folds.
}
\label{fig:cpw-rule-list}
\end{figure}

\begin{figure}[b!]
%$ head *cpw-noloc_*0.01*opt* K = 2
%{cs_objcs:stop-reason=suspicious-object}~1;{inout:outside}~0;default~1 x 7
%{cs_objcs:stop-reason=suspicious-object}~1;{inout:inside}~1;default~0 x 3
\textbf{Weapon prediction $(\Reg = 0.01, \text{Feature Set~D})$}
\vspace{1mm}
\begin{algorithmic}
\State \bif $(stop~reason = suspicious~object)$ \bthen $yes$ \Comment{Found by 7 folds}
\State \belif $(inside~or~outside = outside)$ \bthen $no$
\State \belse $yes$
\end{algorithmic}
\vspace{1mm}
\begin{algorithmic}
\State \bif $(stop~reason = suspicious~object)$ \bthen $yes$ \Comment{Found by 3 folds}
\State \belif $(inside~or~outside = inside)$ \bthen $yes$
\State \belse $no$
\end{algorithmic}
\vspace{5mm}
%$ head *cpw-noloc_*0.005*opt* K = 4
%{cs_objcs:stop-reason=suspicious-object}~1;{cs_lkout:stop-reason=acting-as-lookout}~0;{cs_descr:stop-reason=fits-description}~0;{cs_furtv:stop-reason=furtive-movements}~0;default~1 x 2
%{cs_objcs:stop-reason=suspicious-object}~1;{cs_furtv:stop-reason=furtive-movements}~0;{cs_lkout:stop-reason=acting-as-lookout}~0;{cs_descr:stop-reason=fits-description}~0;default~1 x 2
%{cs_objcs:stop-reason=suspicious-object}~1;{cs_lkout:stop-reason=acting-as-lookout}~0;{cs_furtv:stop-reason=furtive-movements}~0;{cs_descr:stop-reason=fits-description}~0;default~1
%{cs_objcs:stop-reason=suspicious-object}~1;{cs_descr:stop-reason=fits-description}~0;{cs_lkout:stop-reason=acting-as-lookout}~0;{cs_furtv:stop-reason=furtive-movements}~0;default~1
%{cs_objcs:stop-reason=suspicious-object}~1;{cs_furtv:stop-reason=furtive-movements}~0;{cs_descr:stop-reason=fits-description}~0;{cs_lkout:stop-reason=acting-as-lookout}~0;default~1
%
%{cs_objcs:stop-reason=suspicious-object}~1;{cs_descr:stop-reason=fits-description}~0;{cs_casng:stop-reason=casing}~0;{cs_furtv:stop-reason=furtive-movements}~0;default~1
%{cs_objcs:stop-reason=suspicious-object}~1;{cs_casng:stop-reason=casing}~0;{cs_descr:stop-reason=fits-description}~0;{cs_furtv:stop-reason=furtive-movements}~0;default~1
%{cs_objcs:stop-reason=suspicious-object}~1;{cs_casng:stop-reason=casing}~0;{cs_furtv:stop-reason=furtive-movements}~0;{cs_descr:stop-reason=fits-description}~0;default~1
\textbf{Weapon prediction $(\Reg = 0.005, \text{Feature Set~D})$}
\vspace{1mm}
\begin{algorithmic}
\State \bif $(stop~reason = suspicious~object)$ \bthen $yes$ \Comment{Found by 2 folds}
\State \belif $(stop~reason = acting~as~lookout)$ \bthen $no$
\State \belif $(stop~reason = fits~description)$ \bthen $no$
\State \belif $(stop~reason = furtive~movements)$ \bthen $no$
\State \belse $yes$
\end{algorithmic}
\vspace{1mm}
\begin{algorithmic}
\State \bif $(stop~reason = suspicious~object)$ \bthen $yes$ \Comment{Found by 2 folds}
\State \belif $(stop~reason = furtive~movements)$ \bthen $no$
\State \belif $(stop~reason = acting~as~lookout)$ \bthen $no$
\State \belif $(stop~reason = fits~description)$ \bthen $no$
\State \belse $yes$
\end{algorithmic}
\vspace{1mm}
\begin{algorithmic}
\State \bif $(stop~reason = suspicious~object)$ \bthen $yes$ \Comment{Found by 1 fold}
\State \belif $(stop~reason = acting~as~lookout)$ \bthen $no$
\State \belif $(stop~reason = furtive~movements)$ \bthen $no$
\State \belif $(stop~reason = fits~description)$ \bthen $no$
\State \belse $yes$
\end{algorithmic}
\begin{algorithmic}
\vspace{1mm}
\State \bif $(stop~reason = suspicious~object)$ \bthen $yes$ \Comment{Found by 1 fold}
\State \belif $(stop~reason = fits~description)$ \bthen $no$
\State \belif $(stop~reason = acting~as~lookout)$ \bthen $no$
\State \belif $(stop~reason = furtive~movements)$ \bthen $no$
\State \belse $yes$
\end{algorithmic}
\vspace{1mm}
\begin{algorithmic}
\State \bif $(stop~reason = suspicious~object)$ \bthen $yes$ \Comment{Found by 1 fold}
\State \belif $(stop~reason = furtive~movements)$ \bthen $no$
\State \belif $(stop~reason = fits~description)$ \bthen $no$
\State \belif $(stop~reason = acting~as~lookout)$ \bthen $no$
\State \belse $yes$
\end{algorithmic}
\caption{Example optimal rule lists for the NYPD stop-and-frisk data set (Feature Set~D)
found by CORELS with regularization parameters~${\Reg = 0.01}$~(top) and~0.005~(bottom),
across 10 cross-validation folds.
For~${\Reg = 0.005}$, we show results from~7 folds; the remaining~3 folds were equivalent,
up to a permutation of the prefix rules, and started with the same first prefix rule.
}
\label{fig:cpw-noloc-rule-list}
\end{figure}

\clearpage
\subsection{NYCLU Stop-and-frisk Data Set}
We show example optimal rule lists that predict whether a weapon
will be found on a stopped individual who is frisked or searched, learned from the NYCLU data set.
Figure~\ref{fig:weapon-rule-list-04-01} shows
regularization parameters~${\Reg = 0.04}$ and~0.01,
and Figure~\ref{fig:weapon-rule-list-0025} shows~${\Reg = 0.0025}$.
We showed a representative solution for~${\Reg = 0.01}$ in
Figure~\ref{fig:weapon-rule-list}~(\S\ref{sec:examples}).

For each of the two larger regularization parameters in Figure~\ref{fig:weapon-rule-list-04-01},
${\Reg = 0.04}$~(top) and 0.01~(bottom), we observe that across the folds,
all the optimal rule lists contain the same or equivalent rules, up to a permutation.
With the smaller regularization parameter~${\Reg = 0.0025}$ (Figure~\ref{fig:weapon-rule-list-0025}),
we observe a greater diversity of longer optimal rule lists, though they share similar structure.

%logs in jmlr/ from Nicholas 9/27
%tail -n 1 tail -n 1 *weapon*0.04*opt.txt
%{cs_objcs:stop-reason=suspicious-object}~1;{cs_bulge:stop-reason=not-suspicious-bulge}~0;default~1 x 7
%{cs_bulge:stop-reason=suspicious-bulge}~1;{cs_objcs:stop-reason=not-suspicious-object}~0;default~1 x 3
%
%$ tail -n 1 *weapon*curious*with*-minor-*none*1000000000*0.01*opt.txt
%{cs_objcs:stop-reason=suspicious-object}~1;{location:transit-authority}~1;{cs_bulge:stop-reason=not-suspicious-bulge}~0;default~1 x 4
%{location:transit-authority}~1;{cs_bulge:stop-reason=suspicious-bulge}~1;{cs_objcs:stop-reason=suspicious-object}~1;default~0 x 3
%{location:transit-authority}~1;{cs_objcs:stop-reason=suspicious-object}~1;{cs_bulge:stop-reason=suspicious-bulge}~1;default~0 x 2
%{location:transit-authority}~1;{cs_objcs:stop-reason=suspicious-object}~1;{cs_bulge:stop-reason=not-suspicious-bulge}~0;default~1
\begin{figure}[b!]
\textbf{Weapon prediction $(\Reg = 0.04)$}
\vspace{1mm}
\begin{algorithmic}
\State \bif $(stop~reason = suspicious~object)$ \bthen $yes$ \Comment{Found by 7 folds}
\State \belif $(stop~reason \neq suspicious~bulge)$ \bthen $no$
\State \belse $yes$
\end{algorithmic}
\vspace{1mm}
\begin{algorithmic}
\State \bif $(stop~reason = suspicious~bulge)$ \bthen $yes$ \Comment{Found by 3 folds}
\State \belif $(stop~reason \neq suspicious~object)$ \bthen $no$
\State \belse $yes$
\end{algorithmic}
\vspace{5mm}
\textbf{Weapon prediction $(\Reg = 0.01)$}
\vspace{1mm}
\begin{algorithmic}
\State \bif $(stop~reason = suspicious~object)$ \bthen $yes$ \Comment{Found by 4 folds}
\State \belif $(location = transit~authority)$ \bthen $yes$
\State \belif $(stop~reason \neq suspicious~bulge)$ \bthen $no$
\State \belse $yes$
\end{algorithmic}
\vspace{1mm}
\begin{algorithmic}
\State \bif $(location = transit~authority)$ \bthen $yes$ \Comment{Found by 3 folds}
\State \belif $(stop~reason = suspicious~bulge)$ \bthen $yes$
\State \belif $(stop~reason = suspicious~object)$ \bthen $yes$
\State \belse $no$
\end{algorithmic}
\vspace{1mm}
\begin{algorithmic}
\State \bif $(location = transit~authority)$ \bthen $yes$ \Comment{Found by 2 folds}
\State \belif $(stop~reason = suspicious~object)$ \bthen $yes$
\State \belif $(stop~reason = suspicious~bulge)$ \bthen $yes$
\State \belse $no$
\end{algorithmic}
\vspace{1mm}
\begin{algorithmic}
\State \bif $(location = transit~authority)$ \bthen $yes$ \Comment{Found by 1 fold}
\State \belif $(stop~reason = suspicious~object)$ \bthen $yes$
\State \belif $(stop~reason \neq suspicious~bulge)$ \bthen $no$
\State \belse $yes$
\end{algorithmic}
\caption{Example optimal rule lists for the NYCLU stop-and-frisk data set,
found by CORELS with regularization parameters~${\Reg = 0.04}$~(top) and~0.01~(bottom),
across 10 cross-validation folds.
}
\label{fig:weapon-rule-list-04-01}
\end{figure}

\begin{figure}[h!]
\textbf{Weapon prediction $(\Reg = 0.0025)$}
\vspace{0.5mm}
\scriptsize
\begin{algorithmic}
\State \bif $(stop~reason = suspicious~object)$ \bthen $yes$ \Comment{Found by 4 folds $(K=7)$}
\State \belif $(stop~reason = casing)$ \bthen $no$
\State \belif $(stop~reason = suspicious~bulge)$ \bthen $yes$
\State \belif $(stop~reason = fits~description)$ \bthen $no$
\State \belif $(location = transit~authority)$ \bthen $yes$
\State \belif $(inside~or~outside = inside)$ \bthen $no$
\State \belif $(city = Manhattan)$ \bthen $yes$
\State \belse $no$
\end{algorithmic}
\vspace{0.5mm}
\begin{algorithmic}
\State \bif $(stop~reason = suspicious~object)$ \bthen $yes$ \Comment{Found by 1 fold $(K=6)$}
\State \belif $(stop~reason = casing)$ \bthen $no$
\State \belif $(stop~reason = suspicious~bulge)$ \bthen $yes$
\State \belif $(stop~reason = fits~description)$ \bthen $no$
\State \belif $(location = housing~authority)$ \bthen $no$
\State \belif $(city = Manhattan)$ \bthen $yes$
\State \belse $no$
\end{algorithmic}
\vspace{0.5mm}
\begin{algorithmic}
\State \bif $(stop~reason = suspicious~object)$ \bthen $yes$ \Comment{Found by 1 fold $(K=6)$}
\State \belif $(stop~reason = suspicious~bulge)$ \bthen $yes$
\State \belif $(location = housing~authority)$ \bthen $no$
\State \belif $(stop~reason = casing)$ \bthen $no$
\State \belif $(stop~reason = fits~description)$ \bthen $no$
\State \belif $(city = Manhattan)$ \bthen $yes$
\State \belse $no$
\end{algorithmic}
\vspace{0.5mm}
\begin{algorithmic}
\State \bif $(stop~reason = suspicious~object)$ \bthen $yes$ \Comment{Found by 1 fold $(K=6)$}
\State \belif $(stop~reason = casing)$ \bthen $no$
\State \belif $(stop~reason = suspicious~bulge)$ \bthen $yes$
\State \belif $(stop~reason = fits~description)$ \bthen $no$
\State \belif $(location = housing~authority)$ \bthen $no$
\State \belif $(city = Manhattan)$ \bthen $yes$
\State \belse $no$
\end{algorithmic}
\vspace{0.5mm}
\begin{algorithmic}
\State \bif $(stop~reason = drug~transaction)$ \bthen $no$ \Comment{Found by 1 fold $(K=8)$}
\State \belif $(stop~reason = suspicious~object)$ \bthen $yes$
\State \belif $(stop~reason = suspicious~bulge)$ \bthen $yes$
\State \belif $(location = housing~authority)$ \bthen $no$
\State \belif $(stop~reason = fits~description)$ \bthen $no$
\State \belif $(stop~reason = casing)$ \bthen $no$
\State \belif $(city = Manhattan)$ \bthen $yes$
\State \belif $(city = Bronx)$ \bthen $yes$
\State \belse $no$
\end{algorithmic}
\vspace{0.5mm}
\begin{algorithmic}
\State \bif $(stop~reason = suspicious~object)$ \bthen $yes$ \Comment{Found by 1 fold $(K=9)$}
\State \belif $(stop~reason = casing)$ \bthen $no$
\State \belif $(stop~reason = suspicious~bulge)$ \bthen $yes$
\State \belif $(stop~reason = fits~description)$ \bthen $no$
\State \belif $(location = transit~authority)$ \bthen $yes$
\State \belif $(inside~or~outside = inside)$ \bthen $no$
\State \belif $(city = Manhattan)$ \bthen $yes$
\State \belif $(additional~circumstances = changed~direction)$ \bthen $no$
\State \belif $(city = Bronx)$ \bthen $yes$
\State \belse $no$
\end{algorithmic}
\vspace{0.5mm}
%{cs_objcs:stop-reason=suspicious-object}~1;{cs_casng:stop-reason=casing}~0;{cs_bulge:stop-reason=suspicious-bulge}~1;{cs_vcrim:stop-reason=actions-of-violent-crime}~0;{cs_descr:stop-reason=fits-description}~0;{location:transit-authority}~1;{inout:inside}~0;{city:Manhattan}~1;{ac_evasv:circumstances=evasive-response}~0;{city:Bronx}~1;default~0
\begin{algorithmic}
\State \bif $(stop~reason = suspicious~object)$ \bthen $yes$ \Comment{Found by 1 fold $(K=10)$}
\State \belif $(stop~reason = casing)$ \bthen $no$
\State \belif $(stop~reason = suspicious~bulge)$ \bthen $yes$
\State \belif $(stop~reason = actions~of~violent~crime)$ \bthen $no$
\State \belif $(stop~reason = fits~description)$ \bthen $no$
\State \belif $(location = transit~authority)$ \bthen $yes$
\State \belif $(inside~or~outside = inside)$ \bthen $no$
\State \belif $(city = Manhattan)$ \bthen $yes$
\State \belif $(additional~circumstances = evasive~response)$ \bthen $no$
\State \belif $(city = Bronx)$ \bthen $yes$
\State \belse $no$
\end{algorithmic}
\caption{Example optimal rule lists for the NYCLU stop-and-frisk data set~${\Reg = 0.0025}$.
%found by CORELS with regularization parameter, across 10 cross-validation folds.
}
\label{fig:weapon-rule-list-0025}
\end{figure}

\clearpage
\section{Additional Results on Predictive Performance and Model Size for
CORELS and Other Algorithms}
\label{appendix:cpw}

In this appendix, we plot TPR, FPR, and model size for CORELS and three other
algorithms, using the NYPD data set (Feature Set~D).

\begin{figure}[htb!]
\begin{center}
%\includegraphics[width=0.75\textwidth]{figs/sketch-comparison.png}
% left lower right upper
\includegraphics[trim={17mm, 0mm, 27mm, 0mm},
width=0.7\textwidth]{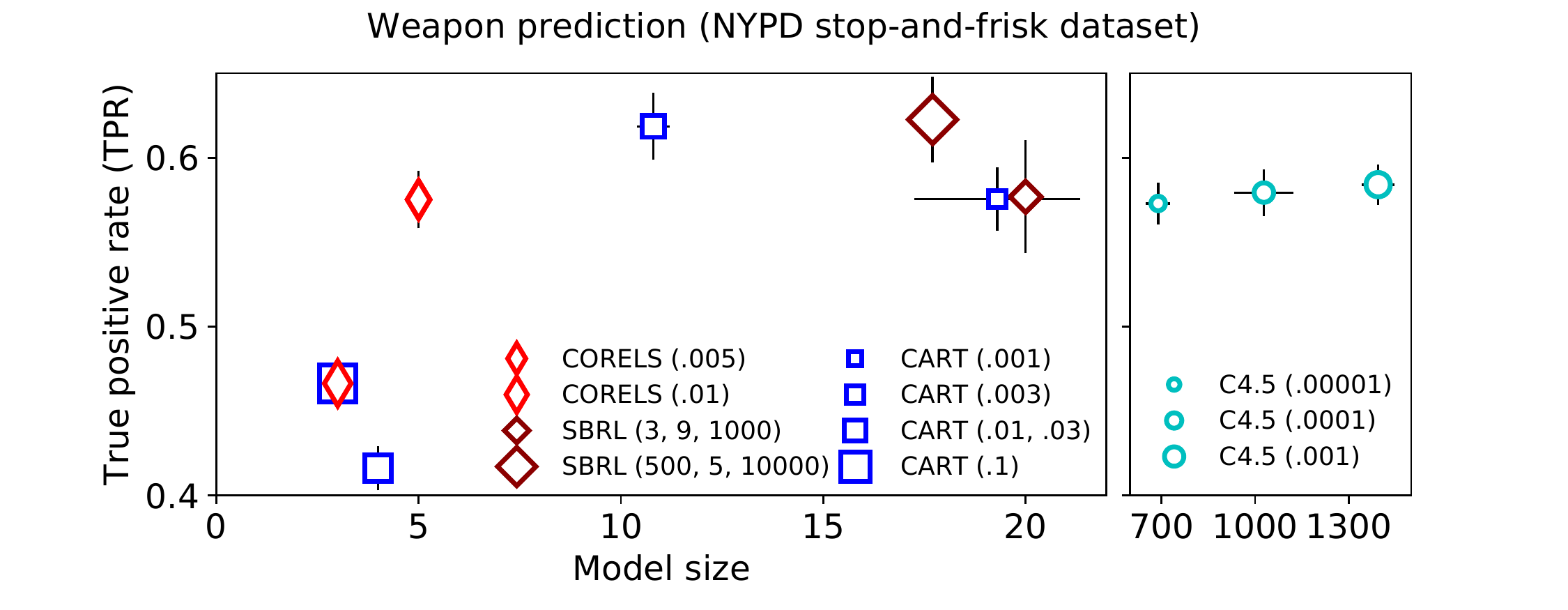}
\includegraphics[trim={17mm, 10mm, 27mm, 4mm},
width=0.7\textwidth]{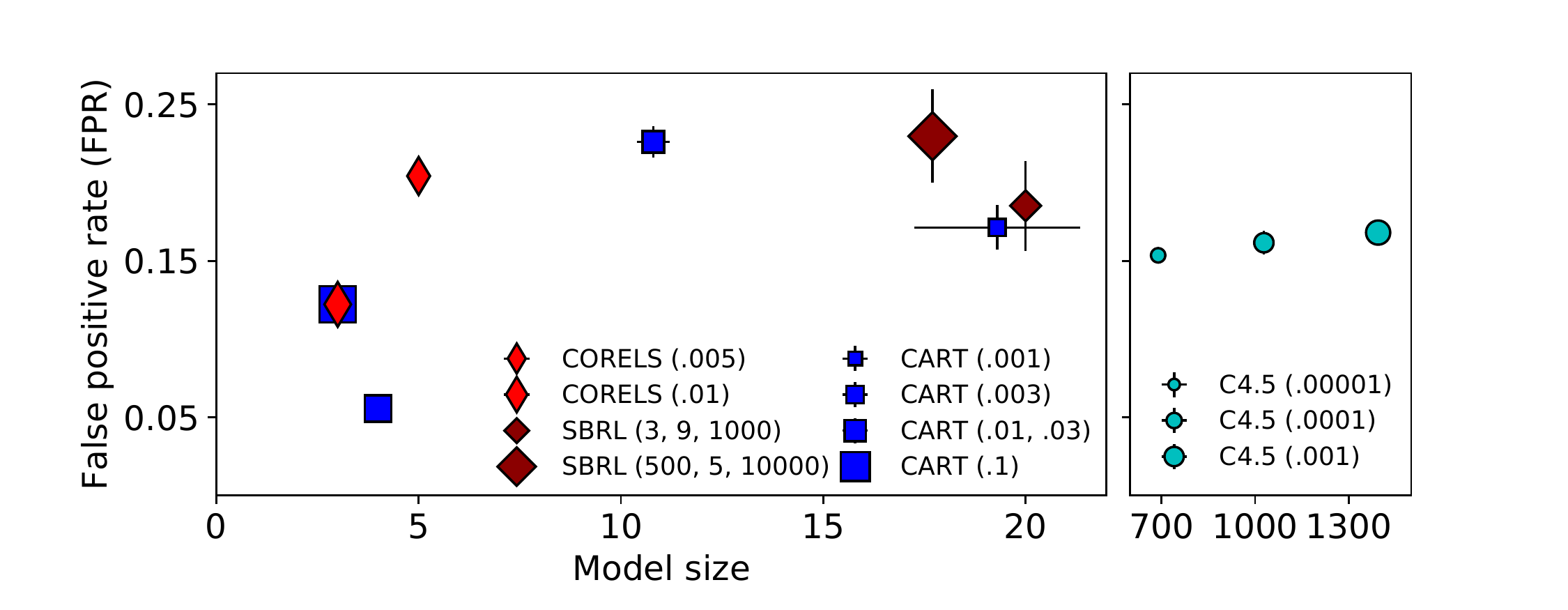}
\end{center}
\caption{TPR (top) and FPR (bottom)
for the test set, as a function of model size, across different methods,
for weapon prediction with the NYPD stop-and-frisk data set (Feature Set~D).
In the legend, numbers in parentheses are algorithm parameters,
as in Figure~\ref{fig:sparsity-weapon}.
Legend markers and error bars indicate means and standard deviations,
respectively, across cross-validation folds.
%
%For CORELS and SBRL, we use ${M = 28}$ antecedents.
%
%CART with ${cp = 0.001}$ significantly overfits;
%C4.5 finds large models and dramatically overfits for all tested parameters.
C4.5 finds large models for all tested parameters.
}
\label{fig:sparsity-cpw}
\end{figure}

\bibliography{17-716}

\end{document}